\documentclass[twoside,11pt]{article}

\usepackage[nohyperref]{jmlr2e}
\usepackage[table,dvipsnames]{xcolor}  %
\usepackage{natbib}

\usepackage{amsfonts}
\usepackage{multirow,tabularx}
\usepackage{listings}
\usepackage{makecell}
\usepackage[labelfont=bf]{caption} %
\usepackage{algorithm}
\usepackage[noend]{algpseudocode}
\algrenewcommand\algorithmicrequire{\textbf{Input:}}
\algrenewcommand\algorithmicensure{\textbf{Output:}}

\usepackage{amssymb}
\usepackage{graphicx}
\usepackage{amsmath}
\usepackage{xfrac,nicefrac} %
\usepackage{mathrsfs}
\usepackage{subcaption}
\usepackage[colorlinks = true,
            linkcolor = blue,
            urlcolor  = blue,
            citecolor = blue,
            anchorcolor = blue]{hyperref}
\usepackage{xspace}
\usepackage{cleveref}
\usepackage{setspace}

\usepackage[scaled]{helvet} %
\usepackage{courier} %
\normalfont
\usepackage[T1]{fontenc}

\usepackage{changepage}
\usepackage{enumitem}
\usepackage{array}
\usepackage{booktabs}
\usepackage{bbm}
\usepackage{soul}
\usepackage{relsize}
\usepackage{eso-pic}
\usepackage{epsfig}
\usepackage{rotating}
\usepackage{balance}
\usepackage{wrapfig}
\usepackage{framed}
\usepackage{mathtools}
\usepackage{pifont}
\usepackage{scrextend}
\usepackage{sidecap}
\usepackage{adjustbox}
\usepackage{csquotes}
\usepackage{minitoc}

\usepackage{etoolbox}

\usepackage{tikz,pgfplots}
\usepackage{pgfplotstable}
\usetikzlibrary{shapes}
\usetikzlibrary{positioning}
\usetikzlibrary{plotmarks}
\usetikzlibrary{patterns}
\usetikzlibrary{intersections,shapes.arrows}
\usetikzlibrary{arrows.meta}
\usetikzlibrary{pgfplots.fillbetween}
\usepgfplotslibrary{groupplots}
\usetikzlibrary{fit}

\usepackage{tcolorbox}

\definecolor{azure}{rgb}{0.0, 0.5, 1.0}
\tcbuselibrary{skins}
\tcbset{enhanced}

\newcommand{\myparagraph}[1]{\paragraph{#1.}\hspace{-0.8em}}  %
\newcommand \reals {\mathbb{R}}

\newcommand \Scal {\mathcal{S}}
\newcommand \Xcal {\mathcal{X}}

\newcommand \Ncal {\mathcal{N}}
\newcommand \Pcal {\mathcal{P}}
\newcommand \Fcal {\mathcal{F}}

\newcommand \Tcal {\mathcal{T}}
\newcommand \Hcal {\mathcal{H}}

\usepackage{bm}

\newcommand \sv {{\bm{s}}}
\newcommand \wv {{\bm{w}}}
\newcommand \xv {{\bm{x}}}

\newcommand \zv {{\bm{z}}}

\newcommand \av {{\bm{a}}}

\newcommand \hv {{\bm{h}}}

\newcommand \cv {{\bm{c}}}
\newcommand \uv {{\bm{u}}}

\newcommand \xiv {{\bm{\xi}}}

\newcommand \ind {\operatorname*{\mathbb{I}}}
\newcommand \prob {\mathbb{P}}
\newcommand \expect {\mathbb{E}}
\newcommand{\ot}{{\mathrm{OT}}}
\newcommand{\kl}{{\mathrm{KL}}}
\newcommand{\js}{{\mathrm{JS}}}
\newcommand{\tv}{{\mathrm{TV}}}

\newcommand{\mray}{{\operatorname{FI}}}  %
\newcommand{\fint}{{\operatorname{FI}}}  %
\newcommand{\midp}{{\operatorname{Mid}}}  %
\newcommand{\mauveray}{Frontier Integral\xspace}
\newcommand{\klam}[1]{\mathrm{KL}_{#1}}
\newcommand{\lerror}[1]{{L}_{#1}}

\newcommand \D {\mathrm{d}}
\newcommand{\cost}[1]{\rho}
\newcommand \indone {\mathbbm{1}}

\newcommand{\Phatn}{\ensuremath{\hat P_n}}

\newcommand{\fdiv}{\ensuremath{f}}
\newcommand{\ftil}{\ensuremath{f^*}}
\newcommand{\ftilg}{\ensuremath{(f^*)'}}
\newcommand{\ftilh}{\ensuremath{(f^*)''}}
\newcommand{\Df}[2]{D_\fdiv(#1 \Vert #2)}
\newcommand{\Dftil}[2]{D_{\ftil}(#1 \Vert #2)}

\newcommand{\ConstZ}{\ensuremath{C_0}}
\newcommand{\ConstZTil}{\ensuremath{C_0^*}}
\newcommand{\ConstI}{\ensuremath{C_1}}
\newcommand{\ConstITil}{\ensuremath{C_1^*}}
\newcommand{\ConstII}{\ensuremath{C_2}}
\newcommand{\ConstIITil}{\ensuremath{C_2^*}}

\DeclarePairedDelimiterX{\inp}[2]{\langle}{\rangle}{#1, #2} %
\DeclarePairedDelimiterX{\norm}[1]{\Vert}{\Vert}{#1} %
\DeclarePairedDelimiterX{\normsq}[1]{\Vert}{\Vert^2}{#1} %
\DeclarePairedDelimiter\abs{\lvert}{\rvert}

\newcommand{\Supp}[1]{\mathrm{Supp}(#1)}

\newcommand \argmin {\operatorname*{arg\,min}} %
\newcommand \argmax {\operatorname*{arg\,max}} %
\newcommand \dist {\operatorname*{dist}} %

\renewcommand \epsilon \varepsilon
\newcommand{\eps}{\ensuremath{\epsilon}}
\newcommand{\zipf}{\mathrm{Zipf}}
\newcommand{\dir}{\mathrm{Dir}}

\newcommand \e {e}

\newtheorem{theorem}{Theorem}
\newtheorem{lemma}[theorem]{Lemma}

\newtheorem{property}[theorem]{Property}
\newtheorem{proposition}[theorem]{Proposition}

\newtheorem{remark}[theorem]{Remark}
\newtheorem{definition}[theorem]{Definition}
\newtheorem{assumption}[theorem]{Assumption}
\newtheorem{example}[theorem]{Example}

\definecolor{C0}{HTML}{1F77B4}
\definecolor{C1}{HTML}{ff7f0e}
\definecolor{C2}{HTML}{2ca02c}
\definecolor{C3}{HTML}{d62728}
\definecolor{C4}{HTML}{9467bd}
\definecolor{keynoteBlue}{HTML}{01A2FF}

\definecolor{cadmiumgreen}{rgb}{0.0, 0.42, 0.24}
\definecolor{oldmauve}{rgb}{0.4, 0.19, 0.28}
\definecolor{royalazure}{rgb}{0.0, 0.22, 0.66}
\definecolor{harvardcrimson}{rgb}{0.79, 0.0, 0.09}
\definecolor{lightmauve}{rgb}{0.86, 0.82, 1.0}
\definecolor{darkbrown}{rgb}{0.4, 0.26, 0.13}
\definecolor{darkpink}{rgb}{0.91, 0.33, 0.5}

\newcommand{\tabemph}[1]{\cellcolor{lightmauve!30}\textcolor{black!50!royalazure}{#1}}%

\newcommand{\mauve}{{\fontfamily{bch}\selectfont{\textsc{Mauve}}}\xspace}
\newcommand{\mystyle}[1]{{\fontfamily{bch}\selectfont{\textsc{#1}}}\xspace}
\newcommand{\name}{\mauve}

\newcommand\blfootnotenew[1]{%
    \bgroup
    \renewcommand\thefootnote{\fnsymbol{footnote}}%
    \renewcommand\thempfootnote{\fnsymbol{mpfootnote}}%
    \footnotetext[0]{#1}%
    \egroup
}

\renewcommand{\cite}{\citep}

\usepackage{lastpage}
\jmlrheading{24}{2023}{1-\pageref{LastPage}}{1/23}{10/23}{23-0023}{Krishna Pillutla, Lang Liu,  John Thickstun, Sean Welleck, Swabha Swayamdipta, Rowan Zellers, Sewoong Oh, Yejin Choi, Zaid Harchaoui}
\ShortHeadings{MAUVE Scores for Generative Models: Theory and Practice}{Pillutla, Liu,  Thickstun, Welleck, Swayamdipta, Zellers, Oh, Choi, Harchaoui}

\firstpageno{1}

\begin{document}
\title{\textbf{MAUVE Scores for Generative Models:}  \textbf{Theory and Practice}}

\author{\name Krishna Pillutla$^{1*}$ \email pillutla@cs.washington.edu 
       \AND
\name Lang Liu$^{2*}$ \email liu16@uw.edu 
       \AND
\name John Thickstun$^{3}$ \email jthickstun@stanford.edu 
    \AND 
\name Sean Welleck$^{1, 4}$ \email wellecks@cs.washington.edu  
    \AND 
\name Swabha Swayamdipta$^{5}$ \email swabhas@usc.edu
    \AND 
\name Rowan Zellers$^{6}$ \email rowanz@cs.washington.edu 
    \AND 
\name Sewoong Oh$^{1, 7}$ \email sewoong@cs.washington.edu 
    \AND
\name Yejin Choi$^{4, 7}$ \email yejin@cs.washington.edu 
    \AND 
\name Zaid Harchaoui$^2$ \email zaid@uw.edu \\
\addr $^1$Google Research \\
\addr $^2$Department of Statistics, University of Washington  \\
\addr $^3$Department of Computer Science, Stanford University\\
\addr $^4$Allen Institute for Artificial Intelligence \\
\addr $^5$Viterbi School of Engineering, University of Southern California \\
\addr $^6$OpenAI \\
\addr $^7$Paul G. Allen School of Computer Science and Engineering, University of Washington \\
       }

\editor{Kilian Weinberger}

\appto\appendix{\addtocontents{toc}{\protect\setcounter{tocdepth}{0}}}

\maketitle

\blfootnotenew{
$^*$These authors contributed equally to this work.  
}

\begin{abstract}%
Generative artificial intelligence has made significant strides, producing text indistinguishable from human prose and remarkably photorealistic images.
Automatically measuring how close the generated data distribution is to the target distribution is central to diagnosing existing models and developing better ones.
We present \mauve, a family of comparison measures between pairs of distributions such as those encountered in the generative modeling of text or images. These scores are statistical summaries of divergence frontiers capturing two types of errors in generative modeling.
We explore three approaches to statistically estimate these scores: vector quantization, non-parametric estimation, and classifier-based estimation. We provide statistical bounds for the vector quantization approach.

Empirically, we find that the proposed scores paired with a range of $f$-divergences and statistical estimation methods can 
quantify the gaps between the distributions of human-written text and those of modern neural language models by correlating 
with human judgments and identifying known properties of the generated texts.
We demonstrate in the vision domain that \mauve can identify known properties of generated images on par with or better than existing metrics.
In conclusion, we present practical recommendations for using \mauve effectively with language and image modalities.
 \end{abstract}


\setcounter{tocdepth}{2}
{\small
\tableofcontents
}
\clearpage

\section{Introduction} \label{sec:intro}
Large-scale generative artificial intelligence models show an ability to produce human-like text and realistic images. Recent chatbots such as ChatGPT/GPT-4~\cite{openai2023gpt4}, Bard~\cite{bard}, and Ernie Bot~\cite{sun2021ernie} have rapidly gained wide prominence in the general public for their articulate responses across many topics and styles.
More generally, 
large language models such as
Llama-2~\cite{touvron2023llama},
Falcon~\cite{falcon40b},
Bloom~\cite{workshop2022bloom},
and Mistral~\cite{jiang2023mistral},
as well as image and multi-modal generative models such as 
Stable Diffusion~\cite{rombach2022high}, 
Imagen~\cite{saharia2022photorealistic}, and
CM3leon~\cite{yu2023scaling}
can produce original content in response to queries in the form of blog posts, poetry, computer programs, and artwork.

However, evaluating the distributions captured by such large-scale generative models requires substantial effort.
Automatic measures can dramatically reduce the cost of evaluation, in turn making it easier to rapidly develop models, choose hyperparameters, and understand a model’s capabilities.

One approach to evaluation is to compare a generative model’s distribution $Q$ with the target distribution $P$ of the real data that it aims to model.
Doing so requires considering two types of errors: (I) the mass of $Q$ that has a low probability under $P$ where the model produces unrealistic or degenerate data, and (II) the mass of $P$ that has a low probability under $Q$ where the model is not able to produce some class of realistic data. However, quantifying these errors in a principled, computationally tractable manner is challenging when faced with real-world text or image distributions.

We present a family of comparison measures between pairs of probability distributions, such as those encountered in the generative modeling of text and images. 
Building upon the notion of divergence frontiers proposed by \citet{djolonga2020precision},
our measures are statistical summaries of \textit{$f$-divergence frontiers}, which capture the two types of errors. 
We explore three methods for estimating these divergence frontiers and their scalar summaries. We provide statistical bounds for two of these estimation methods---vector quantization and nearest-neighbor estimation---as well as 
theoretical guidance on choosing the level of vector quantization. In the spirit of popular metrics in natural language processing such as \mystyle{Bleu}~\cite{papineni2002bleu} and \mystyle{Rouge}~\cite{lin2004rouge}, we call these measures \mauve scores. 

We develop the scores in practice for open-ended text generation. We find that, for a range of $f$-divergences and estimation methods, these measures quantify the gap between the distributions of human-written text and those of modern neural language models efficiently and robustly. 
Moreover, we show that these measures extend to image distributions, aligning well with the widely used Fr\'echet distance in the computer vision domain in quantifying the effect of sampling algorithms and architectural improvements. 
Together, our theoretical and empirical analyses demonstrate that \mauve provides a principled, effective, and powerful recipe for comparing distributions of complex high-dimensional text and images.

\subsection{Contributions}
We make the following contributions in this work.

\myparagraph{Statistical Summaries of Divergence Frontiers (\Cref{sec:mauve})}
Our goal is to provide a scalar summary of the discrepancy between a generative model $Q$ and the target distribution $P$ that it aims to model. To do so, following \citet{djolonga2020precision}, we consider two types of costs: (I) the mass of $Q$ that has low probability under $P$, and (II) the mass of $P$ that has low probability under $Q$.
We formalize these costs using a \textit{divergence frontier},
\begin{align*}
 \Fcal_f(P, Q) = 
        \Big\{
        \big(\Df{P}{R_\lambda}, \,  \Df{Q}{R_\lambda}\big) \,:\,
        \lambda \in (0, 1)
        \Big\}, 
\end{align*}
where $R_\lambda = \lambda P + (1-\lambda) Q$, and $D_f$ is an $f$-divergence
such as the Kullback–Leibler (KL) divergence.
See \Cref{fig:main:illustration} for an illustration.
This extends the frontiers of \citep{djolonga2020precision} to general $f$-divergences.
We shall show in \Cref{sec:mauve} that the nice properties of the divergence frontiers also extend to their variants based on $f$-divergences.

\begin{figure}[t]
    \centering
    \adjustbox{min width=0.48\textwidth}{%
\begin{tikzpicture}[scale=1]
\shade[inner color=gray!15, outer color=gray!50] 
  (1,0) 
  to[out=50,in=105] (4, 0.25) 
  to[out=90,in=55] (0, 1)  
  to[out=50,in=100] (1, 0);

\draw[thick,color=gray] 
    (0.9, 0.75) to[out=65, in=205]
    (1.7, 1.5) to[out=15, in=175]
    (2.7, 1.5) ;

\draw[color=C0, fill,rotate around={65:(0.9,0.75)}] (0.9,0.75) ellipse (1.8pt and 0.84pt) node[label={[shift={(0.1,0.15)}, text=black]270:{\scalebox{0.7}{$P$}}}] (p_manif) {};
\draw[color=C1, fill,rotate around={0:(1.7,1.5)}] (1.7,1.5) ellipse (1.2pt and 0.6pt) node[label={[shift={(0.1,0.17)}, text=black]270:{\scalebox{0.7}{$R_\lambda$}}}] (rl_manif) {};
\draw[color=C2, fill,rotate around={-25:(2.7,1.5)}] (2.7,1.5) ellipse (0.84pt and 0.48pt) node[label={[shift={(-0.04,0.18)}, text=black]270:{\scalebox{0.7}{$Q$}}}] (q_manif) {};

\draw[color=black, fill,rotate around={320:(3.7,0.87)}] (3.7,0.87) ellipse (1.2pt and 0.6pt) node[label={[shift={(-0.1,0)}, text=black]0:{\scalebox{0.7}{$R'$}}}] (r1) {};

\def\MixtureOfGaussianP{\x, {
    0.4 * exp( -((\x-0.15)^2)/ (2 * 0.04^2) ) + 
    0.59 * exp( -((\x-0.45)^2)/ (2 * 0.10^2) ) + 
    0.01 * exp( -((\x-0.83)^2)/ (2 * 0.05^2) )
}}
\def\MixtureOfGaussianQ{\x, {
    0.01 * exp( -((\x-0.20)^2)/ (2 * 0.07^2) ) + 
    0.6 * exp( -((\x-0.4)^2)/ (2 * 0.08^2) ) + 
    0.39 * exp( -((\x-0.75)^2)/ (2 * 0.10^2) ) 
}}
\def\MixtureOfGaussianR{\x, {
    0.4 * (
        0.4 * exp( -((\x-0.15)^2)/ (2 * 0.04^2) ) + 
        0.59 * exp( -((\x-0.45)^2)/ (2 * 0.10^2) ) + 
        0.01 * exp( -((\x-0.83)^2)/ (2 * 0.05^2) )
    ) + 
    0.6 * (
        0.01 * exp( -((\x-0.20)^2)/ (2 * 0.07^2) ) + 
        0.6 * exp( -((\x-0.4)^2)/ (2 * 0.08^2) ) + 
        0.39 * exp( -((\x-0.75)^2)/ (2 * 0.10^2) ) 
    )
}}

\node(pbox) at (0.05, 2.4) {
    \begin{tikzpicture}[scale=1]
    \draw[color=gray!60] (-0.1,-0.1) rectangle (1.24,0.74);
    \draw[color=C0, preaction={fill=C0!60, fill opacity=0.2}, fill opacity=0.2, pattern color=blue, domain=0:1,smooth] (0, 0) -- plot (\MixtureOfGaussianP) -- (1, 0);
    \draw (0,0) -- (1.2,0) {};
    \draw (0,0) -- (0,0.7) {};
    \node at (0.96, 0.6) {{\scalebox{0.7}{$P$}}};
    \end{tikzpicture}
};

\node(rbox) at (1.7, 2.4) {
    \begin{tikzpicture}[scale=1]
    \draw[color=gray!60] (-0.1,-0.1) rectangle (1.24,0.74);
    \draw[color=C1, preaction={fill=Thistle!60, fill opacity=0.2}, fill opacity=0.2, pattern color=blue, domain=0:1,smooth] (0, 0) -- plot (\MixtureOfGaussianR) -- (1, 0);
    \draw[-] (0,0) -- (1.2,0) {};
    \draw[-] (0,0) -- (0,0.7) {};
    \node at (0.96, 0.6) {{\scalebox{0.7}{$R_\lambda$}}};
    \end{tikzpicture}
};

\node(qbox) at (3.35, 2.4) {
    \begin{tikzpicture}[scale=1]
    \draw[color=gray!60] (-0.1,-0.1) rectangle (1.24,0.74);
    \draw[color=C2, preaction={fill=C2!60, fill opacity=0.2}, fill opacity=0.2, pattern color=blue, domain=0:1,smooth] (0, 0) -- plot (\MixtureOfGaussianQ) -- (1, 0);
    \draw (0,0) -- (1.2,0) {};
    \draw (0,0) -- (0,0.7) {};
    \node at (0.96, 0.6) {{\scalebox{0.7}{$Q$}}};
    \end{tikzpicture}
};

\draw[-stealth, bend left=60, color=C0] (p_manif) to (pbox.240) ;
\draw[-stealth, bend left, color=C1] (rl_manif.180) to ([yshift=0.2]rbox.250) ;
\draw[-stealth, bend right=70, color=C2] (q_manif) to (qbox.290) ;

\end{tikzpicture} %
}
    \hfill
    \adjustbox{max width=0.48\textwidth}{%
\begin{tikzpicture}

\def\divcurve{\x, {
    (1 / \x)
    }}

\draw[color=white, preaction={fill=oldmauve!40, fill opacity=0.2}, fill opacity=0.2, pattern color=Maroon!30, pattern=fivepointed stars, domain=0.34:5.9,smooth] (5.9, 3) --  plot (\divcurve) -- (5.9, 3) ;

\draw[color=white, fill=Goldenrod!15, domain=0.34:5.9,smooth] (0.2, 0) --  (0.2, 3) -- plot (\divcurve) --  (5.9, 0) -- cycle ;

\draw[very thick, color=black, domain=0.34:5.9,smooth] plot (\divcurve) ;

\draw[color=C1, fill] (1.5, 0.667) circle (2pt) node[label={[text=black, xshift=-0.25cm, yshift=-0.1cm]80:$R=R_\lambda$}] {};
\draw[color=C2, fill] (5.7, 0.175) circle (2pt) node[label={[text=black, xshift=-0.3cm, yshift=-0.01cm]90:$R\to Q$}] {};
\draw[color=C0, fill] (0.357, 2.8) circle (2pt) node[label={[text=black, xshift=-0.08cm, yshift=0cm]0:$R\to P$}] {};
\draw[color=black, fill] (5.4, 2.3) circle (2pt) node[label={[text=black, xshift=-0.08cm, yshift=-0.05cm]90:$R = R'$}] {};

\draw[thick,->] (0.2,0) -- (6,0) node[below, xshift=-3cm] {$\mathrm{KL}(P \Vert R)$};
\draw[thick,->] (0.2,0) -- (0.2,3) node[left, rotate=90, xshift=-0.6cm, yshift=0.23cm] {$\mathrm{KL}(Q \Vert R)$};

\end{tikzpicture} %
}
    \caption{ \small
    \textbf{Left}:
    Comparing two distributions $P$ and $Q$. Here, $R_\lambda = \lambda P + (1-\lambda)Q$ is the interpolation between $P$ and $Q$ for $\lambda \in (0, 1)$ and $R'$ denotes some arbitrary distribution.
    \textbf{Right}: The corresponding divergence frontier (black curve) between $P$ and $Q$. The interpolations $R_\lambda$ for $\lambda \in (0, 1)$ make up the frontier, while all other distributions such as $R'$ must lie above the frontier. 
    }
    \label{fig:main:illustration}
\end{figure}
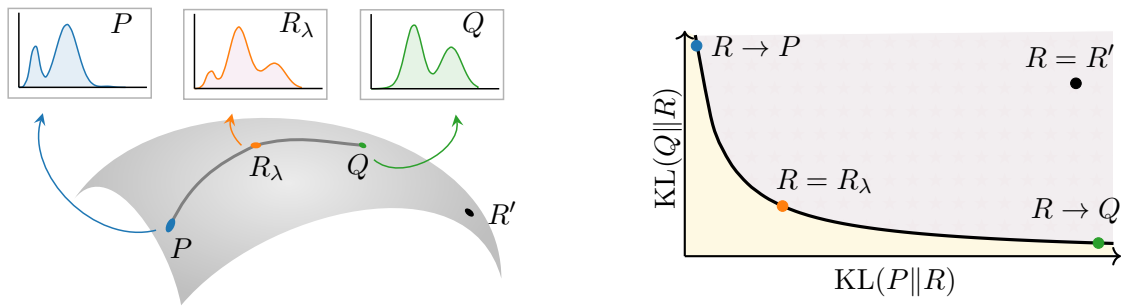

We propose three scalar statistical summaries of divergence frontiers. 
The first summary measures the area under a transformed divergence frontier:
\begin{align*} 
    \mauve_f(P, Q) = \mathrm{AUC}\left(
        \left\{
        \big(\exp(- x), \, \exp(- y)\big) \,:\, (x, y) \in \Fcal_f(P, Q)
        \right\} \cup \{(1, 0), (0, 1) \}
    \right) \,.
\end{align*}
Here, $\exp(\cdot)$ monotonically transforms the frontier to account for unbounded divergences.
Second, we consider an integral summary that sweeps over the coordinates on the divergence frontier and accumulates their costs:
\begin{align*}
    \mray_f(P, Q) := 2\int_0^1 \big(\lambda \,\Df{P}{R_\lambda} + (1 - \lambda) \Df{Q}{R_\lambda}\big) \, \D\lambda.
\end{align*}

Finally, the third summary simply uses costs from
the mid-point of the frontier, i.e., the coordinates corresponding to $\lambda={1}/{2}$:
\begin{align*}
    \midp_f(P,Q):= \frac12 \,\Df{P}{R_{1/2}} + \frac12 \Df{Q}{R_{1/2}} \,.
\end{align*}

\noindent At their core, all three summaries are based on $f$-divergences. 
Thus, all three benefit from our estimation algorithms and error bounds for $f$-divergences, which we discuss next.

\myparagraph{Statistical Estimation Algorithms (\Cref{sec:compute})}
We give algorithms for computing the summaries $\mauve_f$, $\mray_f$, and $\midp_f$ on real-world
distributions of text or images. 
This requires computing $f$-divergences between the target distribution $P$ and the model distribution $Q$, which is challenging due to the lack of direct access to $P$ and $Q$, and the large support of each distribution. 
To address these challenges, we propose three methods for estimating divergence frontiers from i.i.d. samples
using embeddings of the data (e.g., from a large language model for text data):
\begin{enumerate}
\item \textit{Quantization}: we jointly quantize the distributions $P$ and $Q$ in some embedding space to form two multinomial distributions, then estimate the divergence frontier between the two multinomial distributions.
\item \textit{Nearest-neighbor}: we use the nearest neighbors (in some embedding space) of each sample to estimate the likelihood ratio $P(x)/Q(x)$, which we use to estimate the required $f$-divergences.
\item \textit{Classifier}: we train a classifier to identify whether each sample belongs to the target or model distribution. We use the classifier to estimate the likelihood ratio and, in turn, the required $f$-divergences.
\end{enumerate}

\myparagraph{Error Bounds}
We develop error bounds for the first quantization approach.
The total estimation error of the divergence frontier consists of two parts: (i) the statistical error in estimating the frontier from samples, and (ii) the quantization error that arises from passing from the original distributions to their quantized versions.

For the statistical error, \Cref{thm:fdiv:consistency} gives an error bound that allows for long tails and countable support of the distribution $P$. This improves over a naive bound that 
does not allow for distributions with long tails, and requires finite support. 
A key technique that enables this result is considering the \textit{missing mass}~\citep{good1953frequency}: the total probability that does not appear in the finite sample used to estimate the frontier.
When the two distributions $P$ and $Q$ intersect on a finite set of $k$ elements, the bounds simplify further. 
For example, we give the following statistical error bound on the integral summary (Eq.~\ref{eq:tail_bound_ray}):
\begin{align*}
    \expect\abs{\fint(\hat P_n, \hat Q_n) - \fint(P, Q)} \le \tilde O \left( \sqrt{\frac{k}{n}} + \frac{k}{n} \right)\,,
\end{align*}
where $\hat P_n$ and $\hat Q_n$ are the empirical estimators and $n$ is the number of samples.
We give a similar bound for general $f$-divergences (Eq.~\ref{eq:fdiv:stat_error}).
Our results hold under assumptions that are satisfied by many common $f$-divergences (\Cref{table:fdiv}).
To improve the statistical performance of empirical estimators when the quantization size $k$ is large, we also apply \emph{add-constant} smoothing to estimate the two distributions---we add a small constant $b > 0$ to the counts of each bin and normalize them to form a distribution.
We prove in \Cref{thm:addb:fdiv:consistency} a statistical error bound for the add-constant estimators.
Applied to the integral summary, the bound is (Eq.~\ref{eq:addb:tail_bound_ray})
\begin{align*}
    \expect\abs{\fint(\hat P_n^b, \hat Q_n^b) - \fint(P, Q)} \le \tilde O\left( \frac{\sqrt{kn} + kb}{n + kb} \right) \,,
\end{align*}
where $\hat P_{n}^b$ and $\hat Q_n^b$ are the add-constant estimators.
A similar bound for general $f$-divergences is given in Eq.~\ref{eq:addb:fdiv:stat_error}.

For the quantization error, we show that there exists a quantization scheme with error $O(1/k)$, where $k$ is the size of the $k$-partition used to quantize the sample space.
Our analysis is inspired by the 
asymptotic approximation of an $f$-divergence with increasingly finer partitions (\citet{gyorfi1978fdiv}, Theorem 6).
Combining the statistical and quantization error bounds gives us a bound on the total error of the integral summary (Eq.~\ref{eq:est_error_ray_bound}):
\begin{align*}
    \expect\abs{\fint(\hat P_{\Scal_k, n}, \hat Q_{\Scal_k, n}) - \mray(P, Q)} \le \tilde O \left( \sqrt{\frac{k}{n}} + \frac{k}{n} + \frac{1}{k} \right).
\end{align*}

We discuss how to operationalize the nonparametric nearest-neighbor estimation with dimensionality reduction via principal component analysis (PCA). For nearest-neighbor estimation, we discuss bounds from \citet{noshad2017direct} (\Cref{thm:mauve:knn}).

\myparagraph{Experiments (\Cref{sec:experiments})}  
Our experiments are organized into multiple parts, mainly focusing on the open-ended text generation setting. 

We start by analyzing the effectiveness of the proposed measure for comparing text distributions. We focus on the area summary using the KL divergence computed with vector quantization.
We demonstrate that the proposed measures correlate with human quality judgments (\Cref{sec:expt:human-eval}) and quantify known properties of generated text (\Cref{sec:expt:properties}). The main focus of the rest of the experimental study is to analyze the effects of each of the components of the evaluation pipeline: the estimation method, the choice of the divergence, and the choice of the embedding.

First, we consider different \textbf{estimation methods}: vector quantization, nearest neighbor estimation, and classifier-based estimation (\Cref{sec:expt:estimation}).
We also consider a popular parametric Gaussian approximation method---assuming that embedded samples from the target and model distributions are distributed according to multivariate Gaussians, we estimate the parameters of each Gaussian and estimate the divergence frontier by numerical integration (see \Cref{sec:a:parametric} for more details).
We find that all estimation methods identify expected quality trends and correlate with human evaluations.
However, nearest-neighbor and classifier-based estimation show a slightly decreased ability to identify good hyperparameter values, while parametric estimation requires extreme dimensionality reduction.
Thus, we recommend vector quantization as a default.

Second, we experiment with other \textbf{$f$-divergences and optimal transport costs} (\Cref{sec:expt:other-divergence}). Specifically, we compare different variants of the proposed measure based on (i) alternate $f$-divergences, (ii) other statistical summaries of the divergence frontier, and (iii) summaries of frontiers based on optimal transport distances.
We find that all the quantities based on $f$-divergences correlate perfectly.
On the other hand, some of the optimal transport distances fail to capture expected trends.
These results demonstrate the flexibility and effectiveness of our proposed measures.

Third, we perform a thorough exploration of the \textbf{effect of the embedding} in the evaluation pipeline (\Cref{sec:expt:embedding}).
Our experiments reveal that the embedding is crucial to the empirical success of \mauve. While most large language model embeddings (either a masked or a causal language model, including the model used to generate the text) and even shallow GloVe~\cite{pennington2014glove} embeddings yield useful comparison measures, we find that string kernel-based embeddings or embedding-free direct estimation methods fail to capture expected trends.

Finally, we demonstrate that our measures generalize to other AI domains beyond text. Specifically, we show that in the \textbf{image domain}, our measure recovers expected trends with respect to the sampling algorithm and model size, and correlates perfectly with the widely used Fr\'echet distance in this setting (\Cref{sec:expt:vision}).

\myparagraph{Previous Papers}
This work builds upon two previous shorter conference papers. The first \cite{pillutla2021mauve} introduces the area summary in the context of open-ended text generation and conducts an empirical study. 
The second \cite{liu2021divergence} studies the statistical theory behind estimating divergence frontiers with vector quantization and smoothed distribution estimators.
This work unifies both of these works and makes several further contributions. 

First, we introduce the notion of $f$-divergence frontiers and three scalar summaries, generalizing the area summary from \cite{pillutla2021mauve} and the integral summary from \cite{liu2021divergence}. We also systematically study the properties of the three summaries (\Cref{sec:mauve}). 
Second, we consider three estimation algorithms (\Cref{sec:compute}), based on nonparametric estimation, classifier-based estimation, and a parametric Gaussian approximation, and empirically compare their performance for open-ended text generation (\Cref{sec:expt:estimation}). 
Empirically, we perform a thorough exploration of alternatives based on $f$-divergences and optimal transport (\Cref{sec:expt:other-divergence}). We also probe the effect of the embedding (\Cref{sec:expt:embedding}), and 
perform experiments in the vision domain (\Cref{sec:expt:vision}), not covered in the previous two papers.
 
\section{Background and Setup} \label{sec:background}

We discuss the basics of open-ended text generation and set up the problem of comparing multiple generative models.

\subsection{Language Modeling and Open-Ended Text Generation}
\label{sec:bg:lms}
We start with neural autoregressive language models since these form the backbone of prevailing approaches to text generation. 

\myparagraph{Language Modeling}
Consider a sequence $\xv = (x_1, \cdots, x_{|\xv|})$ of natural language text, where each $x_i$ belongs to a finite vocabulary $V$ (e.g., characters or words). 
An autoregressive language model
$\hat P(\cdot \, |\, \xv_{1:t})$
models the conditional distribution over the next token $x_{t+1}$ following the sequence $\xv_{1:t}$.
While neural language models, i.e., language models parameterized by a neural network, date back to at least \cite{bengio2003neural,collobert2011natural},
contemporary models are based on the transformer architecture~\cite{vaswani2017attention} summarized in \Cref{fig:mauve:background:1} (left). 

The usual training objective for neural language modeling is via supervised multi-class classification of the next token. We assume that there is an underlying distribution $P(\cdot \, | \, \xv_{1:t})$ for the next token $x_{t+1}$ humans would write in continuation to a prefix $\xv_{1:t}$. 
The training procedure aims to minimize the Kullback-Liebler (KL) divergence between the distributions $P(\cdot\,|\,\xv_{1:t})$ and $\hat P(\cdot \,|\, \xv_{1:t})$ assigned by humans and the language model respectively over the next token $x_{t+1}$ in continuation to a context $\xv_{1:t} \sim P_t$ coming from \textit{human-written} text:
\begin{align} \label{eq:lm:objective}
    \min_\theta \, \, \expect_{t \sim \text{Unif}([T-1])}\expect_{\xv_{1:t} \sim P_t} \left[ \kl\Big(P(\cdot\, | \, \xv_{1:t}) \Big\Vert \hat P_\theta(\cdot\,|\, \xv_{1:t})  \Big) \right]\,,
\end{align}
where $T$ is the maximum sequence length. 
Since neither the distribution $P_t$ over prefixes of length $t$ nor the distribution $P(\cdot \, | \, \xv_{1:t})$ over the next token is known in practice, plug-in estimates of both are employed in practice. 

Autoregressive models also
yield an estimate of the joint probability $\hat P(\xv)$ 
of a sequence $\xv = (x_1, \cdots, x_{|\xv|})$ as 
\[
    \hat P(\xv) = \prod_{t=0}^{|\xv|-1} \hat P(x_{t+1} \, | \, \xv_{1:t}) \,.
\]

\myparagraph{Open-Ended Text Generation}
The open-ended text generation task asks us to output text $\hat{\xv}_{s+1:|\hat\xv|}$ in continuation of a given context $\xv_{1:s}$. In contrast to directed text generation tasks such as translation, summarization, and question-answering, the task here is open-ended in that the context size $s \ll |\hat\xv|$ is typically small and does not meaningfully constrain the output space.
Unlike directed text generation tasks such as translation, summarization, and question-answering, the goal here is to generate text that is coherent, fluent, creative, and engaging. Since these criteria are hard to make mathematically precise, we instead consider the surrogate goal of generating text which is \textit{human-like}, such that generated text samples can pass for samples from the distribution $P$ over human written text sequences. 

We model a text generation system as a probability distribution $Q(\cdot \, |\, \xv_{1:s})$ such that its generated text $\hat{\xv}_{s+1:|\hat\xv|}$ is an i.i.d. sample from $Q$. 
Given a neural autoregressive language model $\hat P$, we can generate open-ended text in a serial, left-to-right fashion, by sampling $\hat{x}_{s+1} \sim \hat P(\cdot|\xv_{1:s})$, $\hat{x}_{s+2} \sim \hat P(\cdot|\xv_{1:s}, \hat{x}_{s+1})$, etc.  This is also known as \emph{ancestral sampling}, and the induced distribution $Q$ over sequences is 
\[
    Q_{\text{samp}}(\xv_{1:s}, \hat\xv_{s+1:|\hat\xv|}) 
    = \prod_{t=1}^s P(x_{t} | \xv_{1:t-1}) \, 
        \prod_{t=s+1}^{|\hat\xv|} \hat P(\hat x_t | \xv_{1:s}, \hat \xv_{s+1:t-1}) \,,
\]
where we assume that the prefix $\xv_{1:s} \sim P_s$ is drawn from the human distribution.
Note that the distribution $Q_{\text{samp}}$ is identical to $\hat P$, expect for the prefix $\xv_{1:s}$. General decoding algorithms produce samples from a reshaped model distribution, as we discuss next.

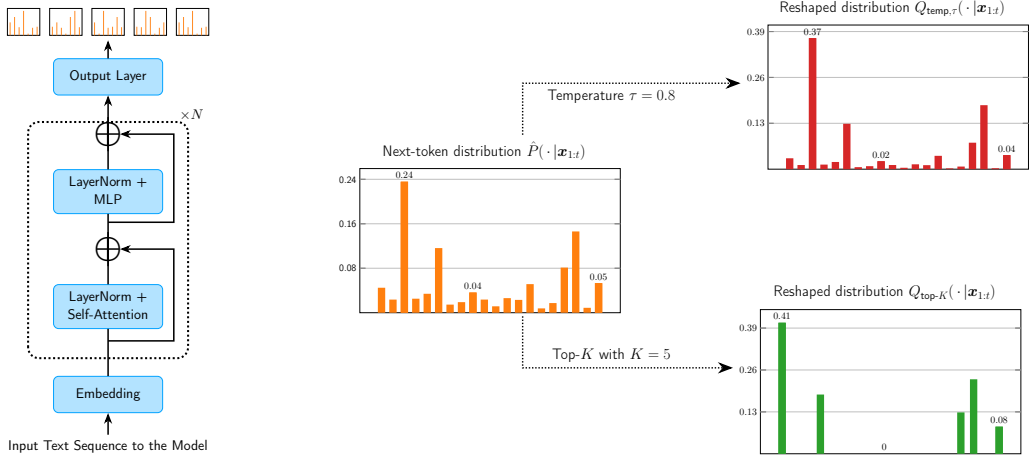
\begin{figure}[t]
    \centering
    \adjustbox{max width=0.2\textwidth}{%
\tikzstyle{startstop} = [rectangle, rounded corners, minimum width=3cm, minimum height=1cm,text centered, draw=keynoteBlue, fill=keynoteBlue!30, text=black]

\tikzstyle{myarrow} = [-{Stealth[length=3mm]}, very thick, solid]

\tikzstyle{myarrow2} = [ very thick, solid]

\begin{tikzpicture}
\fontfamily{cmss}\selectfont

\node (start) [startstop] {Embedding};

\node (text) [below=0.6cm of start] {Input Text Sequence to the Model};

\draw[myarrow](text.north) -- node  {} (start.south);

\node (temp) [above=0.7 cm of start]{} ; 

\node (ln1) [startstop, above = 1.3cm of start] {
\begin{tabular}{c} LayerNorm + \\ Self-Attention \end{tabular}
};

\draw[myarrow2](start.north) -- node  {} (ln1.south);

\node (op1) [above=0.5 cm of ln1] {\Huge $\bm{\oplus}$} ;

\draw[myarrow2](ln1.north) -- node  {} (op1.north);

\node (temp2) [above=0.0 cm of op1] {} ;

\node (ln2) [startstop, above = 0.5cm of op1] {
\begin{tabular}{c} LayerNorm + \\ MLP \end{tabular}
};

\draw[myarrow2](op1.south) -- node  {} (ln2.south);

\node (op2) [above=0.5 cm of ln2] {\Huge $\bm{\oplus}$} ;

\draw[myarrow2](ln2.north) -- node  {} (op2.north);

\node (output) [startstop, above=0.6cm of op2] {Output Layer};

\draw[myarrow](op2.south) -- node  {} (output.south);

\draw[rounded corners=0.3cm, ultra thick, dotted] (-2.2, 1) rectangle (2.2, 7.5) {};
\node at (2.3, 7.7)  {$\times N$}  ;

\draw[myarrow] (temp.north) -| ([xshift=0.3cm]ln1.east) |- ([xshift=-0.2cm] op1.east) {};

\draw[myarrow] (temp2.north) -| ([xshift=0.3cm]ln2.east) |- ([xshift=-0.2cm] op2.east) {};

\node[above = 0.5 cm of output](dist0) {
\begin{tikzpicture}[scale=0.13]
\begin{axis}[
    xtick={}, xticklabel={},
    xmajorticks=false, xminorticks=false,
    ytick={}, yticklabel={},
    ymajorticks=false, yminorticks=false,
    ymin=0,
]
    \addplot[
        ybar,
        bar width=5.5, 
        fill=C1,
        draw=none,
    ] 
    coordinates {
        (0, 0.0908845838997054)
        (1, 0.0117650716000647)
        (2, 0.17136551069477)
        (3, 0.0529825693155907)
        (4, 0.0604040844685279)
        (5, 0.12861509626331)
        (6, 0.0570979026495167)
    };
\end{axis}
\end{tikzpicture}
};

\node[right = 0.0 cm of dist0](distp1) {
\begin{tikzpicture}[scale=0.13]
\begin{axis}[
    xtick={}, xticklabel={},
    xmajorticks=false, xminorticks=false,
    ytick={}, yticklabel={},
    ymajorticks=false, yminorticks=false,
    ymin=0,
]
    \addplot[
        ybar,
        bar width=5.5, 
        fill=C1,
        draw=none,
    ] 
    coordinates {
        (6, 0.0908845838997054)
        (4, 0.0117650716000647)
        (2, 0.17136551069477)
        (5, 0.0529825693155907)
        (0, 0.0604040844685279)
        (3, 0.12861509626331)
        (1, 0.0570979026495167)
    };
\end{axis}
\end{tikzpicture}
};

\node[right = 0.0 cm of distp1](distp2) {
\begin{tikzpicture}[scale=0.13]
\begin{axis}[
    xtick={}, xticklabel={},
    xmajorticks=false, xminorticks=false,
    ytick={}, yticklabel={},
    ymajorticks=false, yminorticks=false,
    ymin=0,
]
    \addplot[
        ybar,
        bar width=5.5, 
        fill=C1,
        draw=none,
    ] 
    coordinates {
        (0, 0.0908845838997054)
        (4, 0.0117650716000647)
        (3, 0.17136551069477)
        (5, 0.0529825693155907)
        (6, 0.0604040844685279)
        (1, 0.12861509626331)
        (2, 0.0570979026495167)
    };
\end{axis}
\end{tikzpicture}
};

\node[left = 0.0 cm of dist0](distm1) {
\begin{tikzpicture}[scale=0.13]
\begin{axis}[
    xtick={}, xticklabel={},
    xmajorticks=false, xminorticks=false,
    ytick={}, yticklabel={},
    ymajorticks=false, yminorticks=false,
    ymin=0,
]
    \addplot[
        ybar,
        bar width=5.5, 
        fill=C1,
        draw=none,
    ] 
    coordinates {
        (1, 0.0908845838997054)
        (3, 0.0117650716000647)
        (5, 0.17136551069477)
        (2, 0.0529825693155907)
        (0, 0.0604040844685279)
        (4, 0.12861509626331)
        (6, 0.0570979026495167)
    };
\end{axis}
\end{tikzpicture}
};

\node[left = 0.0 cm of distm1](distm2) {
\begin{tikzpicture}[scale=0.13]
\begin{axis}[
    xtick={}, xticklabel={},
    xmajorticks=false, xminorticks=false,
    ytick={}, yticklabel={},
    ymajorticks=false, yminorticks=false,
    ymin=0,
]
    \addplot[
        ybar,
        bar width=5.5, 
        fill=C1,
        draw=none,
    ] 
    coordinates {
        (0, 0.0908845838997054)
        (4, 0.0117650716000647)
        (3, 0.17136551069477)
        (5, 0.0529825693155907)
        (6, 0.0604040844685279)
        (1, 0.12861509626331)
        (2, 0.0570979026495167)
    };
\end{axis}
\end{tikzpicture}
};

\draw[myarrow](output.north) -- node  {} (dist0.south);

\end{tikzpicture} %
}
    \hfill
    \adjustbox{max width=0.7\textwidth}{%
\begin{tikzpicture}
\fontfamily{cmss}\selectfont

\node(mainbox) {
\begin{tikzpicture}
\fontfamily{cmss}\selectfont
\begin{axis}[
    title={\LARGE Next-token distribution $\hat P(\,\cdot\, | \xv_{1:t})$ \qquad \qquad \qquad \qquad\hspace{20pt}},
    xtick={},
    xticklabel={},
    xmajorticks=false,
    xminorticks=false,
    ymin=0,
    ytick={0.08, 0.16, 0.24},
    yticklabels={0.08, 0.16, 0.24},
    extra y ticks={0.08, 0.16, 0.24},
    extra y tick labels={},
    extra tick style={grid=major,},
    xscale=1.5,
]
    \addplot[
        ybar,
        bar width=5.5, 
        fill=C1,
        draw=none,
    ] 
    coordinates {
        (0,  0.0449123809481277)
        (1,  0.0236012293718441)
        (2,  0.235725168253162)
        (3,  0.0252025997684605)
        (4,  0.0340682363186177)
        (5,  0.116251761088196)
        (6,  0.014575328324964)
        (7,  0.0189915355209562)
        (8,  0.0363560253966694)
        (9,  0.0236209090811071)
        (10,  0.0115081262807355)
        (11,  0.0262139969747345)
        (12,  0.0230577008467274)
        (13,  0.0514143302095352)
        (14,  0.00776557228699075)
        (15,  0.0173491132898263)
        (16,  0.081167623086079)
        (17,  0.14617430341613)
        (18,  0.00882605048803759)
        (19,  0.0532180090490988)
    };

    \addplot[
        ybar,
        bar width=5.5, 
        fill=C1,
        draw=none,
        nodes near coords,
        nodes near coords
        style={/pgf/number format/fixed, /pgf/number format/precision=2}
    ] 
    coordinates {
        (2,  0.235725168253162)
        (8,  0.0363560253966694)
        (19,  0.0532180090490988)
    };
\end{axis}
\end{tikzpicture}
}; %

\node[above right = -1.5cm and 1.1cm of mainbox](tempbox) {
\begin{tikzpicture}
\fontfamily{cmss}\selectfont
\begin{axis}[
    title={\LARGE Reshaped distribution $Q_{\text{temp}, \tau}(\,\cdot\, | \xv_{1:t})$ \qquad \qquad \qquad \qquad\hspace{20pt}},
    xtick={},
    xticklabel={},
    xmajorticks=false,
    xminorticks=false,
    ymin=0,
    ytick={0.13, 0.26, 0.39},
    extra y ticks={0.13, 0.26, 0.39},
    extra y tick labels={},
    extra tick style={grid=major,},
    xscale=1.5,
]
    \addplot[
        ybar,
        bar width=5.5, 
        fill=C3,
        draw=none,
    ] 
    coordinates {
        (0, 0.0308845838997054)
        (1, 0.0117650716000647)
        (2, 0.37136551069477)
        (3, 0.0129825693155907)
        (4, 0.0204040844685279)
        (5, 0.12861509626331)
        (6, 0.00570979026495167)
        (7, 0.00849244526105692)
        (8, 0.0224935101579097)
        (9, 0.011779789993146)
        (10, 0.00400589502458966)
        (11, 0.0137718569671059)
        (12, 0.0113610020056989)
        (13, 0.037828481628798)
        (14, 0.00222051052584363)
        (15, 0.00741495569683788)
        (16, 0.0750355168723325)
        (17, 0.181342327459978)
        (18, 0.00269055731047729)
        (19, 0.0398364445893047)
    };
    \addplot[
        ybar,
        bar width=5.5, 
        fill=C3,
        draw=none,
        nodes near coords,
        nodes near coords
        style={/pgf/number format/fixed, /pgf/number format/precision=2}
    ] 
    coordinates {
        (2, 0.37136551069477)
        (8, 0.0224935101579097)
        (19, 0.0398364445893047)
    };
\end{axis}
\end{tikzpicture}
};

\node[below right = -1.5cm and 0.8cm of mainbox](kbox) {
\begin{tikzpicture}
\fontfamily{cmss}\selectfont
\begin{axis}[
    title={\LARGE Reshaped distribution $Q_{\text{top-}K}(\,\cdot\, | \xv_{1:t})$ \qquad \qquad \qquad \qquad\hspace{20pt}},
    xtick={},
    xticklabel={},
    xmajorticks=false,
    xminorticks=false,
    ymin=0,
    ytick={0.13, 0.26, 0.39},
    extra y ticks={0.13, 0.26, 0.39},
    extra y tick labels={},
    extra tick style={grid=major,},
    xscale=1.5,
]
    \addplot[
        ybar,
        bar width=5.5, 
        fill=C2,
        draw=none,
    ] 
    coordinates {
        (2, 0.40666292411529)
        (5, 0.183786538841373)
        (16, 0.128320778742046)
        (17, 0.231092148978438)
        (19, 0.0841342410266148)
    };

        \addplot[
        ybar,
        bar width=5.5, 
        fill=C2,
        draw=none,
        nodes near coords,
        nodes near coords
        style={/pgf/number format/fixed, /pgf/number format/precision=2},
    ] 
    coordinates {
        (2, 0.40666292411529)
        (10, 0)
        (19, 0.0841342410266148)
    };
\end{axis}
\end{tikzpicture}
};

\draw[-{Stealth[length=6mm]}, ultra thick, dotted](mainbox.north) |- node[xshift=3.5cm, yshift=-0.5cm]  {\LARGE Temperature $\tau=0.8$}(tempbox.west);

\draw[-{Stealth[length=6mm]}, ultra thick, dotted](mainbox.south) |- node[xshift=3.5cm, yshift=0.5cm] {\LARGE Top-$K$ with $K=5$}(kbox.west);

\end{tikzpicture} %
}
    \caption{ \small
    \textbf{Left}: The transformer architecture takes in a text sequence $\xv=(x_1, \ldots, x_{|\xv|})$ and outputs the next-token distribution $\hat P(\,\cdot \, | \xv_{1:t})$ for each prefix $\xv_{1:t}$.  \textbf{Right}: Illustration of how decoding algorithms (specifically, temperature rescaling and top-$K$ decoding) reshape the model's next-token distribution.} 
    \label{fig:mauve:background:1}
\end{figure}

\myparagraph{Decoding Algorithms}
Assuming the language model learning has succeeded, we have that 
$\hat P(\cdot \,|\, \xv_{1:t}) \approx P(\cdot \,|\, \xv_{1:t})$ for prefixes $\xv_{1:t} \sim P_t$ drawn from the distribution of human-written text, in the sense that 
the objective of \eqref{eq:lm:objective} is bounded above by some $\eps > 0$. 
However, for $\hat\xv_{1:t}$ drawn from a distribution $Q_t$ which is different from the human distribution $P_t$, 
the model's next-token distribution $\hat P(\cdot \,|\, \hat\xv_{1:t})$ can be quite different from $P(\cdot \,|\, \hat\xv_{1:t})$ of humans.
In the iterative process of ancestral sampling, the gap between
$P(\hat\xv_{1:t})$ and $Q_{\text{samp}}(\hat \xv_{1:t})$ keep increasing as the generation length $t$ grows larger, so that $Q_{\text{samp}}$ is quite far from $P$. This leads to \textit{decoding algorithms} which produce samples
\[
    \hat x_{t+1} \sim Q(\cdot \,|\, \xv_{1:s}, \hat\xv_{s+1:t}) \,,
\]
where $Q(\cdot \,|\, \xv_{1:t})$ is a reshaping of the language model $\hat P(\cdot \,|\, \xv_{1:t})$ in order to promote more conservative outputs. We now define a few popular decoding algorithms; see also \Cref{fig:mauve:background:1} (right) for an illustration.

\textit{Temperature rescaling}~\cite{ackley1985learning} applies to language models parameterized with a softmax function: 
\[
    \hat P(x_{t+1} \, | \, \xv_{1:t}) = \frac{\exp\big(\phi(x_{t+1} |  \xv_{1:t})\big)}{\sum_{x \in V} \exp\big(\phi(x | \xv_{1:t})\big)} \,,
\]
for some unnormalized scoring function $\phi( \cdot \,|\, \xv_{1:t}) : V \to \reals$. 
This decoding algorithm rescales the term inside the exponential with a ``temperature'' parameter $\tau > 0$: 
\[
    Q_{\text{temp}, \tau}(x_{t+1} \, | \, \xv_{1:t}) = \frac{\exp\left(\frac{1}{\tau}\phi(x_{t+1} | \xv_{1:t}) \right)}{\sum_{x_{t+1}' \in V} \exp\left(\frac{1}{\tau}\phi(x_{t+1}' | \xv_{1:t}) \right)} \,.
\]
When $\tau < 1$, the distribution $Q_{\text{temp}, \tau}(\cdot \, | \, \xv_{1:t})$ becomes more peaked around the most likely next tokens, making the distribution more conservative. 

For an integer $K < |V|$, \textit{top-$K$ sampling}~\cite{fan2018heirarchical} applies the transformation
\[
    Q_{\text{top-}K}(x_{t+1} | \xv_{1:t}) = 
    \begin{cases}
        \frac{1}{Z} \, \hat P(x_{t+1} | \xv_{1:t}) \,, & \text{ if } x_{t+1} \in V_{\text{top-}K}, \\
        0 \,, & \text{ else},
    \end{cases}
\] 
where $Z$ is a normalizing constant, and $V_{\text{top-}K} = \{ z_{(1)}, \cdots, z_{(K)}  \} \subset V$
is the set of the $K$ highest scoring tokens satisfying 
\[
    \hat P(z_{(1)} | \xv_{1:t}) \ge \cdots \ge \hat P(z_{(K)} | \xv_{1:t}) \ge \max_{z \in V \setminus V_{\text{top-}K}} \hat P(z | \xv_{1:t}) \,.
\]
The extreme $K = |V|$ corresponds to ancestral sampling. 
The other extreme $K = 1$ is known as \textit{greedy decoding}, which corresponds to choosing the most likely next token iteratively.
Greedy decoding is often used to approximate the most likely sequence $\argmax_{\xv} P(\xv | \xv_{1:t})$. 

\textit{Nucleus sampling}~\cite{holtzman2019curious}, similar to top-$K$ sampling, returns a sparse distribution. Given a
parameter $p \in (0, 1)$, it applies the transformation
\begin{align} \label{eq:nucleus-def}
    Q_{\text{nuc},p}(x_{t+1} \mid \xv_{1:t}) 
    =
    \begin{cases}
    \frac{1}{Z} \, \hat P(x_{t+1} \mid \xv_{1:t}), &\text{if } x_{t+1} \in V_{\text{nuc}, p} , \\
    0, & \text{else},
    \end{cases}
\end{align}
where $Z$ is again a normalizing constant.
Here, the top-$p$ vocabulary $V_{\text{nuc}, p}$ is the smallest set $V' \subset V$ such that $\sum_{x \in V'} \hat P(x | \xv_{1:t}) \ge p$.

\subsection{Comparing Generative Models}
The usual approach to evaluating a text generation model is to compare the output of the model to human-written text for the same prompt~\cite[etc.]{papineni2002bleu,lin2004rouge}. This paradigm, however, breaks down for open-ended generation since there can be multiple correct outputs. 

We frame the problem as comparing two distributions. Let $Q \in \Pcal(\Xcal)$ denote the model distribution over some data space $\Xcal$ such as text sequences or images and let $P \in \Pcal(\Xcal)$ denote the target real data distribution. For text distributions, $Q$ depends on the underlying language model $\hat P$ as well as the decoding algorithm. 
The goal of open-ended text generation is to generate human-like text and the goal of image generation is to generate photorealistic images. Both these goals can be framed as finding a model distribution $Q$ that is as close to $P$ as possible in some metric. Therefore, we cast the evaluation of the generative model as measuring the gap between the model distribution $Q$ and the target distribution $P$. 
We will make this precise in \Cref{sec:mauve}.

\subsection{Information Divergences}
\label{ssec:info-div}
We review the definition of $\fdiv$-divergences and give a few examples.

\begin{definition}\label{def:fdiv}
    Let $\fdiv:(0, \infty) \to \mathbb{R}_+$ be a convex function
    with $\fdiv(1) = 0$. 
    Let $P, Q \in \Pcal(\Xcal)$ be dominated by some measure $\mu \in \Pcal(\Xcal)$ with densities $p$ and $q$ respectively.
    Then, the $\fdiv$-divergence between $P$ and $Q$ is defined as
    \[
        \Df{P}{Q} = \int_{\Xcal} q(x) f\left( \frac{p(x)}{q(x)} \right) \D \mu(x) \,,
    \]
    with the convention $f(0) = \lim_{t\to 0^+} f(t)$ and $0 f(p/0) = p \lim_{t \rightarrow 0^+} t\,f(1/t)$. 
\end{definition}
\noindent 
Note that the non-negativity condition on $f$ is without loss of generality.\footnote{
    The generator $\hat f(t) = f(t) + c(t-1)$ yields the same $f$-divergence as a convex function $f$ with $f(1) = 0$ for all $c \in \reals$. By choosing $c$ such that $f'(1) = 0$, we get that $\hat f$ is minimized at $t=1$. This ensures non-negativity: $\inf_{t > 0} \hat f(t) = \hat f(1) = 0$.
}
Since $f$ is convex and nonnegative with $f(1) = 0$,
we have that $f$ is non-increasing on $(0, 1]$ and non-decreasing on $[1, \infty)$.
The conjugate generator to $\fdiv$ is the function 
$\ftil: (0, \infty) \to [0, \infty)$ defined by\footnote{
The conjugacy between $\fdiv$ and $\ftil$, also known as \emph{Csiszár conjugacy}, is unrelated to the Fenchel or Lagrange duality in convex analysis. This notion of conjugacy is related to the perspective transform $g(t, s) = s \, f(t/s)$.
}
\[
    \ftil(t) = t f(1/t) \,,
\]
where again we define $\ftil(0) = \lim_{t\to 0^+} \ftil(t)$.
Since $\ftil$ can be constructed by the perspective transform of $f$, it is also convex.
We can verify that $\ftil(1) = 0$ and $\ftil(t) \ge 0$ for all $t \in (0, \infty)$, so it defines another divergence $D_{\ftil}$.
We call it the \emph{conjugate divergence} to $D_{\fdiv}$ since
\[
    \Dftil{P}{Q} = \Df{Q}{P} \,.
\]
The divergence $D_\fdiv$ is symmetric if and only if $\fdiv = \ftil$, and we write it as $D_\fdiv(P, Q)$ to emphasize the symmetry.

\begin{example}\label{ex:f_div}
    We give a few examples of $f$-divergences. 
    \begin{enumerate}[noitemsep,topsep=0pt, label=(\alph*)]
        \item KL divergence: It is an $f$-divergence generated by 
            $\fdiv_\kl(t) = t\log t - t + 1$.
        \item Interpolated KL divergence:
        For $\lambda \in (0, 1)$, the interpolated KL divergence is given by
        \[
            \klam{\lambda}(P \Vert Q) = \kl(P \Vert \lambda P + (1-\lambda) Q) \,.
        \]
        It is an $\fdiv$-divergence whose generator can be obtained from the upcoming \Cref{property:frontier-as-f-div}.
        \item \label{item:fdiv:def:JS} 
        Jensen-Shannon divergence: The Jensen-Shannon Divergence is defined as
        \[
            D_\js(P, Q) = \frac12 \klam{1/2}(P \Vert Q) + \frac12 \klam{1/2}(Q \Vert P).
        \]
        More generally, we have the $\lambda$-skew Jensen-Shannon Divergence~\cite{nielsen2013matrix}, which is defined for $\lambda \in (0, 1)$ as $D_{\js, \lambda} = \lambda \klam{\lambda}(P \Vert Q) + (1-\lambda) \klam{1-\lambda}(Q \Vert P)$. This is an $\fdiv$-divergence generated by
        \[
            \fdiv_{\js, \lambda}(t) = \lambda t \log\left(\frac{t}{\lambda t + 1-\lambda} \right) + (1-\lambda) \log\left( \frac{1}{\lambda t + 1-\lambda} \right) \,.
        \]
        \item Interpolated $\chi^2$ divergence: Similar to the interpolated KL divergence, we can define the interpolated $\chi^2$ divergence $D_{\chi^2, \lambda}$
        and the corresponding generator $\fdiv_{\chi^2, \lambda}$ 
        for $\lambda \in (0, 1)$ as
        \[
            D_{\chi^2, \lambda}(P \Vert Q) = D_{\chi^2}(P \Vert \lambda P + (1-\lambda) Q) \quad 
            \text{and} \quad
            \fdiv_{\chi^2, \lambda}(t) = \frac{(t-1)^2}{\lambda t + 1-\lambda} \,.
        \]
        The usual $\chi^2$ divergence is obtained in the limit $\lambda \to 0$.
    \end{enumerate}
\end{example} 
\section{Generalizing Divergence Frontiers with $f$-Divergences} 
\label{sec:mauve}

In this section, we start with the notion of KL divergence frontiers from \cite{djolonga2020precision} and define $f$-divergence frontiers in \Cref{sec:mauve:curve}. 
We define three scalar summaries of the frontier in \Cref{ssec:scalar} 
and study their properties in \Cref{sec:mauve:properties}.

\subsection{Tradeoff Curves to Evaluate Generative Models}
\label{sec:mauve:curve}

Consider a generative model $Q \in \Pcal(\Xcal)$ which attempts to model the target distribution $P \in \Pcal(\Xcal)$.
It has been argued in~\cite{sajjadi2018assessing,kynknniemi2019improved} that one must consider two types of costs to evaluate $Q$ with respect to $P$:
(a) a type I cost incurred from generating poor-quality data, which is the mass of $Q$ that has low or zero probability mass under $P$, and 
(b) a type II cost incurred from a failure to capture the diversity of the real data, which is the mass of $P$ that $Q$ does not adequately capture.

Suppose $P$ and $Q$ are uniform distributions on their supports, and $R$ is uniform on the union of their supports. Then, the type I cost is the mass of $\Supp{Q}\setminus \Supp{P}$, or equivalently, the mass of $\Supp{R}\setminus \Supp{P}$.
We measure this using the surrogate 
$\kl(Q\Vert R)$, which is large if 
there exists an atom $\xv$ such that $Q(\xv)$ is large but $R(\xv)$ is small.
Likewise, the type II cost is measured by $\kl(P \Vert R)$. When $P$ and $Q$ are not constrained to be uniform, it is not clear what the measure $R$ should be. \citet{djolonga2020precision} propose to 
vary $R$ over all possible probability measures and consider the Pareto frontier of the multi-objective optimization $\min_R \big( \kl(P\Vert R), \kl(Q \Vert R) \big)$.
This leads to a curve called the {\em divergence frontier}, illustrated in \Cref{fig:main:illustration}), and is reminiscent of the precision-recall curve in binary classification.  See~\cite{pepe2000receiver,cortes2005confidence,clemencon2009precision,clemenccon2010overlaying,flach2012machine} and references therein on trade-off curves in machine learning.

It was shown in \cite[Props. 1 and 2]{djolonga2020precision} that the divergence frontier $\Fcal(P, Q)$ of probability measures $P$ and $Q$ is carved out by mixtures $R_\lambda = \lambda P + (1-\lambda)Q$ for $\lambda \in (0, 1)$. 
We present an elementary proof for completeness. 

\begin{property} \label{prop:div-pareto-opt}
    Consider two distributions $P, Q$
    with finite support. Then, the Pareto frontier
    for the pair of objectives
    $\big(\kl(P \Vert \cdot), \kl(Q \Vert\cdot)\big)$ is given by
    \begin{align} \label{eq:div-pareto-opt}
        \Fcal(P, Q) = 
        \Big\{
        \big(\kl(P \Vert R_\lambda), \,  \kl(Q \Vert R_\lambda)\big) \,:\,
        \lambda \in (0, 1)
        \Big\} \,,
    \end{align}
    where $R_\lambda = \lambda P + (1-\lambda) Q$.
    In other words, there does not exist any distribution $R$ such that 
    $\kl(P|R) < \kl(P|R_\lambda)$
    and 
    $\kl(Q|R) < \kl(Q|R_\lambda)$ simultaneously for any $\lambda \in (0, 1)$.
\end{property}
\begin{proof}
The convexity of $\kl(P\Vert\cdot), \kl(Q \Vert \cdot)$
allows us to compute the Pareto frontier $\Fcal(P, Q)$ 
exactly by minimizing linear combinations of the objectives.
Concretely, we have from~\cite[Thms. 3.4.5 \& 3.5.4]{miettinen2012nonlinear} that
\begin{align*}
    \Fcal(P, Q) &= 
    \Big\{
        \big(\kl(P \Vert R_\lambda^\star),
        \kl(P \Vert R_\lambda^\star)\big) \,:\,
        \lambda \in [0, 1]
    \Big\},
\quad \text{where} \\
    R_\lambda^\star &\in 
    \argmin_R\{ \lambda\, \kl(P \Vert R) + 
    (1 - \lambda)\, \kl(Q \Vert R) \} \,.
\end{align*}
Simple algebra gives us the identity
    \begin{align*}
		\lambda \, \kl(P \Vert R) + (1 - \lambda) \, \kl(Q \Vert R) 
	    &= \lambda \, \kl(P \Vert R_\lambda) + (1 - \lambda) \, \kl(Q \Vert R_\lambda)
			+ \kl(R_\lambda \Vert R) \,.
	\end{align*}
	The first two terms of the right-hand side are independent of $R$ and the last term is minimized at $R = R_\lambda$. Therefore, $R_\lambda^\star = R_\lambda$.
\end{proof}

In this work, we consider a more general family of $f$-divergence frontiers. 
\begin{definition} \label{def:fdiv-frontier}
    The $f$-divergence frontier $\Fcal_f(P, Q)$ for two distributions $P, Q \in \Pcal(\Xcal)$ and a divergence generator function $f$ satisfying $f(0) < \infty$ and $f^*(0) = \infty$
    is defined as 
    \[
        \Fcal_f(P, Q) = 
        \Big\{
        \big(\Df{P}{R_\lambda}, \,  \Df{Q}{R_\lambda}\big) \,:\,
        \lambda \in (0, 1)
        \Big\} \,,
    \]
    where $R_\lambda = \lambda P + (1-\lambda) Q$. 
\end{definition}

The condition $f(0) < \infty$ ensures that $\Df{P}{R_\lambda}$
and $\Df{Q}{R_\lambda}$ are finite for $0 < \lambda < 1$, so 
the $f$-divergence frontier is well defined.
The condition $f^*(0) = \infty$ mimics the behavior of the KL divergence so that $D_f(P \Vert Q) = \infty$ when $P \not\ll Q$ and $D_f(Q \Vert P) = \infty$ when $Q \not\ll P$.
This allows the divergence curve to grow to infinity as $\lambda$ approaches the endpoints of $(0, 1)$ 
if the supports of $P$ and $Q$ are not identical.
When $f$ is not specified, we refer to the KL divergence frontier defined above---it corresponds to $f(t) = t\log t - t + 1$. 

Each coordinate of the $f$-divergence frontier is itself an $f$-divergence as we show next. 
\begin{property} \label{property:frontier-as-f-div}
    Consider the $f$-divergence $D_f$ generated by the convex function $f$. For any $\lambda \in (0, 1)$, we have that 
    $\Df{P}{\lambda P + (1-\lambda)Q} = D_{f_{\lambda}}(P \Vert Q)$ and $\Df{Q}{\lambda P + (1-\lambda)Q} = D_{f_{1-\lambda}}(Q \Vert P)$, where $f_\lambda:(0, \infty) \to \reals_+$ is given by
    \begin{align} \label{eq:fdiv-linear-comb}
        f_\lambda(t) = (\lambda t + 1 - \lambda) \,\,  f\left( \frac{t}{\lambda t + 1 - \lambda}  \right) \,.
    \end{align}
    Further, $D_{f_\lambda}$ is a valid $f$-divergence in that it satisfies the conditions of \Cref{def:fdiv}: $f_\lambda$ is convex, non-negative and $f_\lambda(1) = 0$. 
    Moreover, if $f$ is twice differentiable with $f''(t) > 0$ for all $t > 0$, then 
    $f_\lambda$ is strictly convex with $f_\lambda''(t) > 0$ for all $t > 0$.
\end{property}
\begin{proof}
    We have $f_\lambda \ge 0$ and $f_\lambda(1) = 0$ by definition. In order to establish the convexity of $f_\lambda$, observe that $f_\lambda(t) = (g \circ h_\lambda)(t)$, where $g(t, s) = s\, f(t/s)$ is the perspective transform of $f$, and $h_\lambda(t) = (t, \lambda t + 1-\lambda) \in \reals^2_+$ is a linear map.
    The perspective $g$ of a convex function $f$ is convex, and convexity is preserved upon composition with a linear map $h_\lambda$, so $f_\lambda$ is convex. 
    Finally, $\Df{P}{\lambda P + (1-\lambda)Q} = D_{f_{\lambda}}(P \Vert Q)$ and $\Df{Q}{\lambda P + (1-\lambda)Q} = D_{f_{1-\lambda}}(Q \Vert P)$ can be verified from the definition. 

    To show the strict convexity of $f_\lambda$, we calculate 
    \[
        f_\lambda''(t) = \frac{(1-\lambda)^2}{(\lambda t + 1- \lambda)^3} \, f'' \left( \frac{t}{\lambda t + 1 - \lambda} \right) >0
    \]
    under the given assumptions.
\end{proof}

\subsection{Scalar Summaries of Divergence Frontiers} 
\label{ssec:scalar}
We define three summaries of divergence frontiers.

\myparagraph{Area Summary}
The first summary is inspired by the area under the curve~\cite[e.g.][]{flach2012machine}---a common strategy to summarize tradeoff curves in machine learning. 
Divergence frontiers, however, can be unbounded. 
For instance, as $\lambda \to 1$, we have 
$\kl(Q \Vert R_\lambda) \to \kl(Q \Vert P)$, which can be unbounded. 
The same holds for $f$-divergence frontiers because $\ftil(0) = \infty$.
Therefore, we define \mauve to be the area under a monotonic transformation of 
the $f$-divergence frontier:
\begin{align} \label{eq:mauve:area-def} 
    \mauve_f(P, Q) = \mathrm{AUC}\left(
        \left\{
        \big(\exp(-c x), \, \exp(-c y)\big) \,:\, (x, y) \in \Fcal_f(P, Q)
        \right\} \cup \{(1, 0), (0, 1) \}
    \right) \,.
\end{align}
Here, $c > 0$ is a scaling constant that changes the numerical value of \mauve, but not 
its induced ordering over multiple models $Q_1, \ldots, Q_n$. %
$\mauve_f(P, Q)$ is always bounded between $0$ and $1$ with larger values denoting a greater similarity between $P$ and $Q$. 

\myparagraph{Integral Summary}
For the second summary of the divergence frontier, we take inspiration from 
the minimax theory of hypothesis testing, where the goal is also to study two types of errors and it is common to theoretically analyze their linear combination; see, e.g., \cite[Sec.~1.2]{ingster2003nonparametric} and \cite[Thm. 7]{cai2011optimal}.
Similarly, we consider a linear combination of the two costs that are the two coordinates of the divergence frontier: 
\begin{align}\label{eq:linear_cost}
    \lerror{f, \lambda}(P, Q) := \lambda \,\Df{P}{R_\lambda} + (1 - \lambda) \Df{Q}{R_\lambda} \,.
\end{align}
Note that, for the KL divergence,
$R_\lambda$ is exactly the minimizer of the linearized objective $\lambda \, \kl(P \Vert R) + (1 - \lambda) \kl(Q \Vert R)$ according to \Cref{prop:div-pareto-opt}. In this case,  $\lerror{\lambda}$ is also known as the $\lambda$-skew Jensen-Shannon Divergence (cf. \Cref{ex:f_div}).

The linearized cost $\lerror{f, \lambda}$ depends on the choice of $\lambda$.
To remove this dependency, we define an integral summary as
\begin{align}
    \mray_f(P, Q) := 2\int_0^1 \lerror{f, \lambda}(P, Q) \, \D\lambda \;.
\end{align}
We can interpret the frontier integral as the average linearized cost over $\lambda \in (0, 1)$.
The constant of 2 is arbitrary and is chosen so that $\mray_\kl$ is bounded above by 1, 
as we shall momentarily see in \Cref{sec:mauve:properties}.

\myparagraph{Mid-point Summary}
The third summary is a generalization of the Jensen-Shannon divergence, defined to be the linearized cost with weight $\lambda=1/2$, i.e.,
\begin{align}\label{eq:mid_point}
    \midp_f(P, Q) := L_{f, 1/2}(P, Q) = \frac12 \,\Df{P}{R_{1/2}} + \frac12 \Df{Q}{R_{1/2}} \,.
\end{align}
When $f$ is the generator of the KL (resp. $\chi^2$) divergence, it recovers the Jensen-Shannon (resp. Le Cam) divergence. This summary is intuitively close to the area summary as illustrated in \Cref{fig:mauve-v-midp}.

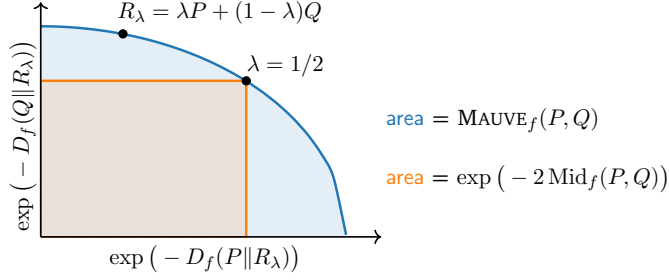
\begin{figure}[t]
\centering
\begin{adjustbox}{max width=0.6\linewidth}
\begin{tikzpicture}
\def\divcurveA{\x, {
    3.59 * sqrt(1 - ( (\x - 0.2) / 5.2)^2)
    }}

\draw[very thick, color=C0!5, domain=0.2:5.4, preaction={fill=C0!60, fill opacity=0.2}, fill opacity=0.1, pattern=dots, pattern color=C0, smooth] (0.2, 0) -- plot (\divcurveA) -- (5.4, 0) -- (0.2, 0) ;

\draw[very thick, color=C0, domain=0.2:5.4,smooth] plot (\divcurveA) ;

\draw[color=black, fill] (1.6, 3.46) circle (2pt) node[label={[text=black, xshift=-0.25cm, yshift=-0.1cm]80:$R_{\lambda} = \lambda P + (1-\lambda) Q$}] {};

\draw[thick, color=C1, domain=0.2:5.4, preaction={fill=C1!60, fill opacity=0.2}, fill opacity=0.1, pattern=fivepointed stars, pattern color=C1, smooth] (0.2, 0) -- (0.2, 2.66) -- (3.7, 2.66) -- (3.7, 0) -- (0.2, 0) ;

\draw[very thick, color=C1] (0.2, 2.66) -- (3.7, 2.66) -- (3.7, 0) ;

\draw[thick,->] (0.2,0) -- (6,0) node[below, xshift=-3cm] {$\exp\big(- D_f(P \Vert R_\lambda) \big)$};
\draw[thick,->] (0.2,0) -- (0.2,4) node[left, rotate=90, xshift=-0.5cm, yshift=0.3cm] {$\exp\big(- D_f(Q \Vert R_\lambda) \big)$} ;

\draw[color=black, fill] (3.7, 2.66) circle (2pt) node[label={[text=black, xshift=-0.15cm, yshift=-0.2cm]80:$\lambda=1/2$}] {};

\node at (7.9, 2)  
{
\fontfamily{cmss}\selectfont
\begin{tabular}{c}  
\textcolor{C0}{area} = $\mauve_f(P, Q)$
\end{tabular}
};

\node at (8.5, 1)  
{
\fontfamily{cmss}\selectfont
\begin{tabular}{c}  
\textcolor{C1}{area} = $\exp\big(- 2\, \midp_f(P, Q)\big)$
\end{tabular}
};

\end{tikzpicture} \end{adjustbox}
\caption{\small
Relationship between the area summary $\mauve_f$ and the mid-point summary $\midp_f$. $\mauve$ is the area under the blue curve, while 
    the mid-point summary $\midp$ is related to the area under the orange rectangle.}
\label{fig:mauve-v-midp}
\end{figure}

\subsection{Properties of Divergence Frontier Summaries}
\label{sec:mauve:properties}
We study some properties of the area summary \mauve. 

\begin{property} \label{prop:mauve-properties}
    Fix an $f$-divergence $\Df{\cdot}{\cdot}$ such that $f(0) < \infty$ 
    and a scaling constant $c > 0$. 
    For any two distributions $P, Q$ with finite support, 
    the area summary $\mauve(P, Q)$ satisfies the following:
    \begin{enumerate}[itemsep=0cm,leftmargin=\widthof{( a ) },topsep=0cm,label=(\alph*)]
        \item $0 \le \mauve_f(P, Q) = \mauve_f(Q, P) \le 1$, 
        \item $\mauve_f(P, P) = 1$, and
        \item if $f$ is strictly convex, $\mauve_f(P, Q) = 1$ if and only if $P=Q$.
    \end{enumerate}
\end{property}
\begin{proof}
    The curve $(\exp(-cx), \, \exp(-cy))$ for $(x, y) \in \Fcal_f$ always lies within the unit square, so $0 \le \mauve_f(P, Q) \le 1$. 
    If $P = Q$, then $\Df{P}{R_\lambda} = \Df{Q}{R_\lambda} = 0$ for all $\lambda \in (0, 1)$, so that $\mauve_f(P, Q)$ is simply the area of the unit square. 
    Conversely, if $P \ne Q$, we have that $\Df{P}{R_\lambda} \neq 0$ and $\Df{Q}{R_\lambda} \neq 0$ for any $\lambda \in (0, 1)$ whenever 
    $f$ is strictly convex. Therefore, 
    the curve $(\exp(-cx), \, \exp(-cy))$ for $(x, y) \in \Fcal_f$ lies strictly within the unit square and $\mauve_f(P, Q) < 1$. 
\end{proof}

We now turn to the integral summary.
\begin{property} \label{prop:fi-properties}
    The integral summary $\fint$ of the $f$-divergence frontier 
    defined by a convex generator $f$ satisfies the following properties: 
    \begin{enumerate}[itemsep=0cm,leftmargin=\widthof{( a ) },topsep=0cm,label=(\alph*)]
    \item $\mray_f$ is an $f$-divergence generated 
        by the convex function 
        \[
            \tilde f(t) = 2 \int_0^1 \Big( \lambda \, f_\lambda(t) + (1-\lambda) f_{1-\lambda}^*(t) \Big) \D \lambda\,,
        \]
        where $f_\lambda$ is as defined in \eqref{eq:fdiv-linear-comb}.
        \item $\mray_f(P, Q) = \mray_f(Q, P)$.
        \item $0 \le \mray_f(P, Q) \le 4 \int_0^1 \lambda f^*(\lambda) \D \lambda + \frac{2}{3} f(0)$.
        \item If $f$ is twice differentiable with $f''(t) > 0$ for all $t > 0$, we have
            $\mray_f(P, Q) = 0$ if and only if $P=Q$.
    \end{enumerate}
\end{property}
\begin{proof}
    We denote $\bar\lambda = 1-\lambda$. 
    For the first part, we have from \Cref{property:frontier-as-f-div}, 
    \begin{align*}
        \mray_f(P, Q) &= 
        2 \int_0^1 \left( \lambda D_{f_\lambda}(P \Vert Q) + \bar\lambda D_{(f_{\bar \lambda})^*} (P \Vert Q) \right) \D \lambda 
        = D_{\tilde f}(P \Vert Q) \,,
    \end{align*}
    by using the definition of $f$-divergences. 
    Note that $\tilde f$ is a convex function as it is the positive linear combination of a family of convex functions. We also directly verify that $\tilde f(t) \ge \tilde f(1) = 0$ for all $t > 0$, so $D_{\tilde f}$ is a well-defined $f$-divergence. 
    For the second part, we get
    \[
        (\tilde f)^*(t) = t \tilde f(1/t) = 
        2 \int_0^1 \Big( \lambda \, f_\lambda^*(t) + (1-\lambda) f_{1-\lambda}(t) \Big) \D \lambda\,
        = \tilde f(t) \,,
    \]
    where the last equality follows by substituting $\lambda' = 1-\lambda$.
    Therefore, $\mray_f(Q, P) = D_{\tilde f}(Q \Vert P) = D_{\tilde f^*}(P \Vert Q) = D_{\tilde f}(P \Vert Q) = \mray_f(P, Q)$.
    For the third part, we use the upper bound on $\lerror{f, \lambda}$ from \Cref{prop:lerror-max} in \Cref{sec:a:properties} to get 
    \[
        \mray_f(P, Q) = 2 \int_0^1  \lerror{f, \lambda}(P \Vert Q) \, \D \lambda 
        \le 2 \int_0^1 \left( \lambda f^*(\lambda) + \bar \lambda f^*(\bar \lambda) + 2\lambda \bar \lambda f(0) \right) \D \lambda \,.
    \]
    Simplifying this integral gives the third part. 
    For the final part, we note that $f_\lambda''(t) > 0$ and $(f_{\bar\lambda}^*)''(t) > 0$ for all $t > 0$ from \Cref{property:frontier-as-f-div}. This gives
    \[
    (\tilde f)''(t) 
    = 2 \int_0^1 \Big( \lambda \, f_\lambda''(t) + (1-\lambda) (f_{1-\lambda}^*)''(t) \Big) \D \lambda\, > 0\,.
    \]
    This implies that $\tilde f$ is strictly convex. Therefore, $D_{\tilde f}(P \Vert Q) = 0$ iff $P= Q$. 
\end{proof}

We can instantiate this property for common divergences. 
The integral summary $\mray_\kl$ of the KL divergence frontier is generated by
\[
    \tilde f_{\kl}(t) = \frac{t+1}{2} - \frac{t}{t-1} \log t \,,
\]
with the understanding that $\tilde f_{\kl}(1) = \lim_{t \to 1} \tilde f_{\kl}(t) = 0$.
Similarly, the corresponding expression for the integral summary of the $\chi^2$ divergence frontier is
\[
    \tilde f_{\chi^2}(t) = \frac{t^2+t+1}{t-1} \, \log t - \frac{3}{2}\left( t+1\right) \,.
\]
We have that $\mray_\kl$ and $\mray_{\chi^2}$ are upper bounded by $1$ and $2$ respectively.

Lastly, we turn to the mid-point summary.
\begin{property} \label{prop:mid-properties}
    The mid-point summary $\midp_f$ of the $f$-divergence frontier 
    defined by a generator $f$ satisfies the following properties: 
    \begin{enumerate}[itemsep=0cm,leftmargin=\widthof{( a ) },topsep=0cm,label=(\alph*)]
    \item $\midp_f$ is an $f$-divergence generated 
        by the convex function $f_{1/2}$ as defined in \eqref{eq:fdiv-linear-comb}.
        \item $\midp_f(P, Q) = \midp_f(Q, P)$.
        \item $0 \le \midp_f(P, Q) \le \frac{1}{2}\left( f(0) + f(2) \right)$.
        \item If $f$ is twice differentiable with $f''(t) > 0$ for all $t > 0$, we have
            $\midp_f(P, Q) = 0$ if and only if $P=Q$.
    \end{enumerate}
\end{property}
\begin{proof}
    The first, second, and fourth parts follow directly from \Cref{property:frontier-as-f-div}.
    The third part is a consequence of \Cref{prop:lerror-max} in \Cref{sec:a:properties}.
\end{proof} 
\section{Practical Computation of the Divergence Frontier and its Summaries} \label{sec:compute}

In this section, we consider how to compute \mauve and related divergence frontier summaries for high dimensional distributions of text or images.
We usually do not have access to the target distribution $P$ representing human-written text or real-world images.
While the model likelihood $Q(\xv)$ can be evaluated for some generative model $Q$ such as 
language models for text, it might not be available for others such as generative adversarial networks for images.
Therefore, we only assume access to 
the distributions $P$ and $Q$ via i.i.d.~samples. 

Given two independent samples 
$\xv_1, \ldots, \xv_n \stackrel{\text{i.i.d.}}{\sim} P$ and $\xv_1', \ldots, \xv_m' \stackrel{\text{i.i.d.}}{\sim} Q$, 
we wish to estimate the summaries $\mauve_f(P, Q)$, $\fint_f(P, Q)$, or $\midp_f(P, Q)$ using these samples.
We will often assume equal sample sizes $m=n$ for simplicity, especially when stating bounds.
In real image or text applications, the distributions $P$ and $Q$ are typically discrete 
distributions whose support size is too large to enumerate. 
For instance, neural language models induce a probability distribution over documents of text. Thus, we cannot tractably compute
the $f$-divergences required by the divergence frontiers or their scalar summaries 
in closed form. Instead, we consider four different estimation methods:
\begin{itemize}[itemsep=0cm,leftmargin=\widthof{( a ) },topsep=0cm]
\item \textbf{Vector Quantization}:
    We quantize the empirical distributions 
    $\hat P_n = (1/n)\sum_{i=1}^n \delta_{\xv_i}$
    and $\hat Q_m = (1/m) \sum_{j=1}^m \delta_{\xv_j'}$
    into $k$-dimensional multinomial distributions 
    $\hat P_{n, k}$ and $\hat Q_{m, k}$, where $k$ is a hyperparameter. 
    We then estimate the divergence frontier by the plug-in estimator
    $\Fcal_f(\hat P_{n, k}, \hat Q_{m, k})$, from which the corresponding summaries
    \mauve, $\fint$, and $\midp$ can be estimated.
    This approach can also be used with add-constant distribution estimators in place of empirical distributions; see \Cref{tab:add_const} for some examples.
\item \textbf{Nearest-neighbor estimation}: 
    We endow the space $\Xcal$ with a metric $\rho:\Xcal \times \Xcal \to \reals_+$ and consider the set $N_k(\xv)$ of the $k$-nearest neighbor of $\xv$ from the union of $X = \{\xv_i\}_{i=1}^n$ and $X' = \{\xv_j'\}_{j=1}^m$. 
    We estimate the likelihood ratio $P(\xv_j')/Q(\xv_j')$ based on the ratio $|N_k(\xv_j') \cap X| / |N_k(\xv_j') \cap X'|$
    for $j = 1,\ldots, m$.
    This likelihood ratio can then be used to estimate the required 
    $f$-divergences. 
    
\item \textbf{Classifier-based estimation}: 
    We train a classifier over samples $\{(\xv_1, +1)\}_{i=1}^{n'} \cup \{(\xv_j', -1)\}_{j=1}^{m'}$ and use this to estimate 
    the likelihood ratio $P(\xv)/Q(\xv)$ over the remaining $n-n' + m - m'$ samples. This likelihood ratio can then be used to estimate the required $f$-divergences. 

\item \textbf{Parametric approximation}: Given an embedding $\varphi: \Xcal \to \reals^d$, 
    we make a parametric assumption that the 
    pushforward distributions $\varphi_{\sharp}P = \Ncal(\mu_P, \Sigma_P)$ 
    and $\varphi_\sharp Q = \Ncal(\mu_Q, \Sigma_Q)$ 
    with unknown parameters $\mu_P, \Sigma_P, \mu_Q, \Sigma_Q$. We 
    estimate $\hat \mu_P, \hat \Sigma_P, \hat \mu_Q, \hat \Sigma_Q$ from data and use
    $\Fcal_f\big(\Ncal(\hat\mu_P, \hat\Sigma_P), \Ncal(\hat\mu_Q, \hat\Sigma_Q)\big)$ as an estimate that is computed
     by numerical integration.
     Although this approach is widely used in practice, it has no theoretical guarantees. Therefore, we defer its discussion to \Cref{sec:a:parametric} and compare its empirical performance with other methods in \Cref{sec:expt:estimation}.
\end{itemize}

In the rest of this section, we consider each in detail. 
In full generality, we will focus on estimating $f$-divergences from samples. The results on estimating the $f$-divergence frontier $\Fcal_f(P, Q)$ follow as corollaries because each point on the frontier is itself an $f$-divergence (\Cref{property:frontier-as-f-div}).

\subsection{Estimation via Vector Quantization}
\label{sec:mauve:quant}

\begin{figure*}[t]
    \centering
    \includegraphics[width=0.9\linewidth]{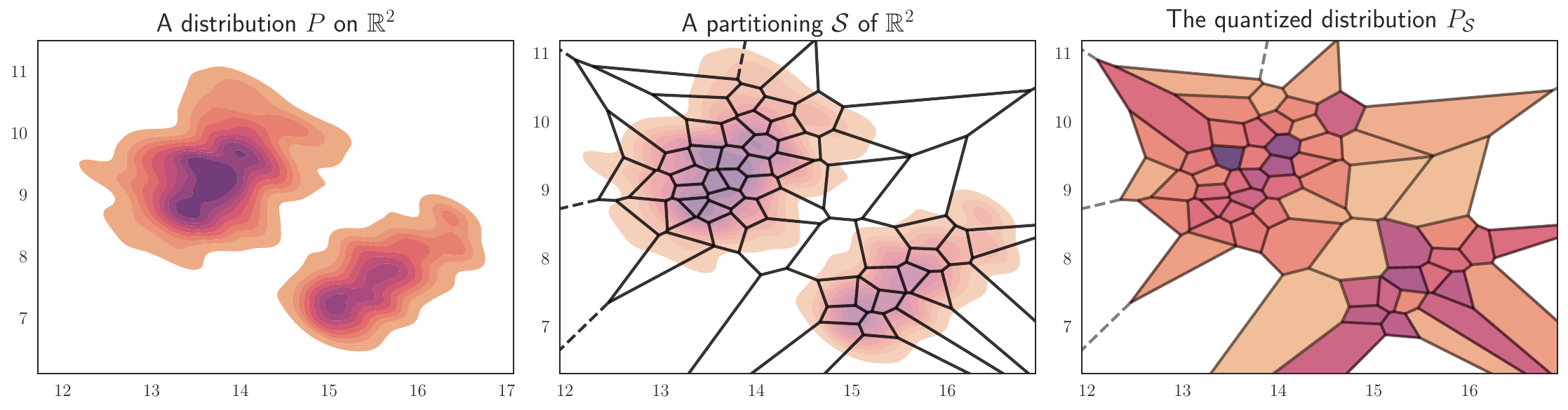}
    \caption{\small
    Illustration of the quantization $P_\Scal$ of a distribution $P$ over the Euclidean plane $\reals^2$ under a partition $\Scal$.
    }
    \label{fig:mauve:quant}
\end{figure*}

Given a $k$-partition $\Scal = \{S_1, \ldots, S_k\}$ of the space $\Xcal$, 
we define the quantization of $P$ over $\Scal$ as $P_\Scal = \big(P(S_1), \ldots, P(S_k)\big)$. Then, $P_\Scal$ and $Q_\Scal$ are multinomial distributions over $k$ atoms; they are piecewise constant approximations of $P$ and $Q$ similar to histograms as illustrated in \Cref{fig:mauve:quant}. 
The quantization approach to estimating the divergence frontier consists of two approximations: 
\begin{itemize}[itemsep=0cm,leftmargin=\widthof{( a ) },topsep=0cm]
    \item approximating the intractable divergence frontier $\Fcal_f(P, Q)$
        with the lower-dimensional counterpart $\Fcal_f(P_\Scal, Q_\Scal)$, and
    \item estimating this frontier $\Fcal_f(P_\Scal, Q_\Scal)$ with its plug-in estimator $\Fcal_f(\hat P_{\Scal, n}, \hat Q_{\Scal, m})$, 
        where $\hat P_{\Scal, n} = \big( n^{-1} \sum_{i=1}^n \indone\{ \xv_i \in S_l \}  \big)_{l=1}^k$ is the empirical distribution of $P_\Scal$, and $\hat Q_{\Scal, m}$ is the corresponding empirical distribution of $Q_\Scal$
\end{itemize}
In practice, the best quantization schemes are data-dependent, such as $k$-means clustering or lattice-type vector quantization of dense representations of images or text; we will discuss this in more detail in \Cref{sec:quant:practical}. 

When the two distributions $P$ and $Q$ have long tails, the empirical estimators $\hat P_{\Scal, n}$ and $\hat Q_{\Scal, m}$ can be of poor quality due to the \emph{missing mass} phenomenon \citep{good1953frequency}, i.e., some probability masses do not appear in the finite sample. This is illustrated in \Cref{fig:missing-mass}.
A widely used technique to address such a challenge is the \emph{add-constant} smoothing \citep[see, e.g.,][]{krichevsky1981performance}.
This approach adds a small constant $b$ to the counts of each bin and normalizes these pseudo-counts to form a normalized probability distribution.
Precisely, the add-$b$ estimator of $P_{\Scal}$ is defined as
\begin{align}\label{eq:add_constant}
    \hat P_{\Scal, n}^b = \left( \frac{b + \sum_{i=1}^n \indone\{ \xv_i \in S_l \}}{n + kb} \right)_{l=1}^k.
\end{align}
Other estimators suitable for this regime have also been considered in the literature such as the Good-Turing estimator \citep{orlitsky2015turing} and absolute discounting \citep{falahatgar2017power}.

\subsubsection{Estimation Error Bounds}
The total estimation error of the divergence frontier consists of two parts: (a) the statistical error in estimating $\Fcal_f(P_\Scal, Q_\Scal)$ from samples, and
(b) the quantization error in passing from $P, Q$ to $P_\Scal, Q_\Scal$.
For simplicity, we assume in this subsection that $m = n$.
In what follows, we establish a statistical error bound of order $O(\sqrt{k/n})$ and
show that there exists a quantization scheme with error $O(1/k)$. 
The theory suggests that we can balance the two errors at $k = \Theta\big(n^{1/3}\big)$.

\begin{figure}[t]
    \centering
    \adjustbox{max width=0.33\textwidth}{%
\begin{tikzpicture}[scale=1]
 
\begin{axis}[
    width=10cm,
    height=6.5cm,
    xmin=0.3,
    xmax=11,
    ymin=0.0,
    ybar=1pt,
    bar width=8pt,
    xlabel={\color{white}{position}},
    axis lines = left,
    axis line style={->, thick},
    legend style={font=\LARGE,draw=none},
    legend image post style={scale=2.5},
	tick label style={font=\Large},
	label style={font=\Large}
]
\addplot[C0!75, fill, postaction={
        pattern=dots, pattern color=C1
    }] coordinates {
    (1, 0.34) (2, 0.204) (3, 0.159) (4, 0.156) (5, 0.075) (6, 0.027) (7, 0.02) (8, 0.015) (9, 0.015) (10, 0.01)
};
\addlegendentry{$P$}

\addplot[C1!95, fill,  postaction={
        pattern=sixpointed stars, pattern color=C1
    }] coordinates {
    (1, 0.35) (2, 0.25) (3, 0.1) (4, 0.2) (5, 0.1) (6, 0.0) (7, 0.0) (8, 0.0) (9, 0.0) (10, 0.0) 
};
\addlegendentry{$\hat P_n$}

\addplot[line width=0pt, samples=50, smooth, C2!95,fill=oldmauve, fill opacity=0.1] coordinates {(5.45, 0.05) (10.6, 0.05)} \closedcycle;
\addplot[ line width=4pt, samples=50, smooth, C2!95] coordinates {(5.47, -0.05) (5.47, 0.05)};
\addplot[ line width=4pt, samples=50, smooth, C2!95] coordinates {(10.6, -0.05) (10.6, 0.05)};
\addplot[ line width=4pt, samples=50, smooth, C2!95] coordinates {(5.47, 0.05) (10.6, 0.05)};

\tikzstyle{textbf} = [draw,rectangle,text width=5cm,text centered]
\node[color=black, font=\LARGE] at (axis cs: 8.0, 0.07) {Missing mass};

\end{axis}

\end{tikzpicture} %
}
    \adjustbox{max width=0.3\textwidth}{%
\begin{tikzpicture}[scale=1]
    \draw[draw=white] (0, 0) rectangle ++(0.3,1);
    \draw [-stealth,line width=1.1pt](0,1.6) -- (2.4,1.6);
    \node at (1.1,1.8) {\scriptsize Krichevsky-Trofimov};
\end{tikzpicture} %
}
    \adjustbox{max width=0.33\textwidth}{%
\begin{tikzpicture}[scale=1]
 
\begin{axis}[
    width=10cm,
    height=6.5cm,
    xmin=0.3,
    xmax=11,
    ymin=0.0,
    ybar=1pt,
    bar width=8pt,
    xlabel={\color{white}{position}},
    axis lines = left,
    axis line style={->, thick},
    legend style={font=\LARGE,draw=none},
    legend image post style={scale=2.5},
	tick label style={font=\Large},
	label style={font=\Large}
]
\addplot[C0!75, fill, postaction={
        pattern=dots, pattern color=C1
    }] coordinates {
    (1, 0.34) (2, 0.204) (3, 0.159) (4, 0.156) (5, 0.075) (6, 0.027) (7, 0.02) (8, 0.015) (9, 0.015) (10, 0.01)
};
\addlegendentry{$P$}

\addplot[C1!95, fill,  postaction={
        pattern=sixpointed stars, pattern color=C1
    }] coordinates {
    (1, 0.3) (2, 0.22) (3, 0.1) (4, 0.18) (5, 0.1) (6, 0.02) (7, 0.02) (8, 0.02) (9, 0.02) (10, 0.02) 
};
\addlegendentry{$\hat P_n^b$}

\end{axis}

\end{tikzpicture} %
}
    \caption{\small
    \textbf{Left}:
    Missing mass of a sample corresponds to those entries $l \in \Supp{P}$ that do not appear in the sample, i.e., $\hat P_{n, l} = 0$.
    \textbf{Right}: Add-constant smoothing adds a constant $b$ to counts of each bin $l \in \Supp{P}$, including those that do not appear in the sample. Krichevsky–Trofimov smoothing corresponds to $b=1/2$.
    }
    \label{fig:missing-mass}
\end{figure}
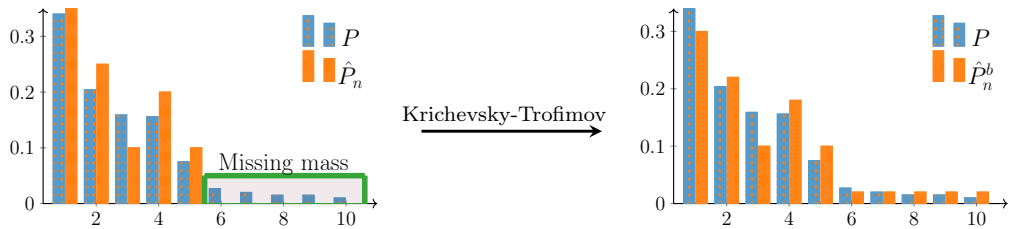

\myparagraph{Statistical Estimation Error}
We establish a statistical bound on estimating a general $f$-divergence $\Df{P}{Q}$ between discrete distributions $P, Q$ using their plug-in estimators $\hat P_n, \hat Q_n$ from samples, respectively. 
To this end, we require the generator $f$ and its conjugate $\ftil$ to satisfy some smoothness and tail assumptions. 
\begin{assumption}\label{asmp:fdiv}
The generator $\fdiv$ is twice continuously differentiable with $f'(1) = 0$. Furthermore,
\begin{enumerate}[noitemsep,topsep=0pt,label={\textbf{(A\arabic*})},leftmargin=\widthof{\textbf{ (A8) }}]
    \item \label{asmp:fdiv:bounded} 
        We have $\ConstZ := \fdiv(0) < \infty$
        and $\ConstZTil := \ftil(0) < \infty$.
    \item \label{asmp:fdiv:1st-deriv} 
        There exist constants $\ConstI, \ConstITil < \infty$  
        such that for every $t \in (0, 1)$, we have, 
        \begin{align*}
            |\fdiv'(t)| &\le \ConstI \left(1 \vee \log ({1}/{t})  \right), \quad \text{and}, \quad
            |\ftilg(t)| \le \ConstITil \left(1 \vee \log ({1}/{t})  \right) \,.
        \end{align*}
    \item \label{asmp:fdiv:2nd-deriv} 
        There exist constants $\ConstII, \ConstIITil < \infty$ such that 
        for every $t \in (0, \infty)$, we have, 
        \begin{align*}
            \frac{t}{2} \fdiv''(t) &\le \ConstII \,, \quad \text{and}, \quad
            \frac{t}{2} \ftilh(t) \le \ConstIITil \,.
        \end{align*}
\end{enumerate}
\end{assumption}
Some boundedness assumption is necessary since the minimax quadratic risk of estimating the KL divergence over all discrete distributions with $k$ atoms is always infinity~\cite{bu2018kl}. 
Assumption~\ref{asmp:fdiv:bounded} is a necessary and sufficient condition for $\Df{P}{Q}$ and $\Dftil{P}{Q}$ to remain bounded for all distributions $P, Q$. 
Assumption~\ref{asmp:fdiv:1st-deriv} guarantees that $\fdiv$ is approximately Lipschitz and cannot vary too fast, while
\ref{asmp:fdiv:2nd-deriv} is a technical assumption that helps control the variation of $\fdiv$ around zero.

These assumptions hold for many $f$-divergences, as shown in \Cref{tab:fdiv:asmp-examples}. Notably, they hold for the $\fint_\kl$ and $\midp_\kl$, as well as the coordinates of the KL and $\chi^2$ divergence frontiers.

\begin{table}[t]
    \centering
    \begin{adjustbox}{max width=0.97\linewidth}
    \begin{tabular}{l c cccccc}
    \toprule
    $\fdiv$-divergence & \begin{tabular}{c} Satisfies \\ Assumptions? \end{tabular} & 
    $\ConstZ$ & $\ConstZTil$ & $\ConstI$ & $\ConstITil$ & 
    $\ConstII$ & $\ConstIITil$ 
    \\
    \toprule
    KL & No & $1$ & $\infty$ & & & &  
    \\[0.2cm]
    Interpolated KL & Yes & $\bar \lambda$ & $ \log\tfrac{1}{\lambda} - \bar\lambda$ & $1$ & $\tfrac{\bar \lambda^2}{\lambda}$ & $\tfrac{1}{2}$ & $\tfrac{\bar \lambda}{8\lambda}$ 
    \\[0.2cm]
    Jensen-Shannon (JS) / $\midp_\kl$ & Yes & $\tfrac{1}{2}\log 2$ & $\tfrac{1}{2}\log 2$ & $\tfrac{1}{2}$ & $\tfrac{1}{2}$ & $\tfrac{1}{4}$ & $\tfrac{1}{4}$ 
    \\[0.2cm]
    Skew JS & Yes & $\bar\lambda \log\tfrac{1}{\bar \lambda}$ & $\lambda \log\tfrac{1}{\lambda}$ & $\lambda$ & $\bar \lambda$ & $\tfrac{\lambda}{2}$ & $\tfrac{\bar\lambda}{2}$
    \\[0.2cm]
    $\fint_\kl$ & Yes & $\tfrac{1}{2}$ & $\tfrac{1}{2}$ & $4$ & $4$ & $\tfrac{1}{2}$ & $\tfrac{1}{2}$ 
    \\[0.2cm]
    Interpolated $\chi^2$ & Yes & $\tfrac{1}{\bar\lambda}$ & $\tfrac{1}{\lambda}$ & $\tfrac{2}{\bar\lambda^2}$ & $\tfrac{2}{\lambda^2}$ & $\frac{4}{27 \lambda \bar\lambda^2}$ & $\tfrac{4}{27 \lambda^2\bar\lambda}$ 
    \\[0.2cm]
    Le Cam / $\midp_{\chi^2}$ & Yes & $\tfrac{1}{2}$ & $\tfrac{1}{2}$ & $2$ & $2$ & $\tfrac{8}{27}$ & $\tfrac{8}{27}$  
    \\[0.2cm]
    Squared Hellinger & No & $1$ & $1$ & $\infty$ & $\infty$ & & \\
    \bottomrule
    \end{tabular}
    \end{adjustbox}
    \caption{\small
    Examples of $\fdiv$-divergences and whether they satisfy Assumptions~\ref{asmp:fdiv:bounded}-\ref{asmp:fdiv:2nd-deriv}. Here, $\lambda \in (0, 1)$ is a parameter 
    of the interpolated or skew divergences, and we define 
    $\bar \lambda := 1- \lambda$.
    \label{tab:fdiv:asmp-examples}
    }
\end{table}

We now turn to the statistical error bound. 
When both $P$ and $Q$ are supported on a finite alphabet with $k$ items, a natural strategy is to exploit the smoothness properties of the $f$-divergence, namely Assumption~\ref{asmp:fdiv:1st-deriv}. This gives a na\"ive upper bound $O(L \sqrt{k/n})$ on the absolute error, where $L = C_1 \log{(1/p_{*})}$ with $p_{*} = \min_{l \in \Supp{P}} P_l$ reflects the smoothness of the $f$-divergence.
The dependency on $p_{*}$ requires $P$ to have finite support and a short tail.
However, in many real-world applications, the distributions can either be supported on a countable set or have long tails~\cite{chen1999empirical,wang2017learning}.
By considering the \emph{missing mass} in the sample, 
that is the total probability mass that does not appear in the finite sample~\cite{good1953frequency}, 
we can obtain a bound that is independent of $p_{*}$.
We refer to \Cref{fig:missing-mass} (left) for an illustration of the missing mass.

\begin{theorem}\label{thm:fdiv:consistency}
    Assume that $k := |\Supp{P}| \vee |\Supp{Q}| \in \mathbb{N} \cup \{\infty\}$. 
    Let $n \ge 3$, $c_1 :=  \ConstI + \ConstITil$,
    and $c_2 := \ConstII \vee \ConstZTil +\ConstIITil \vee \ConstZ$.
    Under \Cref{asmp:fdiv},
    we have,
    \begin{align}\label{eq:fdiv:stat_error_oracle}
        \expect| \Df{P}{Q} - \Df{\Phatn}{\hat Q_n}|
        &\le \big(\ConstI \log{n} + \ConstZTil \vee \ConstII\big) \alpha_{n}(P) + \big(\ConstITil \log{n} + \ConstZ \vee \ConstIITil\big) \alpha_{n}(Q) \\
        &\quad + \big(\ConstI + \ConstZTil \vee \ConstII\big) \beta_{n}(P) + \big(\ConstITil + \ConstZ \vee \ConstIITil\big) \beta_{n}(Q)\,, \nonumber
    \end{align}
    where $\alpha_n(P) = \sum_{l=1}^k \sqrt{n^{-1} P_l}$ and $\beta_n(P) = \expect\big[ \sum_{l: \hat P_n(l) = 0} P_l \max\left\{ 1, \log{(1/P_l)} \right\} \big]$.
    Furthermore, if $k < \infty$, then
    \begin{align}\label{eq:fdiv:stat_error}
        \expect| \Df{P}{Q} - \Df{\Phatn}{\hat Q_n}| \le \big(c_1 \log{n} + c_2\big) \left(\sqrt\frac{k}{n} + \frac{k}{n}  \right) \,.
    \end{align}
    In particular, for the \mauveray, it gives a statistical error bound of 
    \begin{align}\label{eq:tail_bound_ray}
            \expect\abs{\fint(\hat P_n, \hat Q_n) - \fint(P, Q)} \le C \left( \sqrt{\frac{k}{n}} + \frac{k}{n} \right) \log{n}\,,
    \end{align}
    where $C$ is some absolute constant.
\end{theorem}

Some remarks about the bounds in \Cref{thm:fdiv:consistency} are as follows. 
First, the bound \eqref{eq:fdiv:stat_error_oracle} holds for any distributions with a countable support.
Second, it does not depend on $p_{*}$ and is adapted to the tail behavior of $P$ and $Q$.
For instance, if $P$ is defined as $P_l \propto l^{-2}$ for $l \in [k]$, then $\alpha_n(P) \propto (\log{k})/\sqrt{n}$, which is much smaller than $\sqrt{k/n}$ in \eqref{eq:fdiv:stat_error} in terms of the dependency on $k$. This result 
justifies the practice of using a large quantization size $k$ on real data. 
Third, it captures a parametric rate of convergence, i.e., $O(n^{-1/2})$, up to a logarithmic factor.
This rate is not improvable in a related problem of estimating $\kl(P \Vert Q)$, even with the assumption that $P/Q$ is bounded~\cite{bu2018kl}.
The bound in \eqref{eq:fdiv:stat_error} is a distribution-free bound, assuming $k$ is finite.
Note that it also gives an upper bound on the sample complexity by setting the right-hand side of \eqref{eq:fdiv:stat_error} to be $\epsilon$ and solving for $n$; this is roughly $k/\epsilon^2$, ignoring constants and log factors.

\begin{proof}[Proof Sketch of \Cref{thm:fdiv:consistency}]
    We sketch the proof for the $\fint_\kl(P, Q) = D_{\tilde f}(P \Vert Q)$ with full details given in \Cref{apppend:sub:stat_error}. The proof relies on a careful analysis of the derivatives of the $f$-divergence while accounting for the missing mass. We start by defining the bivariate scalar function $\psi(p, q) = q \, \tilde f(p/q)$ where $\tilde f$ is the generator of $\fint$.
    Then, we have $\fint(P, Q) = \sum_{l=1}^k \psi(P_l, Q_l)$.  
    By the triangle inequality, we have, 
    \[
        \left| 
            \fint(\hat P_n, \hat Q_n) - \fint(P, Q)
        \right|
        \le 
        \sum_{l=1}^k
        \underbrace{\left| 
            \psi(\hat P_{n, l}, \hat Q_{n, l}) - \psi(P_l, \hat Q_{n,l})
        \right|}_{=: \Delta_l} +
        \underbrace{\left| 
            \psi(P_l, \hat Q_{n,l}) - \psi(P_l, Q_l)
        \right|}_{=: \Delta_l'} \,.
    \]
    We bound $\Delta_l$ in terms of $|\hat P_{n,l} - P_l|$ so that summing over all coordinates gives a bound on the total variation distance $\|\hat P_n - P\|_{\tv} = \sum_{l=1}^k |\hat P_{n,l} - P_l|$.
    A first-order Taylor expansion gives the bound 
    \[
        \Delta_l \le \sup_{s \in [0, 1]} 
        |\psi_p(s P_l + (1-s)\hat P_{n,l}, Q_l)|
        \,\, |P_l - \hat P_{n,l}| \,,
    \]
    where $\psi_p$ denotes the partial derivative of $\psi$ w.r.t. its first argument. 
    Unfortunately, as $p\to 0$ for fixed $q\neq 0$, we have that $|\psi_p(p, q)| = \abs{\tilde f'(p/q)} \le \log(q/p) \to \infty$ by Assumption~\ref{asmp:fdiv:1st-deriv}.
    
    We use a two-pronged approach to overcome this issue. First, we take a second-order Taylor expansion and carefully bound the remainder term using Assumption~\ref{asmp:fdiv:2nd-deriv} to get
    \begin{align}\label{eq:delta_l}
        \Delta_l 
            \le \frac{1}{2}  \, 
            |\hat P_{n,l} - P_l|  \,
            \log \left(\frac{1}{\max\{P_l, \hat P_{n,l}\}}\right)
            \,.
    \end{align}
    Secondly, because $\hat P_n$ is an empirical distribution, we only have two possibilities: $\hat P_{n,l} \ge 1/n$ or $\hat P_{n,l} = 0$. The first case gives an additional $\log n$ dependence on the total variation distance (based on Assumption~\ref{asmp:fdiv:1st-deriv}), while the second case is in the missing mass regime. 
    Based on results from the missing mass literature~\cite{berend2012missing,mcallester2005concentration}, we show
    \[
        \beta_n(P) = 
        \expect\left[\sum_{l=1}^k \ind(\hat P_{n,l} = 0) \, \, P_l \log \frac{1}{P_l}
        \right] \le \frac{k \log n}{n} \,,
    \]
    where $\beta_n(P)$ is constructed from the upper bound \eqref{eq:delta_l} with $\hat P_{n,l} = 0$.
    Finally, we bound the total variation term by repeatedly applying Jensen's inequality as
    \begin{align*}
        \expect \norm{\hat P_n - P}_\tv
        \le \sum_{l=1}^k \sqrt{\expect(\hat P_{n,l}- P_l)^2} 
        = \sum_{l=1}^k \sqrt{\frac{P_l(1 - P_l)}{n}} \le \alpha_n(P) \le \sqrt{\frac{k}{n}} \,.
    \end{align*}
\end{proof}

Following \Cref{property:frontier-as-f-div}, we can specialize \Cref{thm:fdiv:consistency} to show the consistent estimation of the entire $f$-divergence frontier $\Fcal(P, Q)$.
\begin{proposition}\label{prop:consis_df}
    Take an arbitrary $\lambda_0 \in (0, 1)$.
    Suppose we are given distributions $P, Q$ with $k := |\Supp{P}| \vee |\Supp{Q}| \in \mathbb{N} \cup \{\infty\}$.
    Assume that \Cref{asmp:fdiv} holds true for $f_\lambda$ with $\lambda \in [\lambda_0, 1-\lambda_0]$.
    If the sample size $n \ge 3$, the bounds in \eqref{eq:fdiv:stat_error_oracle} and \eqref{eq:fdiv:stat_error} hold for
    \begin{align} \label{eq:quant:worst-case}
        \expect\left[\sup_{\lambda \in [\lambda_0, 1 - \lambda_0]} 
        \Big\{
        \big\vert \Df{\hat P_n}{\hat R_\lambda} - \Df{P}{R_\lambda}
        \big\vert
        + 
        \big\vert \Df{\hat Q_n}{\hat R_\lambda} - \Df{Q}{R_\lambda} \big\vert
        \Big\}
        \right] \,,
    \end{align}
    where $\hat R_\lambda := \lambda \hat P_n + (1 - \lambda) \hat Q_n$,
    with constants replaced by $C / \lambda_0$ for some absolute constant $C$.
    In particular, if $\lambda_0=\lambda_n$ is chosen as $\lambda_n = o(1)$ and $\lambda_n = \omega(\sqrt{k/n} \log{n})$, then the expected worst-case error \eqref{eq:quant:worst-case} converges to zero at rate $O(\lambda_n^{-1} \sqrt{k/n} \log{n})$.
\end{proposition}
    When $f$ is the generator to the KL divergence, \Cref{asmp:fdiv} holds for $f_\lambda$.
    Hence, \Cref{prop:consis_df} holds for the KL divergence frontier.
In the absence of additional assumptions, the truncation in \Cref{prop:consis_df} is necessary to ensure boundedness of the estimated quantities, since $\kl(P \Vert R_\lambda)$ is close to $\kl(P \Vert Q)$ for small $\lambda$, and this can be unbounded.  

\begin{table}[t]
\centering
\small
\begin{tabular}{lll}
\hline
Braess-Sauer & Krichevsky-Troﬁmov & Laplace                                      \\ \hline
\begin{tabular}[c]{@{}l@{}} $b_l = 1/2$ if $l$ does not appear \\ $b_l = 1$ if $l$ appears once \\ $b_l = 3/4$ if $l$ appears more than once \end{tabular}       &  $b \equiv 1/2$ & $b \equiv 1$                   \\ \hline
\end{tabular}
\caption{Add-constant smoothed distribution estimators.}
\label{tab:add_const}
\end{table}

\myparagraph{Estimation Error With Smoothing}
We bound the statistical error in estimating the divergence $D_f(P \Vert Q)$ between $P$ and $Q$ using their add-constant estimators $\hat P_{n}^b$ and $\hat Q_{n}^b$ introduced in \eqref{eq:add_constant} and illustrated in \Cref{fig:missing-mass}.
Again, this result also holds for the $\fint_\kl$ and $\midp_\kl$, as well as the coordinates of the KL and $\chi^2$ divergence frontiers.
This result is proved in \Cref{apppend:sub:stat_error_smoothing}.

\begin{theorem}\label{thm:addb:fdiv:consistency}
    Assume that $k := |\Supp{P}| \vee |\Supp{Q}| \in \mathbb{N} \cup \{\infty\}$. 
    Let $n \ge 3$, $c_1 :=  \ConstI + \ConstITil$,
    and $c_2 := \ConstII \vee \ConstZTil +\ConstIITil \vee \ConstZ$.
    Under \Cref{asmp:fdiv},
    we have,
    \begin{align}\label{eq:addb:fdiv:stat_error_oracle}
        \expect| \Df{P}{Q} - \Df{\hat P_{n}^b}{\hat Q_{n}^b}|
        &\le \left( \frac{n \alpha_n(P)}{n + kb} + \gamma_{n,k}(P) \right) \left( C_1 \log{\left(\frac{n}{b} + k \right)} + C_0^* \vee C_2 \right) \\
        &\quad + \left( \frac{n \alpha_n(Q)}{n + kb} + \gamma_{n,k}(P) \right) \left( C_1^* \log{\left(\frac{n}{b} + k \right)} + C_0 \vee C_2^* \right). \nonumber
    \end{align}
    where $\gamma_{n,k}(P) = (n + bk)^{-1} bk \sum_{l=1}^k \abs{P_l - k^{-1}}$.
    Furthermore, if $k < \infty$, then
    \begin{align}\label{eq:addb:fdiv:stat_error}
        \expect| \Df{P}{Q} - \Df{\hat P_{n}^b}{\hat Q_{n}^b}| \le \left(c_1 \log{\left( \frac{n}{b} + k \right)} + c_2 \right) \frac{\sqrt{kn} + 2b(k-1)}{n + kb} \,.
    \end{align}
    In particular, for the \mauveray, it gives a statistical error bound of 
    \begin{align}\label{eq:addb:tail_bound_ray}
        \expect\abs{\fint(\hat P_n^b, \hat Q_n^b) - \fint(P, Q)} \le C \frac{\sqrt{kn} + 2b(k-1)}{n + kb} \log{\left( \frac{n}{b} + k \right)}\,,
    \end{align}
    where $C$ is some absolute constant.
\end{theorem}

Let us compare the bounds in \Cref{thm:addb:fdiv:consistency} with the ones in \Cref{thm:fdiv:consistency}.
For the distribution-dependent bound, the term $\alpha_n(P) \log{n}$ in \eqref{eq:fdiv:stat_error_oracle} is improved by a factor $n/(n+bk)$ in \eqref{eq:addb:fdiv:stat_error_oracle}.
The missing mass term $\beta_n(P)$ is replaced by the total variation distance between $P$ and the uniform distribution on $[k]$ with a factor $bk/(n + bk)$.
The improvements in both two terms are most significant when $k/n$ is large.
As for the distribution-free bound,
when $k/n$ is small, the bound in \eqref{eq:addb:fdiv:stat_error} scales the same as the one in \eqref{eq:fdiv:stat_error};
when $k/n$ is large (i.e., bounded away from $0$ or diverging), it scales as $O(\log{n} + \log{(k/n)} + k^{-1})$ while the one in \eqref{eq:fdiv:stat_error} scales as $O(k\log{n}/n + k^{-1})$.

\begin{figure}[t]
    \centering
    \includegraphics[width=\textwidth]{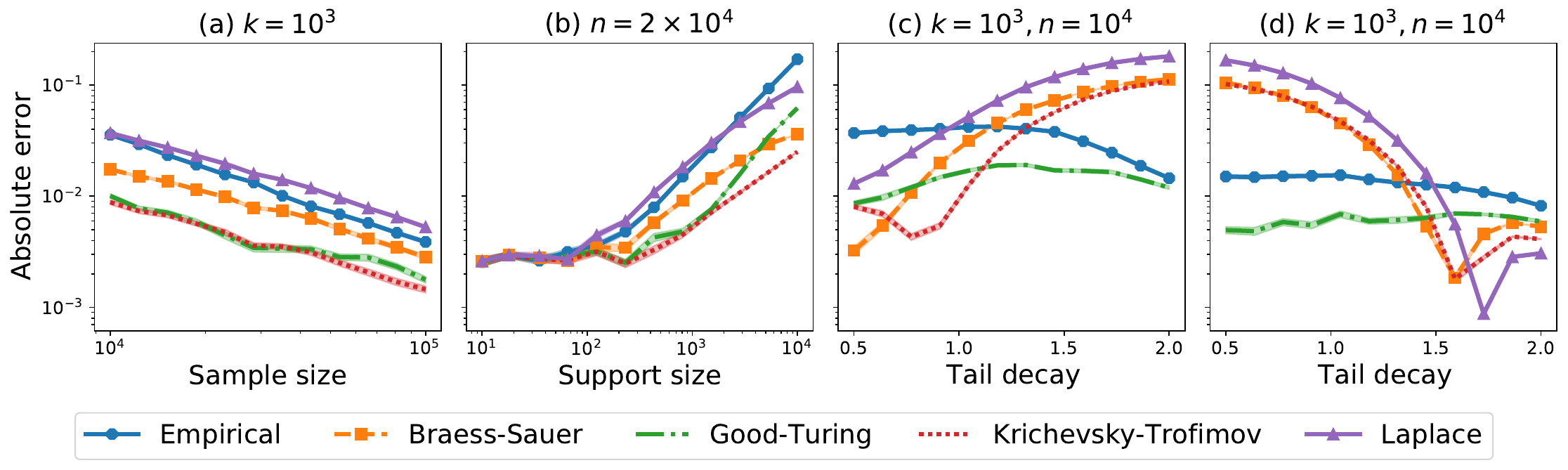}
    \caption{Statistical error with smoothed distribution estimators on synthetic data. \textbf{(a)}: $\zipf(0)$ and $\dir(\mathbf{1}/2)$ with $k = 10^3$; \textbf{(b)}: $\zipf(0)$ and $\dir(\mathbf{1}/2)$ with $n = 2\times 10^4$; \textbf{(c)}: $\dir(\mathbf{1})$ and $\zipf(r)$ with $k = 10^3$ and $n = 10^4$; \textbf{(d)}: $\zipf(2)$ and $\zipf(r)$ with $k = 10^3$ and $n = 10^4$.}
    \label{fig:main:smoothing:new}
\end{figure}

\myparagraph{Simulations of Smoothing}
We conduct a simple simulation study to empirically verify the effectiveness of smoothing.
Following the experimental settings used by~\citet{orlitsky2015turing}, we consider two types of distributions: (a) the Zipf$(r)$ distribution with $r \in [0, 2]$ where $P_l \propto l^{-r}$. 
(b) the Dirichlet distribution $\dir(\alpha)$ with $\alpha \in \{\mathbf{1}/2, \mathbf{1}\}$.
For each pair $(P, Q)$, we generate i.i.d.~samples of size $n$ from each of them and estimate the \mauveray from these samples.
We compare 4 different smoothed distribution estimators with the empirical distribution (``Empirical'') as discussed in~\cite{orlitsky2015turing}.
For each $l \in \Xcal$, let $n_l$ be the number of times $l$ appears in the sample and let $\varphi_t$ be the number of symbols appearing $t$ times in the sample.
The \emph{(modified) Good-Turing} estimator is defined as $\hat P_{n, l}^\mathrm{GT} \propto n_l$ if $n_l > \varphi_{n_l+1}$ and $\hat P_{n,l}^\mathrm{GT} \propto \big(\varphi_{n_l+1} + 1\big) (n_l+1) / \varphi_{n_l}$ otherwise.
The remaining three estimators are all based on the add-$b$ smoothing.
For the \emph{Braess-Sauer} estimator, the pseudo-count parameter $b = b_l$ is data-dependent and chosen as $b_l = 1/2$ if $n_l = 0$, $b_l = 1$ if $n_l = 1$ and $b_l = 3/4$ otherwise.
For the \emph{Krichevsky-Trofimov} estimator, the parameter $b \equiv 1/2$.
For the \emph{Laplace} estimator, the parameter $b \equiv 1$.
See \Cref{tab:add_const} for a summary.

As shown in \Cref{fig:main:smoothing:new}, the smoothed distribution estimators reduce the absolute error. For parts (a) and (b), the Good-Turing and the Krichevsky-Trofimov estimators have the best absolute error. For parts (c) and (d), the Good-Turing estimator is adapted to various regimes of tail-decay, outperforming the empirical estimator. The Krichevsky-Trofimov and Braess-Sauer estimators, on the other hand, exhibit small absolute errors for particular decay regimes.
While the smoothed estimators offer a marked improvement when $k/n$ is large (that is, close to 1), the best estimator is problem-dependent.
As a rule of thumb, we suggest the Krichevsky-Trofimov estimator which works well in the large $k/n$ regime but is still competitive when $k/n$ is small (i.e., large $n$).

\myparagraph{Quantization Error}
We now turn to the quantization error of $\fdiv$-divergences, i.e.,
\begin{align*}
    \inf_{\abs{\Scal} \le k} \abs{\Df{P}{Q} - \Df{P_{\Scal}}{ Q_{\Scal}}},
\end{align*}
where the infimum is over all partitions $\Scal$ of $\Xcal$ of size no larger than $k$, and $P_{\Scal}$ and $Q_{\Scal}$ are the quantized versions of $P$ and $Q$ according to $\Scal$.
We do not assume $\Xcal$ to be discrete, nor do we need \Cref{asmp:fdiv} to hold.
All the results hold for the \mauveray (\Cref{prop:fi-properties}) and pointwisely on the divergence frontier (\Cref{property:frontier-as-f-div}). 
Our analysis is inspired by the asymptotic approximation of an $\fdiv$-divergence with increasingly finer partitions~\cite[Theorem 6]{gyorfi1978fdiv}.
The key idea behind the proof is shown in \Cref{fig:quant:proof} and the full proof is given in \Cref{append:sub:quant_error}. 

\begin{proposition}\label{thm:quant_error_fdiv}
    For any two distributions $P, Q$ over $\Xcal$ and any $k \ge 1$, we have 
    \begin{align} \label{eq:fint:quant}
        \inf_{|\Scal|\le 2k}
        \left| 
            \Df{P}{Q} - \Df{P_\Scal}{Q_\Scal}
        \right|
        \le \frac{\fdiv(0) + \ftil(0)}{k} \,,
    \end{align}
    where the infimum is over all partitions of $\Scal$ of size at most $2k$. 
\end{proposition}

\begin{figure}[t]
    \centering
    \adjustbox{max width=0.65\textwidth}{%
\begin{tikzpicture}[scale=3.5]
\draw[->] (0,0) -- (1.2,0) node[below, xshift=0cm, yshift=-0.13cm] {$x$};
\draw[->] (0,0) -- (0,0.8) node[left, xshift=0cm, yshift=0] {$f\big(\tfrac{p(x)}{q(x)}\big)$};    

\def\fn{\x, {
    0.65 * exp( -((\x-0.66)^2)/ (2 * 0.08 * 0.08) ) + 
    0.35 * exp( -((\x-0.25)^2)/ (2 * 0.05 * 0.05) )
    }}
\fill [C0!20] (0.001,0.001) rectangle (0.1849,0.66);
\fill [C1!30] (0.1849,0.001) rectangle (0.3151,0.66);
\fill [C0!20] (0.3151,0.001) rectangle (0.5230,0.66);
\fill [C1!30] (0.5230,0.001) rectangle (0.5812,0.66);
\fill [C2!40] (0.5812,0.001) rectangle (0.7388,0.66);
\fill [C1!30] (0.7388,0.001) rectangle (0.7970,0.66);
\fill [C0!20] (0.7970,0.001) rectangle (1.1,0.66);
\draw[thick, color=black, domain=0:1.1,smooth] plot (\fn) ;
\draw[thick, densely dotted, color=black!70, smooth] (0.001, 0.66) -- (1.1, 0.66) ;
\draw[thick, densely dotted, color=black!70, smooth] (0.001, 0.4) -- (1.1, 0.4) ;
\draw[thick, densely dotted, color=black!70, smooth] (0.001, 0.15) -- (1.1, 0.15) ;

\node at (-0.1, 0.01) {\small $T_0$};
\node at (-0.1, 0.15) {\small $T_1$};
\node at (-0.1, 0.4) {\small $T_2$};
\node at (-0.1, 0.65) {\small $T_3$};

\end{tikzpicture}
}
    \caption{\small
Oracle quantization $\Scal$ in the estimation of the $f$-divergence $\Df{P}{Q}$ with $\Df{P_\Scal}{Q_\Scal}$, where $P$ and $Q$ have densities $p$ and $q$. This example shows quantization into $|\Scal| = 3$ bins: blue, orange, and green. Bin $i$ is given by the set $\{x \, :\, \fdiv(p(x)/q(x)) \in [T_{i-1}, T_{i})\}$.
    }
    \label{fig:quant:proof}
\end{figure}
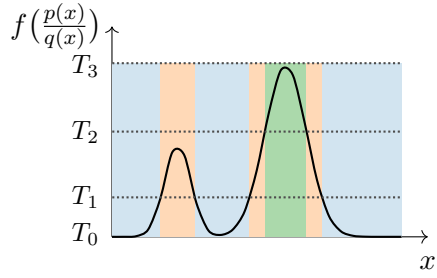

\myparagraph{Total Error}
Combining the bounds on the statistical and quantization errors leads to the following bound for the total estimation error for the \mauveray.
\begin{theorem}\label{thm:est_error_ray}
    Assume that $\Scal_k$ is a partition of $\Xcal$ such that $\abs{\Scal_k} = k \ge 2$.
    Then, the total error $\expect\abs{\mray(\hat P_{\Scal_k, n}, \hat Q_{\Scal_k, n}) - \mray(P, Q)}$ is upper bounded by
    \begin{align}\label{eq:est_error_ray}
        C \big[ \left( \alpha_n(P) + \alpha_n(Q) \right) \log{n} + \beta_n(P) + \beta_n(Q) + \abs{\mray(P, Q) - \mray(P_{\Scal_k}, Q_{\Scal_k})} \big].
    \end{align}
    Moreover, if the quantization error of $\Scal_k$ satisfies the bound in \eqref{eq:fint:quant}, we have
    \begin{align}\label{eq:est_error_ray_bound}
        \expect\abs{\fint(\hat P_{\Scal_k, n}, \hat Q_{\Scal_k, n}) - \mray(P, Q)} \le C \left[ \left( \sqrt{\frac{k}{n}} + \frac{k}{n} \right) \log{n} + \frac{1}{k} \right].
    \end{align}
\end{theorem}

Based on the bound in \eqref{eq:est_error_ray_bound}, a good choice of $k$ is $\Theta(n^{1/3})$ which balances between the statistical error and the quantization error.
This balancing is enabled by the existence of a good vector quantizer with a distribution-free bound in \eqref{eq:fint:quant}.
In practice, this suggests a data-dependent vector quantizer using nonparametric density estimators.
However, directions such as kernel density estimation \cite{ritter2002quantizing,hegde2004vector,van1999faithful} and nearest-neighbor methods \cite{alamgir2014density} have not met empirical success \emph{for vector quantization}, as they suffer from the curse of dimensionality common in nonparametric estimation.
In particular, 
\citet{wang2005divergence,silva2007universal,silva2010information} propose quantized divergence estimators but only prove asymptotic consistency and little progress has been made since then. 
On the other hand, modern data-dependent vector quantization techniques based on deep neural networks can successfully estimate properties of the density from high dimensional data \cite{sablayrolles2018spreading,hamalainen2020deep}.
Theoretical results for those techniques could complement our analysis.
We leverage these powerful methods to scale our approach on real data in \Cref{sec:experiments}.
In addition, while nonparametric estimators are not very successful for vector quantization, we can utilize them to estimate the $f$-divergences directly; we return to this in \Cref{sec:mauve:knn}.

\begin{algorithm}[t]
\caption{\mauve estimation via vector quantization}
\label{alg:mauve:quant}
\begin{algorithmic}[1]
\setstretch{1.5}
\setlength{\abovedisplayskip}{2.5pt}
\setlength{\belowdisplayskip}{2.5pt}
\Require{
Samples $\{\xv_i\}_{i=1}^n \stackrel{\text{i.i.d.}}{\sim} P$ and $\{\xv_j'\}_{j=1}^{m} \stackrel{\text{i.i.d.}}{\sim }Q$, 
quantization size $k$, smoothing constant $b$, embedding model $\varphi$, discretization $\Lambda$ of $[0, 1]$.}
\State $\{\varphi(\xv_i)\}_{i=1}^n$, $\{\varphi(\xv_j')\}_{j=1}^{m}\gets \texttt{embed}\left(\varphi, \{\xv_i\}_{i=1}^n, \{\xv_j'\}_{j=1}^{m} \right)$
\Comment{Embed the samples}
\State $C = \texttt{quantize}\left(\{\varphi(\xv_i)\}_{i=1}^n, \{\varphi(\xv_j')\}_{j=1}^{m}\right) $ \Comment{Cluster embeddings jointly}
\State For $l = 1, \ldots, k$, set \Comment{Count cluster assignments}
\[
\hat P_{\Scal, n, l}^b =  \frac{1}{n + kb}\left(\sum_{i=1}^n \indone\big\{ C(\xv_i) = l \big\}  + b\right), \,
\hat Q_{\Scal, m, l}^b = \frac{1}{m + kb} \left(\sum_{j=1}^m \indone\big\{ C(\xv_j') = l \big\} + b\right)
\]
\State Compute $\hat \Fcal_f(\hat P_{\Scal, n}^b, \hat Q_{\Scal, m}^b)$
from \eqref{eq:divf:discretized}
for $\lambda \in \Lambda$
\Comment{Build the divergence frontier}
\State \Return{$\mauve_f(P, Q) \approx \texttt{AUC}\left(\exp\left(-c \, \hat \Fcal_f(\hat P_{\Scal, n}^b, \hat Q_{\Scal, m}^b)\right) \right)$}
\Comment{Numerical quadrature}  
\end{algorithmic}
\end{algorithm}

\subsubsection{Towards a Practical Algorithm}
\label{sec:quant:practical}

To develop a practical vector quantization-based estimation procedure 
for the divergence frontier $\Fcal_f(P, Q)$, we use a data-dependent partitioning $\Scal$ based on quantizing the samples in some embedding space. The overall procedure is summarized in \Cref{alg:mauve:quant}.

Recall that we use vector quantization because the support size of real-world text or image distributions is extremely large. 
We employ embeddings from a pre-trained deep neural network to compute the vector quantization; 
such deep representations have been shown to capture the important properties of the data across modalities~\cite{zhang2018unreasonable,devlin2018bert}.

Concretely, we embed the samples using a model $\varphi: \Xcal \to \reals^d$ to get $\{\varphi(\xv_i)\}_{i=1}^n$ and $\{\varphi(\xv_j')\}_{j=1}^m$. 
Then, we jointly quantize the embedded samples
to obtain a mapping $C: \Xcal \to [k]$. This induces a partitioning 
$\Scal = (S_1, \ldots, S_k)$ with $S_l = \{\xv \in \Xcal \,:\, C(\xv) = l\}$. 
For instance, with $k$-means clustering~\cite{manning2001foundations,jurafsky2009speech},  $C(\xv)$ denotes the index $l$ of a cluster center $\cv_l$ that is closest to embedding $\varphi(\xv)$ in terms of $L_2$ distance so that
each partition $S_l \in \Scal$ is the Voronoi cell
\[
    S_l = \big\{
        \xv \in \Xcal \, :\, \norm{\varphi(\xv) - \cv_l}_2 \le \norm{\varphi(\xv) - \cv_j}_2 \text{ for } j = 1, \ldots, k
    \big\} \,.
\]
Here, we assume that ties are broken arbitrarily. 

The quantized distribution $P_{\Scal}$ is now computed from the fraction of the points in each partition.
For the add-$b$ smoothing, the estimator is
\[
    \hat P_{\Scal, n, l}^b = \frac{1}{n + kb} \left(\sum_{i=1}^n \indone\{\xv_i \in S_l\}  + b \right)\,, \quad \mbox{for } l = 1, \dots, k\,.
\]
Note that $b=0$ reduces to the empirical distribution, and this coincides with the approach used in \cite{pillutla2021mauve}.
In this work, we default to Krichevsky-Trofimov smoothing, which corresponds to $b=1/2$.

Each coordinate of the estimated divergence curve is now 
an $f$-divergence of the form $\Df{P_{\Scal, n}}{\lambda \hat P_{\Scal, n}^b + (1-\lambda) \hat Q_{\Scal, m}^b}$ and can be computed by summing over the $k$ coordinates. 
The full divergence frontier $\Fcal_f(\hat P_{\Scal, n}^b, \hat Q_{\Scal, m}^b)$ is a continuously parameterized curve for $\lambda \in (0, 1)$. For computational tractability, we take a discretization $\Lambda$ of $(0, 1)$ and take 
    \begin{align} \label{eq:divf:discretized}
    \hat \Fcal_f(P, Q) =
    \Bigg\{ \big( \Df{Q}{R_\lambda}, \Df{P}{R_\lambda} \big)
    \, :\, 
    \begin{matrix}
    R_\lambda = \lambda P + (1-\lambda)Q, \\
    \lambda \in \Lambda
    \end{matrix}
    \Bigg\} \,.
    \end{align}
We take a uniform grid $\Lambda = \{1 / N, 2/N, \ldots, (N-1)/N\}$ with $N$ points.
Finally, we approximate $\mauve_f(P, Q) \approx \mauve_f(\hat P_{\Scal, n}^b, \hat Q_{\Scal, m}^b)$ using numerical quadrature on the discretized frontier $\hat \Fcal_f(\hat P_{\Scal, n}^b, \hat Q_{\Scal, m}^b)$. 
For $\fint_f$, we can directly estimate $\fint_f(P, Q) \approx \fint_f(\hat P_{\Scal, n}^b, \hat Q_{\Scal, m}^b)$ when a closed-form expression is derived from \Cref{prop:fi-properties} (e.g., for KL and $\chi^2$ divergences). 

\myparagraph{Computational Complexity}
The computational complexity of the overall procedure in \Cref{alg:mauve:quant} is dominated by the cost of quantization. The complexity of $k$-means quantization is $O(Tknd)$, where $T$ is the maximum number of Lloyd's iterations and $d$ is the embedding dimension.

\subsection{Estimation via Nearest Neighbors}
\label{sec:mauve:knn}

We now turn to the estimation of the divergence frontier and its summaries by counting the nearest neighbors of each sample. 
We consider nearest neighbors from the $\ell_2$-distance in an embedding space. Given an embedding model $\varphi: \Xcal \to \reals^d$, we define a metric $\rho$ on the data space $\Xcal$ as
\[
    \rho(\xv, \xv') = \norm*{\varphi(\xv) - \varphi(\xv')}_2 \,.
\]

Let $N_k(\xv)$ denote the set of $k$-nearest neighbors (under the metric $\rho$) of $\xv$ from the set $X \cup X'$ where $X = \{\xv_i\}_{i=1}^n$ are samples from $P$ and $X' = \{\xv'_j\}_{j=1}^m$ are samples from $Q$. Following~\cite{noshad2017direct}, we estimate the $f$-divergence $\Df{P}{Q}$ with the estimator
\begin{align} \label{eq:knn-estimator}
    \hat D_{f, k}(X, X') = 
    0 \vee \frac{1}{m} \sum_{j=1}^m f\left( \frac{|N_k(\xv_j') \cap X| / n}{|N_k(\xv_j') \cap X'| / m}  \right) \,.
\end{align}
The intuition behind the estimator is that we expect $|N_k(\xv_j') \cap X| \propto P(\xv_j')$ and  $|N_k(\xv_j') \cap X'| \propto Q(\xv_j')$, so their ratio (with appropriate normalization) 
\begin{align} \label{eq:knn:likelihood-ratio}
    \hat r(\xv_j') = \frac{|N_k(\xv_j') \cap X| / n}{|N_k(\xv_j') \cap X'| / m}
\end{align}
can be considered an estimate of the likelihood ratio $r(\xv_j') := P(\xv_j') / Q(\xv_j')$. The $f$-divergence $\Df{P}{Q}$ is them estimated as 
\begin{align} \label{eq:knn:fdiv-from-LR}
    \hat D_{f, k}(X, X') = 0 \vee \frac{1}{m} \sum_{j=1}^m f\big( \hat r(\xv_j')\big) \,.
\end{align}

\subsubsection{Estimation Error Bounds}

Nearest neighbor estimation of $f$-divergences typically requires continuous distributions on a Euclidean space with densities satisfying certain regularity conditions. To this end, we consider estimation on a noisy version of the problem.

First, we pass from a discrete data space $\Xcal$ to an Euclidean embedding space by taking embeddings from a model $\varphi: \Xcal \to \reals^d$. While the pushforward distributions $\varphi_\sharp P$ and $\varphi_\sharp Q$ are now supported on $\reals^d$, 
they are not guaranteed to have a density w.r.t. the Lebesgue measure. 
To overcome this, we consider smooth these pushforward distributions by convolving them with a Gaussian $\Ncal(0, \sigma^2 I_d)$ to get distributions $P' = \varphi_\sharp P \star \Ncal(0, \sigma^2 I_d)$ and $Q' = \varphi_\sharp Q \star \Ncal(0, \sigma^2 I_d)$. 
Sampling from the convolved distribution is trivial: $\uv_i = \varphi(\xv_i) + \xiv_i$ and $\uv_j' = \varphi(\xv_j') + \xiv_j'$  are a valid samples from $P'$ and $Q'$ respectively for $\xv_i  \sim P$ and $\xv_j' \sim Q$ with independent Gaussian noise $\xiv_i, \xiv_j' \sim \Ncal(0, \sigma^2 I_d)$. 
We analyze the corresponding version of \eqref{eq:knn-estimator} that is constructed using the $\ell_2$ distance between the noisy vectors $\uv_i, \uv_j'$. We show that this procedure always underestimates the $f$-divergence. 

\begin{property}
    For any divergence generator $f$, we have 
    \[
        \Df{P'}{Q'} \le \Df{\varphi_\sharp P}{\varphi_\sharp Q} \le \Df{P}{Q} \,.
    \]
    Further, if the data space $\Xcal$ is discrete and 
    the embedding model is injective, i.e., $\varphi(\xv) \neq \varphi(\xv')$ for all distinct $\xv, \xv' \in \Xcal$, then the last inequality hold with equality.
\end{property}
\begin{proof}
    The inequalities are direct applications of the data processing inequality for $f$-divergences. 
    When $\varphi$ is injective, we have, 
    $(\varphi_\sharp P)\big(\varphi(\xv)\big) = P(\xv)$ for all $\xv \in \Xcal$ and similarly for $Q$. Therefore, 
    $\Df{\varphi_\sharp P}{\varphi_\sharp Q} = \Df{P}{Q}$
    follows from an equality on each term of the summation defining the $f$-divergence. 
\end{proof}

The nearest neighbor estimation \eqref{eq:knn-estimator} of $\Df{P'}{Q'}$ requires the following assumptions. 

\begin{assumption} \label{asmp:knn-mauve}
The smoothed distributions $P', Q'$ have densities $p', q'$ w.r.t. the Lebesgue measure, which satisfy the following: 
\begin{enumerate}[noitemsep,topsep=0pt,label={\textbf{(B\arabic*})},leftmargin=\widthof{\textbf{ (B8) }}]
    \item \label{asmp:knn:bounded} 
        There exists a $B > 0$ such that
        we have $1/B \le p'(\uv) / q'(\uv) \le B$ for all $\uv \in \reals^d$.
    \item \label{asmp:knn:holder} 
        The densities $p', q'$ are Hölder continuous with coefficient $\gamma \in (0, 1]$. That is, there exists a 
        constant $H > 0$ such that 
        \[
            |p'(\uv) - p'(\uv')| \le H \norm{\uv - \uv'}_2^\gamma \quad \text{for all } \uv, \uv' \in \reals^d \,,
        \]
        and similarly for $q'$. 
\end{enumerate}
\end{assumption}

\noindent The estimator \eqref{eq:knn-estimator} satisfies the following guarantee. 
\begin{theorem}[\citet{noshad2017direct}]
\label{thm:mauve:knn}
    Suppose the smoothed distributions $P', Q'$ satisfy \Cref{asmp:knn-mauve}, and the divergence generator 
    $f$ is $L$-Lipschitz over $[1/B, B]$, where $B$ is from Assumption \ref{asmp:knn:bounded}. Then, the $k$-nearest neighbor estimator $\eqref{eq:knn-estimator}$ with sample size $m=n$ satisfies
    \[
        \left| 
        \expect[\hat D_{f, k}(X, X')] - \Df{P'}{Q'} 
        \right|
        \le O\left( \left(\frac{k}{n}\right)^{\gamma/ d} + \frac{1}{k} \right) \,.
    \]
\end{theorem}
The assumption of $f$ being Lipschitz on a restricted domain $[1/B, B]$ follows directly from Assumption \ref{asmp:fdiv:1st-deriv} with a $\log B$ factor. Thus, this assumption holds for many $f$-divergences as shown in \Cref{tab:fdiv:asmp-examples}. 
The bound shows that this estimator suffers from the curse of dimensionality, as is common for nonparametric estimators.
The two terms of the error are balanced at $k = n^{\gamma/(d+\gamma)}$ and the optimal rate is $n^{-2\gamma/(d+\gamma)}$.

\begin{algorithm}[t]
\caption{\mauve estimation via nearest-neighbors}
\label{alg:mauve:knn}
\begin{algorithmic}[1]
\setstretch{1.5}
\setlength{\abovedisplayskip}{2.5pt}
\setlength{\belowdisplayskip}{2.5pt}
\Require{
Samples $X = \{\xv_i\}_{i=1}^n \stackrel{\text{i.i.d.}}{\sim} P$ and $X' = \{\xv_j'\}_{j=1}^{m} \stackrel{\text{i.i.d.}}{\sim }Q$, 
number of nearest neighbors $k$, lower dimension $d'$, embedding model $\varphi$, discretization $\Lambda$ of $[0, 1]$.}
\State $\{\varphi(\xv_i)\}_{i=1}^n$, $\{\varphi(\xv_j')\}_{j=1}^{m}\gets \texttt{embed}\left(\varphi, \{\xv_i\}_{i=1}^n, \{\xv_j'\}_{j=1}^{m} \right)$
\Comment{Embed the samples}
\State $U \cup U' = \texttt{PCA}\left(\{\varphi(\xv_i)\}_{i=1}^n \cup \{\varphi(\xv_j')\}_{j=1}^{m}, d'\right)$
\Comment{Joint dimensionality reduction}
\State Find $N_k(\uv) = \texttt{k-NN}(k, \uv, U \cup U')$ for $\uv \in U \cup U'$ \Comment{Find $k$-nearest neighbors jointly}
\State Estimate $\hat r(\uv)$ for $\uv \in U \cup U'$ as
\Comment{Estimate the likelihood ratio}
\[
    \hat r(\uv) =  \frac{|N_k(\uv) \cap U| / n}{|N_k(\uv) \cap U'| / m}
\]
\State Compute $\hat \Fcal_{f, k}(P, Q)$
from \eqref{eq:divf:discretized:knn}
for $\lambda \in \Lambda$
\Comment{Build the divergence frontier}
\State \Return{$\mauve_{f, k}(P, Q) = \texttt{AUC}\left(\exp\left(-c \, \hat \Fcal_{f, k}(P, Q)\right) \right)$}
\Comment{Numerical quadrature}  
\end{algorithmic}
\end{algorithm}

\subsubsection{Towards a Practical Algorithm}

We note from \Cref{thm:mauve:knn} that the nearest neighbor estimator \eqref{eq:knn-estimator} suffers from the curse of dimensionality. The embeddings obtained from pre-trained deep nets are extremely high-dimensional, ranging between $10^3$ and $10^4$ for typical text and image models. 
We find empirically that a dimensionality reduction step to $d' < 50$ dimensions with principal component analysis (PCA) is crucial for the estimator to work.
The overall algorithm is given in \Cref{alg:mauve:knn}. 

As in the case of estimation via quantization, we only consider the points on the divergence frontier at a discretization $\Lambda$ of $(0, 1)$. We then approximate each coordinate $x(\lambda)$ and $y(\lambda)$ of the divergence frontier for $\lambda \in \Lambda$ by using the nearest neighbor estimator \eqref{eq:knn-estimator}. Concretely, this gives us
\begin{align} \label{eq:divf:discretized:knn}
    \hat \Fcal_{f, k}(P, Q) =
    \Bigg\{ \big( \hat D_{f_\lambda, k}(X, X'), \hat D_{f_{1-\lambda}, k}(X', X) \big)
    \, :\, 
    \lambda \in \Lambda
    \Bigg\} \,,
\end{align}
where $f_\lambda$ is as defined in \Cref{property:frontier-as-f-div} so that $D_{f_\lambda}(P \Vert Q) = \Df{P}{\lambda P + (1-\lambda)Q}$. 
Finally, we estimate $\mauve_f(P, Q)$, $\fint_f(P, Q)$, and $\midp_f(P, Q)$ from this curve with numerical quadrature or with closed-form expressions when available.

\myparagraph{Computational Complexity}
The PCA step of \Cref{alg:mauve:knn} has time complexity $O(d n^2 + d' d^2)$ while the nearest neighbor search with K-d tree or ball tree structures takes time $O((d' + k) n \log{n})$, assuming $n = m$. While both steps can be sped up with approximate randomized algorithms, efficient open-source implementations of exact algorithms are fast enough for problems with a few thousand samples. We use the library Faiss~\cite{johnson2019billion} in our experiments in \S\ref{sec:experiments}.

\subsubsection{Extensions and Variants} 
\label{sec:kde-mauve}
We could also similarly define a kernel density estimator instead of the nearest neighbor estimator~\cite[e.g.][]{devroye1996probabilistic}. Given a kernel $\kappa: \reals^d \to \reals_+$ normalized such that $\int_{\reals^d} \kappa(\zv) \D \zv = 1$, 
the kernel density estimate of the density of a distribution $R$
using i.i.d. samples $U = \{\uv_1, \ldots, \uv_n\}$ is given by 
\begin{align} \label{eq:kde}
    g_{\kappa, h, U}(\uv) = \frac{1}{|U| h^d} \, \kappa\left( \frac{\uv - \uv_i}{h}\right) \,,
\end{align}
where $h$ is a bandwidth parameter. A typical choice of kernel is the Gaussian kernel
$\kappa(\zv) = (2 \pi)^{-d/2} \exp(-\norm{\zv}_2^2/2)$.

Similar to the nearest neighbor approach, we define the kernel density estimator in the embedding space of a model $\varphi: \Xcal \to \reals^d$.
We approximate $\Df{P}{Q}$ that of the kernel density estimator using samples $X \sim P^n$ and $X' \sim Q^m$ as $\Df{g_{\varphi(X)}}{g_{\varphi(X')}}$, which is in turn
estimated using its plug-in estimate
\begin{align}
    \hat D_{f, \kappa, h}(X, X') = \frac{1}{m} \sum_{j=1}^m f\left( \frac{g_{\kappa, h, \varphi(X)}\big(\varphi(\xv_j')\big)}{g_{\kappa, h, \varphi(X' \setminus \{\xv_j'\} )}\big(\varphi(\xv_j')\big)} \right) \,.
\end{align}
The expectation over $Q$ is approximated by a sample average over $X'$. The numerator of the term inside $f(\cdot)$ is simply the kernel density estimate \eqref{eq:kde} of $P$ at $\xv_j'$ using all $n$ samples from $X$, while the denominator is the corresponding estimate for $Q$ using the other $m-1$ samples $X' \setminus \{\xv_j'\}$. 
The rest of the estimation procedure is identical to \Cref{alg:mauve:knn}.

\subsection{Estimation via Classification}
\label{sec:mauve:classifier}

Here, we consider estimating the likelihood ratio $r(\xv):= P(\xv) / Q(\xv)$ with a probabilistic classifier such as logistic regression~\cite{sugiyama2012density}. The $f$-divergences can then be estimated from this likelihood ratio. 

We first set up a binary classification problem 
to discriminate between the two distributions $P$ and $Q$. 
Concretely, define the class prior as $\prob(y=+1) = n / (n+m)$ and $\prob(y=-1) = m / (n+m)$ and the class-conditional distribution by
$\prob(\xv | y=+1) = P(\xv)$ and $\prob(\xv | y=-1) = Q(\xv)$.
By the Bayes rule, the likelihood ratio can equivalently be written as 
\[
    r(\xv) := \frac{P(\xv)}{Q(\xv)} 
    = \frac{\prob(y = +1 | \xv)}{\prob(y=-1 | \xv)} \, \frac{\prob(y = -1)}{\prob(y=+1)} \,.
\]

Given a probabilistic classifier that outputs an estimate $\hat \eta(\xv)$ for $\prob(y=1 | \xv)$, we can estimate the likelihood ratio as \begin{align}
    \hat r(\xv) = \frac{m \, \hat \eta(\xv)}{n (1- \hat \eta(\xv))} = \frac{m}{n} \, \rho(\xv) \,,
\end{align}
where $\rho(\xv) := \hat \eta(\xv) / (1 - \hat \eta(\xv))$ is the odds ratio. 
We then estimate the $f$-divergence $\Df{P}{Q}$ using the Monte Carlo estimate
\begin{align}
    \hat D_{f}(X, X'; \hat p) = \frac{1}{m} \sum_{j=1}^m f\left( \hat r(\xv_j') \right) = 
    \frac{1}{m} \sum_{j=1}^m f\left(\frac{m \, \hat \eta(\xv_j')}{n (1- \hat \eta(\xv_j'))}  \right) \,.
\end{align}

To train a classifier, we split $X = X_1 \cup X_2$ and $X' = X_1' \cup X_2'$, train a probabilistic classifier such as a logistic regression model to separate $X_1$ from $X_2$ (train set) and evaluate the likelihood ratios on $X_1'$ and $X_2'$ (validation set) to estimate the $f$-divergence.

\myparagraph{Practical Considerations}
Logistic regression can fail to yield meaningful odds ratio estimates when the two distributions are well-separated. For evaluation of image generative models such as GANs, 
\citet{lopezpaz2017revisiting} found that neural networks on the pixel space capitalize on artifacts in the generated images, leading to perfect classification and therefore, poor likelihood ratio estimates.
To avoid this issue, we employ a linear model on frozen embeddings $\varphi: \Xcal \to \reals^d$. 

\section{Related Work} \label{sec:related}

We focus in this paper on information divergence-based scores to evaluate generative models. 
While the evaluation process is \textit{post hoc} and external to a generative model, it is worthwhile to mention the increasingly active research area analyzing (classes of) generative models and establishing theoretical results such as statistical consistency, universal approximation, sample complexity; see e.g.~\cite{biau2021some, schreuder2021statistical} and references therein. 
We review the related work on statistical trade-off curves, information divergence-based scores for texts and images, and theoretical results on the statistical estimation of information divergences in mathematical statistics and information theory.

\subsection{Divergence Frontiers for Generative Models}
\label{sec:related-frontiers}

\citet{sajjadi2018assessing} and \citet{kynknniemi2019improved} both proposed to account for the two types of errors of generative modeling using trade-off curves in the spirit of operation characteristics and precision-recall curves for binary classification and statistical detection~\cite{cortes2005confidence,clemencon2009precision,clemenccon2010overlaying,flach2012machine}.
In an inspiring paper, \citet{djolonga2020precision} proposed information divergence frontiers based on R\'enyi divergences thereby encompassing both~\cite{sajjadi2018assessing} and~\cite{kynknniemi2019improved}. 
The authors of \cite{djolonga2020precision} show how to compute the divergence frontiers in special cases such as exponential families.
Their exploration of statistical estimation via vector quantization leads to two observations. 
First, a small quantization size can lead to a bias of optimism, where $D_f(P_\Scal \Vert Q_\Scal) \le D_f(P \Vert Q)$ and this gap can be large when $|\Scal|$ is small.
Second, the statistical error from small sample sizes can lead to pessimistic estimates of the divergences. 
However, \cite{djolonga2020precision} do not provide statistical bounds for vector quantization nor do they analyze statistical properties of divergence frontiers defined using $f$-divergences. 
Moreover, the above research does not consider applications to open-ended text generation. %

We extend the above line of work, presenting a general framework for estimating divergence frontiers and their statistical summaries for generative models. Theoretically, we provide quantitative upper bounds for both the statistical error and quantization error. Specifically, we show that the statistical error is bounded by $\tilde O(\sqrt{k/n})$. 
Our bounds also demonstrate the interest of using smoothed distribution statistical estimators to account for the missing mass problem. 
We explore other estimation procedures based on nonparametric nearest-neighbor and kernel density estimation,
classifier-based estimation, and parametric Gaussian approximations. 
We also perform a thorough empirical evaluation and operationalize these scores for large-scale text and image models. Finally, based on our observations, we discuss practical recommendations, to facilitate the application to applied AI domains.  

After the publication of the conference paper~\cite{pillutla2021mauve}, subsequent work has corroborated that the original \mauve score compares favorably to other automatic metrics for evaluating neural text \citep{kour2022measuring}. 
\citet{pimentel2022cluster} corroborated the correlation between this score and human judgment.
    Their empirical analysis shows that a 5-gram estimation of \mauve%
    \footnote{
    This involves estimating $\mauve(P, Q) \approx \mauve(\hat P_{5\text{-gram}}, \hat Q_{5\text{-gram}})$ using 5-gram language models $\hat P_{5\text{-gram}}, \hat Q_{5\text{-gram}}$ fit to samples from $P, Q$ respectively.
    }
    has a much weaker correlation with human evaluations than the vector quantization procedure used in \cite{pillutla2021mauve} (cf. \S\ref{sec:mauve:quant}).
    Based on this analysis, \citet{pimentel2022cluster} conclude that 
    the key to the empirical success of \mauve is the vector quantization procedure. The experiments in \Cref{sec:experiments} indicate that the reality is much more nuanced. 
    Indeed, we show that the other nonparametric, parametric, and classifier-based estimation of \Cref{sec:compute} can be nearly as effective as vector quantization with the right hyperparameters (\S\ref{sec:expt:estimation}); thus the vector quantization cannot be the driving factor behind \mauve's usefulness as an evaluation metric. 
    We note, however, that vector quantization has several other benefits, including its simplicity and the availability of fast open-source implementations.
    Instead, we show that \emph{\mauve requires an embedding of text to vectors to work well in practice}: modern transformer language model embeddings as used in \cite{pillutla2021mauve} work well but simple non-contextual GloVe embeddings also work equally well (\S\ref{sec:expt:embedding}). However,
    estimation from string kernel embeddings\footnote{
        An example of a string kernel is the $N$-gram kernel defined in \S\ref{sec:string-kernel}; this is directly comparable with the $N$-gram estimation of \mauve in the analysis in \cite{pimentel2022cluster}.
    }
    (\S\ref{sec:string-kernel}) or direct estimation with language model probabilities (\S\ref{sec:direct-est}) both fail to quantify previously known trends.

The original \mauve score has since then been adopted by the language modeling and computational linguistics communities to measure performance and to tune hyper-parameters in diverse language generation settings, including the design of decoding algorithms \cite{meister2022locally,hewitt2022truncation,su2022contrastive,li2022contrastive,finlayson2023closing}, controllable text generation \cite{yang2022unified}, architectural innovations \cite{hu2022fuse}, and differentially private language generation \cite{mattern2022differentially,yue2022synthetic,kurakin2023harnessing}.

On the theoretical side, \citet{cheema2023precision} propose a variant of generative precision-recall of \cite{kynknniemi2019improved} and show that it converges to a well-defined population quantity. 
As shown in our preliminary conference paper~\cite{liu2021divergence} and elaborated on in this work, ($f$-)divergence frontiers can also be estimated in a statistically consistent manner with both vector quantization and nonparametric $k$-nearest neighbor-based approaches.

\subsection{Divergence Measures for Text}

Prior measures of similarity/divergence between machine text and human text come in three broad categories: (a) reference-based, (b) statistics-based, and (c) language modeling.

\textit{Reference-based metrics} evaluate generated text by comparing it with a (small set of) reference text sample(s), rather than comparing distributions over full sequence.
These include classical metrics for $n$-gram matching~\cite{papineni2002bleu,lin2004rouge,banerjee2005meteor}, which are designed to capture similarities in the surface form of the generated text and the human references, making them fundamentally ill-suited for open-ended generation. 
Moreover, it has been shown in \cite{novikova2017need} that these classical metrics only weakly agree with human judgments.

More recent reference-based metrics are capable of comparisons in a high dimensional embedding space~\cite{shimanaka2018ruse,zhang2020bertscore,sellam2020bleurt,clark2019sentence}, thereby capturing distributional semantics beyond superficial $n$-gram statistics.
For instance, Moverscore \cite{zhao2019moverscore} relies on the Word Mover's distance \cite{kusner2015word}, and is an instance of an optimal transport distance~\cite{villani2021topics}.
Moverscore computes the minimum cost of transforming the generated text to the reference text, taking into account the Euclidean distance between vector representations of $n$-grams, as well as their document frequencies.
The paradigm of reference-based metrics is useful for targeted generation tasks such as translation and summarization, where matching a set of references is paramount. 
However, this family of metrics is unsuitable for the open-ended generation task where there typically are several plausible continuations for each context and creative generations are desirable.
\citet{chan2022distribution} consider distribution-aware reference-based metrics for conditional generation tasks to account for the diversity in the output space.

\textit{Statistics-based metrics} compare
the model distribution $Q$ with respect to the human distribution $P$ on the basis of some statistic 
$T(P)$ and $T(Q)$. 
Property-specific statistics such as the amount of repetition \citep{holtzman2019curious,welleck2020neural}, verifiability \citep{massarelli2019decoding}, or termination \citep{welleck2020consistency} 
are orthogonal to \mauve, which provides a summary of the overall gap between $P$ and $Q$ rather than focusing on an individual property.
Another statistic is the generation perplexity \citep{fan2018heirarchical,holtzman2019curious},
which compares the perplexity of the model text $\xv \sim Q$ with that of human text $\xv' \sim P$ under an external model $R$.
We find in \Cref{sec:experiments} 
that generation perplexity fails to correctly capture the effect of the decoding algorithm and the text length. Moreover, it can easily be fooled by an adversarial decoder that generates gibberish text that nevertheless has the right perplexity, as we show in \Cref{sec:expt:properties}.

\textit{Language modeling metrics} calculate how (un)likely
human text $\xv\sim P$ is under the model distribution $Q$, for instance, using the probability $Q(\xv)$.  
These metrics are related to a single point on the divergence curve, rather than a full summary. 
Examples include the perplexity of the test set (which is a sample from $P$) under the model $Q$ and its generalizations to handle sparse distributions~\cite{martins2020sparse}. 
Unlike the proposed measures, these metrics never see model text samples $\xv' \sim Q$, so they cannot account for how likely the model text is under the human distribution $P$. 
Moreover, they cannot be used for decoding algorithms such as beam search which do not define a token-level distribution.

Automatic metrics have been proposed for specific domains such as generation of dialogues~\cite{tao2018ruber}, stories~\cite{guan2020union}, and others~\cite{optiz2021towards}. 
They capture task-specific properties; see the surveys~\cite{celikyilmaz2020evaluation,sai2020survey}.
In contrast, \mauve compares machine and human text in a domain-agnostic manner.
Other related work has proposed metrics that rely on multiple samples for quality-diversity evaluation  \cite{Caccia2020Language}, and Bayesian approaches to compare the distribution of statistics in machine translation \cite{eikema2020map}.

\citet{gehrmann2023repairing} point out the challenges involved in designing good automatic evaluation metrics with a focus on directed generation tasks. They outline many suggestions including continuously updated suites of datasets, documentation, and benchmarks, as well as a multi-dimensional evaluation with each metric focusing on a small yet more precisely defined scope.
\citet{liang2023holistic} advocate for a multi-metric approach for evaluating generated language, going beyond quality and considering specific attributes such as toxicity and bias.

\myparagraph{Non-Automatic Metrics} 
HUSE~\citep{hashimoto2019unifying} aims to combine human judgments of Type I errors with Type II errors measured using perplexity under $Q$. 
Due to the costs of human evaluation, we consider HUSE and other metrics requiring human evaluation, such as single-pair evaluation, complementary to the proposed automatic measures.
As a separate technical caveat, it is unclear how to use HUSE for sparse $Q$ that assigns zero probability to a subset of text, which is the case with state-of-the-art decoding algorithms~\cite{holtzman2019curious,martins2020sparse,meister2022locally}.

\subsection{Divergence Measures for Images}

Evaluation of generative models is an active area of research in computer vision, where implicit models including generative adversarial networks \cite{goodfellow2014generative} preclude even basic divergence evaluations based on test-set log-likelihoods.
The popular Inception Score \cite{salimans2016improved} is based on large-scale supervised classification tasks; it is unclear how to adapt this score to other modeling domains, such as open-ended text generation.
The Fr\'{e}chet Inception Distance~\cite{heusel2017gans,semeniuta2018accurate} and its unbiased counterpart, the Kernel Inception Distance \cite{bikowski2018demystifying} are both used for evaluating generative models, but, unlike divergence frontier methods, do not take into account trade-offs between different kinds of errors between the learned and the reference distribution.
We find in \Cref{sec:expt:properties} that the Fr\'echet distance adopted to the text setting fails to capture the dependence on the text length, while our proposed approach can. We note that this sequential temporal aspect is absent in the image modality. An exploration of this property of Fr\'echet distance and \mauve in other sequential modalities such as video~\cite{unterthiner2018towards} and speech~\cite{kilgour2019frechet} is an interesting direction for future work.

\subsection{Statistical Estimation of Information Divergences}

A closely related problem is the estimation of functionals of discrete distributions; see \citep{verdu2019survey} for an overview.
In particular, the estimation of KL divergences has been studied in both fixed and large alphabet regimes \citep{cai2006universal,zhang2014nonpar,bu2018kl,han2020minimax}.
An important result from this line of research is that the minimax quadratic risk of the na\"ive plug-in estimator is infinite \citep{bu2018kl}.
The main challenge arises from the missing mass phenomenon \citep{good1953frequency} which is especially prominent in the large alphabet regime.
This challenge can be addressed by applying add-constant smoothing \citep{krichevsky1981performance,braess2004bernstein} to the empirical distribution estimator and requiring the two distributions to have a bounded density ratio.
Our results also utilize add-constant smoothing without the need for the boundedness assumption.
Other choices of estimators include the Good-Turing~\cite{good1953frequency} and the absolute discounting~\cite{falahatgar2017power} estimators.

On the practical side, there is a new line of successful work that uses deep neural networks to find data-dependent vector quantization to estimate information-theoretic quantities from samples \citep{sablayrolles2018spreading,hamalainen2020deep}.
Our experimental results also rely on such data-dependent vector quantizers.

There exists a rich literature on statistical estimation of $f$-divergences using other methods. 
Nonparametric estimation of $f$-divergences via nearest-neighbor and kernel density estimation was studied in \cite{poczos2011nonparametric,moon2014ensemble,kandasamy2015nonparametric,noshad2017direct}, to name a few.
The variational expression for $f$-divergences was leveraged for optimization-based estimation in \cite{nguyen2010estimating,sreekumar2022neural}. Estimation under structural assumptions satisfied in applications such as autoencoders was considered in \cite{rubenstein2019practical}. 
While not directly related to statistical estimation, a general optimization-based methodology to derive sharp inequalities between various $f$-divergences was given in \cite{guntuboyina2014sharp}.
In contrast, we focus on vector quantization-based estimation while empirically comparing them to approaches based on nonparametric estimators, classifier-based estimation, and parametric approximation. 
\section{Experiments: Setup} \label{sec:expt-setup}

We consider open-ended text generation tasks, where the model has to generate text in continuation of a given text prompt. The open-endedness of the task is reflected in the relative lengths of the prompt and the generation: the prompt is often quite short ($35$ to $50$ tokens), while the generation is $10\times$ to $30 \times$ longer (approximately $500$ to $1000$ tokens).

\begin{table}[t!]
\centering
\begin{adjustbox}{max width=0.99\textwidth}
\begin{tabular}{llclrrr}
\toprule
\bf Task Domain & \bf Model & \bf Finetuning & \bf Dataset & 
\bf \begin{tabular}{r} Prompt \\ Length \end{tabular} & 
\bf \begin{tabular}{r} Max. \\ Generation \\ Length \end{tabular} & 
\bf \begin{tabular}{r} Number of \\ Generations \end{tabular} \\

\midrule

Web text & 
GPT-2 (all sizes) & 
 Pretrained  &
Webtext & $35$ tokens  & $1024$ tokens & $5000$
\\
News & 
Grover (all sizes)& 
Pretrained  & 
RealNews & varying & $1024$ tokens & $5000$
\\
Stories & GPT-2 medium & Finetuned &
WritingPrompts & $50$ tokens & $512$ tokens & $5000$ 
\\
\bottomrule
\end{tabular}
\end{adjustbox} \vspace*{0.5mm}
\caption{\small 
Dataset and task summary for open-ended text generation. Note that $1024$ tokens
correspond to $\sim750$ words on average.} 
\label{table:expt:details}
\end{table}

\subsection{Task Domains and Models}
We consider three different text domains: web text, news, and stories.
For each domain, we consider generation with size-based variants of transformer language models. See~\Cref{table:expt:details} for a summary.

\myparagraph{Web Text Generation}
The goal of this task is to generate articles from the publicly available analogue of the Webtext dataset\footnote{
\url{https://github.com/openai/gpt-2-output-dataset}
} 
using pre-trained GPT-2 models for various sizes~\cite{radford2019language,brown2020language}.
At generation time, we use as prompts the first $35$ tokens of each of the $5000$ articles from the Webtext test set, keeping the maximum generation length to $1024$ tokens (which corresponds, on average, to around $750$ words).
For comparison with human text, we use the corresponding human-written continuations from the test set (up to a maximum length of $1024$ tokens).

\myparagraph{News Generation}
Under this task, the goal is to generate the body of a news article, given the title and metadata (publication domain, date, author names).
We use a left-to-right transformer language model, Grover \cite{zellers2019grover}, which is similar to GPT-2 but tailored to generating news by conditioning on the metadata of the article as well. 
Our generations rely on pre-trained Grover architectures of various sizes.
The generation prompt comprises the headline and metadata of $5000$ randomly chosen articles from the ``April2019'' set 
of the RealNews dataset~\cite{zellers2019grover}, and the maximum article length was $1024$ tokens.
We reuse the publicly available Grover generations\footnote{
available at \url{https://github.com/rowanz/grover/tree/master/generation_examples}
} for our evaluation.

\myparagraph{Story Continuation}
Given a situation and a (human-written) starting of the story as a prompt,
the goal of this task is to continue the story. Here, we use a GPT-2 medium model fine-tuned for one epoch on the WritingPrompts dataset~\cite{fan2018heirarchical}. 
We use as generation prompts the first $50$ tokens of $5000$ randomly chosen samples of the test set of WritingPrompts.
The model generations are allowed to be up to $512$ tokens long. The 
corresponding test examples, truncated at $512$ tokens are used as samples from $P$.

\subsection{Decoding Algorithms}
We consider three common decoding algorithms described in \Cref{sec:bg:lms}.
\begin{enumerate}[noitemsep,topsep=0pt, label=(\alph*)]
    \item \textbf{Greedy decoding} selects the most likely next token $x_{t+1}=\arg\max_{x\in \mathcal{V}} \hat P(x\mid\xv_{1:t})$ and is representative of a broader class of approximate likelihood maximization decoding algorithms.
    \item \textbf{Ancestral sampling} samples directly from the language model's per-step distributions as $x_{t+1}\sim \hat P( \, \cdot \mid \xv_{1:t})$, and generates unbiased samples from the model distribution. 
    \item \textbf{Nucleus sampling}~\citep{holtzman2019curious} samples from top-$p$ truncated per-step distributions, $x_{t+1}\sim \hat Q_{\text{nuc},p}(\,\cdot\mid \xv_{1:t})$ as defined in \Cref{eq:nucleus-def}. 
\end{enumerate}

Greedy decoding attempts to find text that approximately maximizes its likelihood under the model. While such algorithms are highly successful for directed text generation tasks such as translation, they produce highly degenerate repetitive text in the open-ended setting. While ancestral sampling produces unbiased samples from the model distribution, it also has been found to generate degenerate text~\cite{holtzman2019curious}, ostensibly because the model is imperfect, especially in the low-probability tail of the next-token distribution. Nucleus sampling attempts to fix this by truncating the tail and is representative of the broader class of truncated sampling methods that are now widely considered state-of-the-art. 
We vary the nucleus parameter $p \in \{0.9, 0.92, 0.95, 0.99\}$ for web text generation
and story continuation, and $p \in \{0.9, 0.92, 0.94, 0.96, 0.98\}$ for 
news generation.

In addition, we also consider the following decoding algorithms:
\begin{enumerate}[noitemsep,topsep=0pt, label=(\alph*)]
\setcounter{enumi}{3}
    \item \textbf{Beam search} is a more sophisticated approximate likelihood maximization algorithm that maintains a set of $b$ promising prefixes. At each time step, all possible one-token continuations of the current $b$ prefixes are considered and the top $b$ of them are retained.
     \item \textbf{Locally typical sampling}~\cite{meister2022locally} is a truncation sampling method, which we use as a representative of recent truncation-based decoding algorithms such as Mirostat~\cite{basu2021mirostat} and $\eta$-sampling~\cite{hewitt2022truncation}. 
     Locally typical sampling with hyperparameter $\tau \in (0, 1)$ samples the next token from the truncated vocabulary 
     \[
        V_{\text{typ}, \tau} = 
        \argmin_{V'}\left\{ 
            \sum_{x \in V'} \left|\log \hat P(x | \xv_{1:t}) + H\big(\hat P(\cdot | \xv_{1:t})\big) \right| \,:\,
            \sum_{x \in V'} \log \hat P(x | \xv_{1:t}) \ge \tau
        \right\}
     \]
     of the language model $\hat P$, 
     where $H(p) = -\sum_{x \in V} p(x) \log p(x)$ is the Shannon entropy. 
     This is a set that covers at least $\tau$-fraction of the probability mass but also has log probabilities that are as close to the conditional entropy as possible. 
     The samples are obtained by sampling from this truncated distribution as
     \begin{align*} 
        Q_{\text{typ},\tau}(x_{t+1} \mid \xv_{1:t}) 
        =
        \begin{cases}
        \frac{1}{Z} \, \hat P(x_{t+1} \mid \xv_{1:t}), &\text{if } x_{t+1} \in V_{\text{typ}, \tau} , \\
        0, & \text{else},
        \end{cases}
    \end{align*}
    where $Z$ is a normalizing constant.
    
    \item \textbf{Adversarial perplexity sampling} is designed to generate low-quality text that nevertheless matches the perplexity of human text. Adversarial perplexity sampling proceeds in two phases: (1) we generate the first 15\% of tokens in a sequence uniformly at random from the vocabulary, and (2) we generate the remaining tokens greedily to make the running perplexity of the generated sequence as close as possible to the perplexity of human text. 
\end{enumerate}

\subsection{Baseline Metrics} 

We compare the proposed measures to the following automatic evaluation metrics used previously to evaluate open-ended generation. 
\begin{itemize}[nosep]
    \item \textbf{Generation Perplexity (Gen. PPL.)}: We compute the perplexity of the generated text under the GPT-2 large model. 
    A common heuristic is to match 
    \item \textbf{Zipf Coefficient}: we report the slope of the best-fit line 
        on the log-log plot of the rank versus unigram frequency plot. Note that the Zipf coefficient only depends on unigram count statistics and is invariant to, for instance, permuting the generations. We use the publicly available implementation of \cite{holtzman2019curious}.\footnote{ \url{https://github.com/ari-holtzman/degen/blob/master/metrics/zipf.py}}
    \item \textbf{Repetition Frequency (Rep.)}: The fraction of generations which devolved into repetitions. Any generation that contains at least two contiguous copies of the same phrase of any length appearing at the 
    end of a phrase is considered a repetition. We consider repetitions at the token level. 
    This metric is useful to quantify degenerate repetitiveness that sometimes comes up with neural text (e.g., with greedy decoding). 
    \item \textbf{Distinct-$n$}: The fraction of distinct $n$-grams from all possible $n$-grams across all generations. We use $n = 4$. This is a measure of how diverse the generated text is.
    \item \textbf{Self-BLEU}: Self-BLEU is calculated by computing the  BLEU score of each generation against all other generations as references. We report the Self-BLEU using $4$-grams. 
    This operation is extremely expensive, so we follow the protocol of \cite{holtzman2019curious}: sample $1000$ generations and compute the BLEU against all other $4999$ generations. A lower Self-BLEU score implies higher diversity.
    \item \textbf{Discriminator Accuracy}: We train a binary classifier to classify text as human or not. A smaller discrimination accuracy means that model text is harder to distinguish from human text. A separate classifier is trained for each model and decoding algorithm pair. For the story continuation task, we 
    train a classification head on a frozen GPT-2 large model
    using the logistic loss.
    We use $25\%$ of the data as a test set and the rest for training;
    a regularization parameter is selected with $5$-fold cross-validation. For the news dataset, we follow the protocol of \cite{zellers2019grover}, i.e., a Grover large model 
    finetuned with a binary classification head.
\end{itemize}

Apart from discriminator accuracy, every other metric quantifies a property $T(Q)$ of the distribution $Q$ of the generated text. This number makes sense only in comparison to the corresponding quantity $T(P)$ of the human text distribution $P$. For each of these, we use $|T(Q) - T(P)|$ as a measure of the gap between $P$ and $Q$. 

\subsection{Human Judgements and Evaluation of Automatic Metrics}
\label{sec:expt-setup:human-eval}

An effective metric should yield judgments that correlate highly with human judgments, assuming that human evaluators represent a gold standard.\footnote{While recent work has shown that human evaluation might not always be consistent \cite{clark2021all,karpinska2021perils,gehrmann2023repairing}, human judgments continue to be the gold standard for evaluating open-ended text generation.}
We evaluate how the quality judgments of the proposed measures correlate with human quality judgments.
In our study, a quality judgment means choosing a particular (model, decoder) setting based on the resultant generations.

\myparagraph{Evaluation Protocol}
Since our goal is to measure the gap between a model text distribution $Q$ and a human text distribution $P$, we employ a pairwise setup for human evaluations. At each round, an annotator receives a context and continuations from two different (model, decoder) settings, and selects the continuation they found more 
(a) human-like, 
(b) interesting, and 
(c) sensible
on a 5-point Likert scale.
Our interface for collecting annotations is shown in \Cref{fig:mauve:human-eval-interface}
of \Cref{sec:a:human-eval}.

We collect these annotations for web text generation with 8 different (model, decoder) settings plus a ninth setting for human-written continuations. Each setting is a GPT-2 model size paired with either ancestral or nucleus sampling. This gives us a total of $36$ pairs of settings.
Given the known difficulties with human evaluation of longer texts \cite{ippolito2020automatic}, we use a maximum completion length of $256$ tokens.
We obtain $90$ preference ratings for each pair of settings, coming from a total of $214$ crowd-workers from the Amazon Mechanical Turk platform.
The evaluators were paid USD $0.40$ per evaluation based on an estimated wage of USD $16$ per hour.

\myparagraph{Pairwise Scores to a Ranking}
We convert these pairwise preferences to a ranking by fitting a Bradley-Terry model~\cite{bt:book:1995}, a parametric model used to predict the outcome of a head-to-head comparison. In particular, we obtain a score $w_i$ for each setting $i$ so that the log odds of humans preferring setting $i$ to setting $j$ in a head-to-head comparison is given by the difference $w_i - w_j$.

\myparagraph{Evaluation of Automatic Metrics}
Consider an automatic metric $M$ with mean values 
$\av = (a_1, \ldots, a_n)$ and standard deviations $\sv = (s_1, \ldots, s_n)$ across $n$ different (model, decoder) pairs, where the mean and standard deviation is over repetitions with multiple random seeds. We assume that higher values of the metric mean that the text is closer to human text. Let $\hv = (h_1, \ldots, h_n)$ denote the Bradley-Terry coefficients obtained from the human evaluation protocol designed above. 
We evaluate the automatic metric $M$ by comparing the ranking it induces over the (model, decoder) pairs to that obtained by the human evaluation using the Spearman rank correlation. 

In order to account for the standard deviation of the metric, we define the \textbf{worst-case Spearman rank correlation} between $a_1 \pm s_1, \ldots, a_n \pm s_n$ with $
\hv = (h_1, \ldots h_n)$ as 
\begin{align} \label{eq:worse-case-spearman}
    \rho_{\min}(\av, \sv, \hv) 
    = \min_{\sigma_1, \ldots, \sigma_n \in \{-1, 1\}^n} \rho\big(  
    (a_i + \sigma_i s_i)_{i=1}^n, \hv
    \big) \,,
\end{align}
where $\rho(\av, \hv)$ denotes the Spearman rank correlation between $\av$ and $\hv$. 
The end result is a correlation score in $[-1,1]$, with higher values meaning that quality judgments using the automatic metric correlate with quality judgments made by human evaluators up to one standard deviation from the randomness of sampling.

\subsection{Hyperparameters}
By default, we summarize the divergence frontier with $\mauve_\kl$ computed using $k$-means vector quantization (\Cref{alg:mauve:quant}) with $k=500$ buckets. 
Following the discussion in \Cref{sec:mauve:quant}, we use the Krichevsky–Trofimov (add-$1/2$) smoothing. 
This is different from the default setting of \cite{pillutla2021mauve}, 
where the empirical estimator is used instead (with the other hyperparameters remaining the same).
To make this distinction clear, we refer to the version computed by the smoothed estimator as $\mauve_\kl^\star$
and the original version of \cite{pillutla2021mauve} as $\mauve_\kl$ (or $\mauve^\star$ and $\mauve$ respectively when the $\kl$ divergence is clear from the context). 
We compare this choice with different estimation methods in 
\Cref{sec:expt:estimation} and different divergence frontier summaries in \Cref{sec:expt:other-divergence}. 
\section{Experimental Results} \label{sec:experiments}

We present the main experimental results in this section. 
We start by comparing the rankings induced by \mauve to that of the human evaluators in \Cref{sec:expt:human-eval}.
Next, we demonstrate in \Cref{sec:expt:properties} that the proposed measures can quantify how the properties of the generated text vary with model size, decoding algorithms, and text length.
Then, we compare in \Cref{sec:expt:estimation} the different statistical estimation methods discussed in \Cref{sec:compute}. 
We perform a detailed comparison of various $f$-divergence and optimal transport-based alternatives in \Cref{sec:expt:other-divergence}. 
We demonstrate the effect of the embedding model in \Cref{sec:expt:embedding}, and explore the applicability of generative precision-recall~\cite{kynknniemi2019improved}, originally proposed for the vision modality, to the natural language modality in \Cref{sec:expt-kynka-pr}. 
Finally, 
we go beyond the language domain
to show how the proposed methods can be useful in the vision modality in \Cref{sec:expt:vision}.

\begin{table*}[t!]
\centering
\begin{adjustbox}{max width=0.99\textwidth}
\begin{tabular}{llcccccc}
\toprule
\textbf{Metric} & \textbf{Task} & \textbf{Gen. PPL} & \textbf{Zipf Coef.} & \textbf{REP} & \textbf{Distinct-4} & \textbf{Self-BLEU} & \textbf{$\name^\star$} \\
\midrule
       Human-like/BT &      Web text &           $0.810$ &             $0.762$ &     $-0.500$ &             $0.738$ &            $0.500$ &        \tabemph{$\mathbf{0.857}$} \\
      Interesting/BT &      Web text &           $0.643$ &             $0.405$ &     $-0.571$ &             $0.524$ &            $0.262$ &        \tabemph{$\mathbf{0.714}$} \\
         Sensible/BT &      Web text &           $0.738$ &             $0.643$ &     $-0.476$ &             $0.595$ &            $0.452$ &        \tabemph{$\mathbf{0.762}$} \\
      Disc. Acc. &          News &           $0.468$ &             $0.595$ &      $0.792$ &             $0.653$ &            $0.516$ &        \tabemph{$\mathbf{0.956}$} \\
       Disc. Acc. &       Stories &           $0.690$ & 	$0.762$ 	& $0.190$ & 	$0.833$ & 	\tabemph{$\mathbf{0.905}$} & 		\tabemph{$\mathbf{0.905}$} \\
\bottomrule
\end{tabular}

 \end{adjustbox}
\caption{\small
Correlation of various automatic metrics with human judgments when available, and the accuracy of a trained discriminator otherwise. 
``BT'' denotes the Bradley-Terry score for a pairwise human evaluation. We show the worst-case Spearman rank correlation defined in \eqref{eq:worse-case-spearman} for the BT scores. Boldfaced/highlighted numbers indicate the highest correlation in each row. 
}
\label{tab:mauve:expt:gold-correlation-main}
\end{table*}

\subsection{Comparison to Human Evaluation}
\label{sec:expt:human-eval}

We now compare the ranking induced by the proposed measure to that of the human evaluation scores. 

\myparagraph{Correlation with Human Judgments}
\Cref{tab:mauve:expt:gold-correlation-main} shows the correlation between human judgments and five automatic evaluation metrics obtained using our evaluation protocol on the web text domain.
\mauve correlates highly with human judgments of how human-like ($0.857$), interesting $(0.714)$, and sensible $(0.762)$ the machine text is.
\mauve's correlations with human judgments are substantially higher than those for the other automated measures; for instance, the commonly used generation perplexity has correlations that are $0.810$, $0.643$, and $0.738$ respectively.
The results suggest that the proposed measures may act as an effective, automatic surrogate for costly human judgments.

\myparagraph{Correlation with Learned Discriminators}
We also measure the quality of generations by how well a trained model (a discriminator) can distinguish between real and generated text~\cite{lopezpaz2017revisiting}. 
We report the test accuracy of a binary classifier trained to discriminate between machine and human text; a lower discrimination accuracy implies that the generation is harder to distinguish from human text.
We report the accuracy of 
Grover-large as the discriminator for the news generations
as it produced the highest discrimination accuracy~\cite{zellers2019grover}
while we use GPT-2 large for the story domain.
As seen in \Cref{tab:mauve:expt:gold-correlation-main}, \mauve correlates the highest with the discrimination accuracy ($0.956$ for news and $0.905$ for stories) among all comparison measures. 
Computing the discrimination accuracy for each (model, decoder) pair requires fine-tuning a separate model, which is particularly expensive for large models. 
The proposed measures, on the other hand, do not require any training when computed using vector quantization.

\myparagraph{Disagreements between \mauve and Human Judgements}
\Cref{tab:mauve:expt:webtext-full} gives the values of \mauve and the Bradley-Terry coefficients of the human evaluation for how human-like the text is. 
Human evaluators find GPT-2 xl with ancestral sampling (BT score of $8.97$) to produce text that is more human-like than GPT-2 medium with
nucleus sampling (BT score of $-3.43$), while their \mauve scores are $0.908$ and $0.936$ respectively. Similarly, \mauve finds GPT-2 large with ancestral sampling to be worse than GPT-2 small with nucleus sampling, while human evaluators disagree. \mauve agrees with the human evaluators on all other pairwise comparisons.

\begin{table*}[t!]
\centering
\begin{adjustbox}{max width=\textwidth}
\begin{tabular}{lllllllrl}
\toprule
\bf GPT-2  Size    &   \bf    Decoding             &      \bf              Gen. PPL &           \bf     Zipf Coef. &               \bf        Rep. &       \bf         Distinct-4 &           \bf      Self-BLEU & \bf Human($\uparrow$) & \bf $\mauve^\star$ ($\uparrow$) \\
\midrule
\multirow{4}{*}{small} & Sampling &          $101.880_{0.627}$ &           $0.926_{0.001}$ &           $0.001_{0.000}$ &           $0.941_{0.001}$ &           $0.327_{0.003}$ &                      $-27.52$ &            $0.655_{0.018}$ \\
      & Greedy &                    $1.224$ &                   $1.037$ &                   $0.942$ &                   $0.072$ &           $0.465_{0.000}$ &                 --            &                   $0.019$ \\
      & Nucleus, $0.9$ &           $23.788_{0.144}$ &           $1.012_{0.002}$ &           $0.010_{0.001}$ &           $0.859_{0.002}$ &           $0.436_{0.004}$ &              $-15.78$ &            $0.906_{0.005}$ \\
      & Adversarial &            \tabemph{$\mathbf{12.554}$} &           $1.073$ &           $0.006$ &           $0.365$ &           $0.525$ &             -- &            $0.043$ \\
\midrule
\multirow{4}{*}{medium} & Sampling &          $129.263_{0.798}$ &           $0.872_{0.001}$ &           $0.001_{0.000}$ &           $0.953_{0.001}$ &           $0.281_{0.002}$ &              $-30.77$ &           $0.446_{0.010}$ \\
      & Greedy &                    $1.241$ &                   $0.978$ &                   $0.903$ &                   $0.091$ &                   $0.415$ &                 --        &                    $0.024$ \\
      & Nucleus, $0.9$ &           $21.073_{0.134}$ &           \tabemph{$\mathbf{0.957}_{0.001}$} &           $0.005_{0.001}$ &  \tabemph{$\mathbf{0.884}_{0.001}$} &  \tabemph{$\mathbf{0.402}_{0.003}$} &                  $-3.43$ &            $0.936_{0.004}$ \\
      & Adversarial &           \tabemph{$\mathbf{12.554}$} &           $1.006$ &           $0.005$ &           $0.381$ &           $0.444$ &            -- &            $0.044$ \\
\midrule
\multirow{4}{*}{large} & Sampling &           $30.080_{0.196}$ &           $0.930_{0.002}$ &  \tabemph{$\mathbf{0.002}_{0.001}$} &           $0.916_{0.001}$ &           $0.358_{0.001}$ &                $-6.93$ &         $0.878_{0.008}$ \\
      & Greedy &                    $1.232$ &                   $0.983$ &                   $0.881$ &                   $0.100$ &                   $0.413$ &       --          &                    $0.026$ \\
      & Nucleus, $0.95$ &  $13.499_{0.058}$ &  $0.967_{0.002}$ &           $0.006_{0.001}$ &           $0.870_{0.001}$ &           $0.412_{0.002}$ &             $12.55$ &            $0.952_{0.002}$ \\
      & Adversarial &           \tabemph{$\mathbf{12.554}$} &           $0.965$ &           $0.005$ &           $0.395$ &           $0.429$ &        -- &            $0.035$ \\
\midrule[0.03em]
\multirow{4}{*}{xl} & Sampling &           $31.886_{0.447}$ &           $0.930_{0.001}$ &           $0.002_{0.001}$ &           $0.913_{0.001}$ &           $0.360_{0.003}$ &            $8.97$ &            $0.908_{0.005}$ \\
      & Greedy &                    $1.278$ &                   $0.975$ &                   $0.859$ &                   $0.115$ &                   $0.417$ &               --        &                    $0.033$ \\
      & Nucleus, $0.95$ &           $14.143_{0.043}$ &           $0.966_{0.002}$ &           $0.005_{0.000}$ &           $0.868_{0.001}$ &           $0.413_{0.002}$ &      \tabemph{$\mathbf{15.66}$} &   
      \tabemph{$\mathbf{0.955_{0.004}}$} \\
      & Adversarial &           \tabemph{$\mathbf{12.554}$} &           $0.986$ &           $0.005$ &           $0.397$ &           $0.448$ &      -- &            $0.057$ \\
\midrule
Human &         n/a          &                   $12.602$ &                   $0.952$ &                   $0.002$ &                   $0.878$ &                   $0.382$ &         $47.25$ &            --                \\
\bottomrule
\end{tabular}
 \end{adjustbox}
\caption{ \small
Automatic metrics across different model sizes and decoding approaches for web text generations. 
Subscripts indicate the standard deviation across 5 runs for the sampling-based methods; greedy decoding, being deterministic, always returns the same value for a given model.
For nucleus sampling, we show the best hyperparameter value from $\{0.9, 0.92, 0.95, 0.99\}$ as per \mauve.
The column ``Human'' gives the Bradley-Terry score obtained from how human-like the text is (\Cref{sec:expt-setup:human-eval}).
Boldfaced numbers indicate the best performance according to the metric, or closest to the human reference, when applicable. 
}
\label{tab:mauve:expt:webtext-full}
\end{table*}

\subsection{Quantifying the Effect of Model Size, Decoding, Text Length}
\label{sec:expt:properties}

To study the effectiveness of the proposed measures for comparing text distributions, we first examine how they quantify known properties of generated text: a good metric should meet expected behavior that is known from existing research on each property.
Specifically, we investigate how \mauve behaves under changes in model size, decoding algorithm, and generation length.
We give the results of web text generation in \Cref{tab:mauve:expt:webtext-full}; the corresponding results for the other domains can be found in \Cref{sec:a:expt}.

\myparagraph{Effect of the Model Size}
Scaling the model size has been a critical driver of recent advances in natural language processing, with larger models leading to better language modeling and higher-quality open-ended generation.
An effective metric should capture the relationship between model size and generation quality, which we verify with human evaluations.

We see from \Cref{tab:mauve:expt:webtext-full}
that \mauve increases as the model size increases, agreeing with the human evaluation and the expectation that larger models should have higher quality generations.
The widely-used generation perplexity, however, incorrectly rates the large model's text as better than the xl model.
In this case, human evaluators rate generations from the small model better than those from the medium model. 
Interestingly, \mauve and Gen. PPL. both identify this relationship, agreeing with the human ratings, in contrast to the other automatic metrics we surveyed.

\begin{table*}[t!]
\centering
\begin{adjustbox}{max width=0.95\textwidth}
{ \small
\begin{tabular}{lccccccc}
\toprule
\multirow{2}{*}{\textbf{Decoding}}
& \multirow{2}{*}{Greedy} & 
\multicolumn{2}{c}{Beam} 
& 
\multicolumn{2}{c}{Beam + no $4$-gram repeat}
& 
\multirow{2}{*}{Ancestral} & 
\multirow{2}{*}{Nucleus} \\
\cmidrule(lr){3-4}
\cmidrule(lr){5-6}
& & $b=4$ & $b=8$ & $b=4$ & $b=8$ & & \\
\midrule
$\textbf{\mauve}^\star$ & 
$0.019$ & 
$0.040$ & 
$0.049$ & 
$0.438$ & 
$0.415$ & 
$0.655_{0.021}$ & 
\tabemph{} $\mathbf{0.906_{0.005}}$  \\
\bottomrule
\end{tabular} }
\end{adjustbox}
\caption{\small
Beam search with beam sizes $b=4,8$ (with and without allowing $4$-gram repetitions) versus other decoding algorithms of \Cref{tab:mauve:expt:webtext-full} for web text generation with GPT-2 small. The subscript denotes the standard deviation over $5$ random seeds and is omitted for the deterministic greedy decoding and beam search.
}
\label{tab:mauve:expt:beam-search}
\end{table*}

\begin{table*}[t!]
\centering
\begin{adjustbox}{max width=0.95\textwidth}
{ \small
\begin{tabular}{lccccccc}
\toprule
\multirow{2}{*}{
\textbf{Decoding}
}
& 
\multicolumn{6}{c}{Locally Typical Sampling}
& \multirow{2}{*}{Nucleus}\\
\cmidrule{2-7}
&
$\tau=0.2$ & $\tau=0.5$ & $\tau=0.7$ & $\tau=0.9$ & $\tau=0.95$ & $\tau=0.99$ &  
 \\
\midrule
$\textbf{\mauve}^\star$ & 
$0.862_{0.012}$ &
$0.896_{0.005}$ & 
$0.88_{0.01}$ & 
$0.939_{0.009}$ & 
$\tabemph{} \mathbf{0.950}_{0.005}$ & 
$0.914_{0.007}$ & 
$\tabemph{} \mathbf{0.952}_{0.003}$
\\
\bottomrule
\end{tabular} }
\end{adjustbox}
\caption{\small
Comparing locally typical sampling~\cite{meister2022locally} 
to nucleus sampling ($p=0.95$) with \mauve for web text generations from GPT-2 large. The subscript denotes the standard deviation over $5$ random seeds. 
}
\label{tab:mauve:expt:typical}
\end{table*}

\myparagraph{Effect of the Decoding Algorithm}
Recent work has identified two clear trends in open-ended text generation with standard autoregressive models: (1) using greedy decoding 
results in repetitive, degenerate text~\cite{holtzman2019curious,welleck2020neural,welleck2020consistency}; (2) nucleus sampling 
(and related truncated sampling methods)  with the right hyperparameter
yields higher quality text than ancestral sampling~\citep{fan2018heirarchical,holtzman2019curious}.
An effective measure should thus indicate the quality relationship $\text{greedy}\prec\text{ancestral}\prec{\text{nucleus}}$.

We see from \Cref{tab:mauve:expt:webtext-full} that \mauve correctly identifies the expected quality relationship, assigning the lowest quality to greedy decoding for the xl model followed by ancestral sampling, and the highest quality to nucleus sampling for all model sizes --- these values are $0.016, 0.882, 0.940$ respectively for the xl model.
Other commonly used metrics fail to identify this relationship: generation perplexity rates the highly degenerate greedy-decoded text as better than ancestral sampling (a difference of $11.324$ w.r.t. the human perplexity vs. $19.284$). 
Furthermore, generation perplexity falls victim to the adversarial decoder that produces gibberish text.
\mauve, on the other hand, rightly rates it poorly.

\noindent We see in \Cref{tab:mauve:expt:beam-search} that \mauve identifies degeneracy of beam search, thus quantifying the qualitative observations of \citet{holtzman2019curious}.
Next, \Cref{tab:mauve:expt:typical} shows that 
locally typical sampling produces text that is comparable in its \mauve score to nucleus sampling and outperforms other decoding algorithms, echoing the results of \citet{meister2022locally}.

\begin{figure*}[t]
\includegraphics[width=\linewidth]{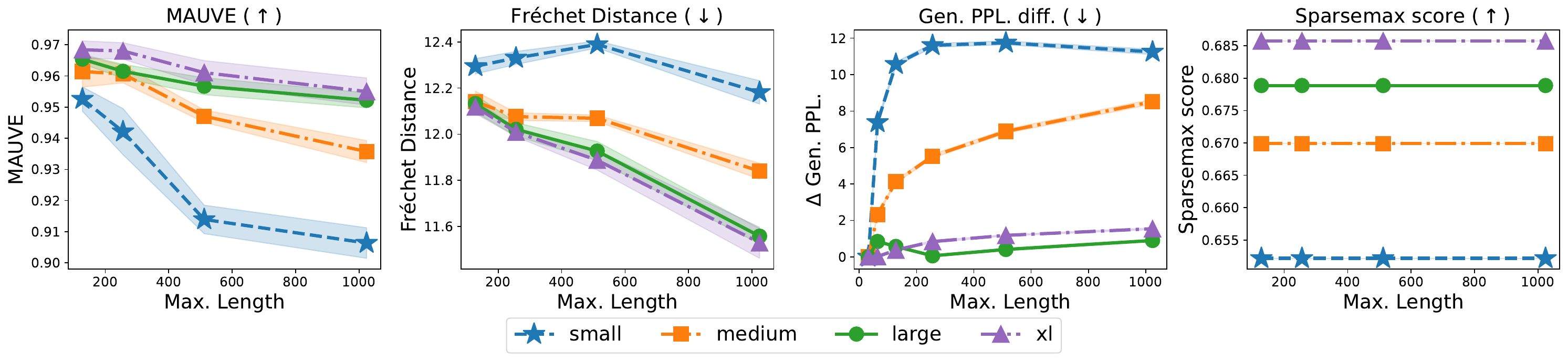}
\caption{
\small
Generation quality versus maximum generation length according to \name and three alternative measures (web text, GPT-2).
\name is the only comparison measure that identifies that generation quality decreases monotonically with increasing text length. 
The shaded area shows one standard deviation over generations from 5 random seeds. 
}
\label{fig:expt:length:main}
\end{figure*}

\myparagraph{Effect of the Generation Length}
Although large transformer-based models can generate remarkably fluent text, it has been observed that the quality of generation deteriorates with text length: 
as the generation gets longer, the model starts to wander, switching to unrelated topics and becoming incoherent~\cite{rashkin2020plotmachines}.
As a result, an effective measure should indicate lower quality (e.g. lower \mauve) as generation length increases. 

\Cref{fig:expt:length:main} shows \name as the generation length increases, along with three alternative metrics: generation perplexity, sparsemax score~\cite{martins2020sparse}, and Fr\'echet distance~\cite{heusel2017gans,semeniuta2018accurate}. 
\mauve reflects the desired behavior, showing a decrease in quality as generation length grows, with the trend consistent across model sizes. 
The other three metrics, however, show less favorable trends.
Fr\'echet distance indicates \textit{improving} quality as the length increases, while generation perplexity shows non-monotonic quality trends for the small and large models.
Finally, language modeling metrics such as the sparsemax score~\cite{martins2020sparse} remain constant, since they do not depend on the samples generated.

\subsection{Comparison of Statistical Estimation Methods}
\label{sec:expt:estimation}

We now compare different methods of estimating \mauve from \Cref{sec:compute} as well as direct estimation from model probabilities.

\subsubsection{Comparison of Vector Quantization with Other Approximations}
We now compare the different statistical estimation methods from \Cref{sec:compute}: vector quantization (\Cref{alg:mauve:quant} with Krichevsky–Trofimov smoothing, our default), nearest neighbor estimation (\Cref{alg:mauve:knn}), kernel density estimation (\Cref{alg:mauve:knn} modified as in \Cref{sec:kde-mauve}), 
classifier-based estimation (\Cref{sec:mauve:classifier}), and 
parametric approximation (\Cref{sec:a:parametric}). 
We also compare these to the direct estimation of \mauve based on model probabilities. 
We perform these comparisons for the web text domain. 

\begin{figure}[t]
    \centering
    \includegraphics[width=0.99\linewidth]{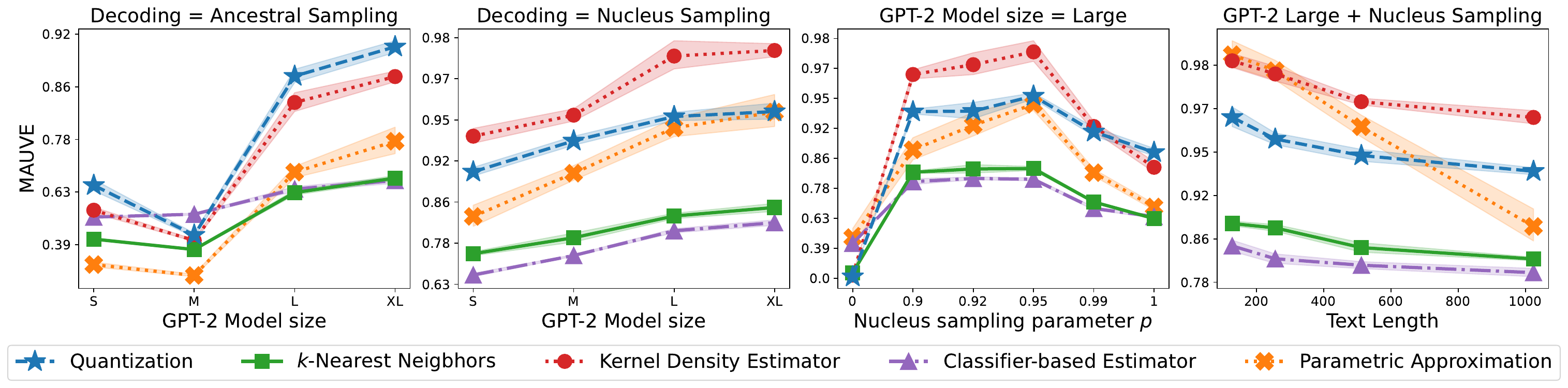}

    \includegraphics[width=0.99\linewidth]{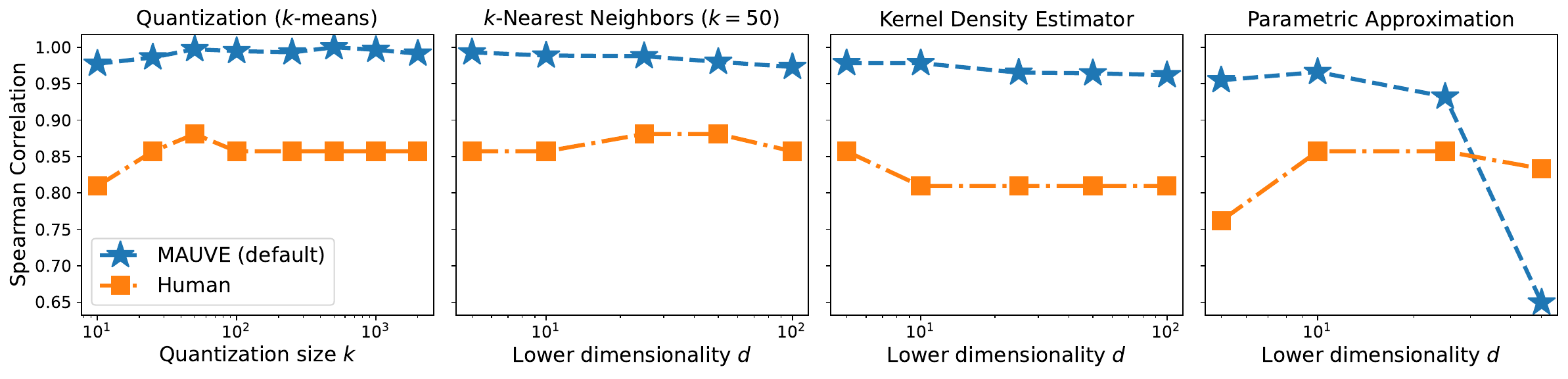}
    \caption{\small
    Computing \mauve using different estimation procedures of \Cref{sec:compute}. \textbf{Top row}: Trends from varying model size, decoding algorithm, and text length. 
    \textbf{Bottom row}: Effect of estimation hyperparameters on the correlation with the default setting of \mauve (vector quantization with $k=500$) and human evaluations from \Cref{tab:mauve:expt:webtext-full}. 
    These correlations for the classifier-based estimator are $\textbf{0.979}$ and $\textbf{0.857}$ respectively.  %
    }
    \label{fig:estimation}
\end{figure}

\myparagraph{Hyperparameters} 
The non-parametric nearest neighbor and kernel density estimators and the parametric approximation require the $d=1024$ dimensional embeddings to be projected into a small $m$-dimensional subspace. We use the first $m$ principal components of the embeddings. Empirically, we find that the monotonicity property $\kl(P \Vert \lambda_1 P + (1-\lambda_1) Q) \le \kl(P \Vert \lambda_2 P + (1-\lambda_2) Q)$ for $\lambda_1 \ge \lambda_2$ can fail to hold in the non-parametric and parametric estimates if $m > 100$. This is a manifestation of the well-known curse of dimensionality for non-parametric estimation and the Monte-Carlo estimation (\Cref{eq:parametric_est_fdiv:monte-carlo} in \Cref{sec:a:parametric}) required by our parametric approximation. 
Practically, the failure of the monotonicity property makes it challenging to estimate the area under the curve.
We employ a $\ell_2$-regularization term to the classifier-based estimator and found the results to be robust to the choice of the regularization parameter in the range $1/n$ to $10^{-3}/n$ where $n$ is the number of samples. 
We use a different scaling constant $c$ within the exponential (cf. \eqref{eq:mauve:area-def}) 
for each method:  $c=5$ for vector quantization, $c=10$ for nearest neighbor and kernel density estimation, $c=2.5$ for classification, and $c=1$ for the Gaussian approximation. Note that this does not change the induced rankings.

\myparagraph{Results}
The results are given in \Cref{fig:estimation}.
We see that each of the estimation methods can identify most of the trends of \Cref{sec:expt:properties}. 
As a notable exception, the classifier-based estimate fails to identify the trend that the GPT-2 small model with ancestral sampling is better than the medium one (cf. \Cref{tab:mauve:expt:webtext-full}). 
Notably, the parametric approximation identifies the correct dependence on the text length while the parametric approximation of the optimal transport cost, namely the Fr\'echet distance fails to capture this trend (cf. \Cref{fig:expt:length:main}). 
Interestingly, $m=5$ or $10$ principal components of the embeddings allow us to capture the trends with respect to the model size, decoding algorithms, and text length. 

\myparagraph{Correlation Analysis}
We note that each estimation method exhibits a high Spearman rank correlation with the default vector quantization approach of $0.95$ to $1.0$ and a worst-case Spearman correlation of at least $0.857$ with the human evaluations \emph{for the best hyperparameter values}. 
We find that the parametric approximation is not robust to the number $m$ of principal components --- its performance steeply drops off at $m = 100$. 

\myparagraph{Pros and Cons of the Estimation Methods}
All the tested estimation methods are consistent with each other, demonstrating the versatility of \mauve's recipe of estimating information divergences from vector embeddings of data. 
However, there are some minor differences.
First,
the $k$-nearest neighbor and classifier-based estimators report a tie between nucleus sampling with $p=0.9$ and $p=0.95$. In contrast, the vector quantization approach ranks $p=0.95$ as better than $p=0.9$; this is also the case with the Gen. PPL. baseline.
Second, the non-parametric nearest neighbor and kernel density estimators, as well as the parametric Gaussian approximation require extreme dimensionality reduction, which makes it important to select the lower dimension correctly. In contrast, the quantization performance is more robust to its hyperparameter (the quantization size $k$).
Thus, we recommend the vector quantization approach as a reliable default as it is relatively computationally inexpensive and does not require much hyperparameter tuning.

\begin{table*}[t!]
\centering
\begin{adjustbox}{max width=0.9\textwidth}
\small
\begin{tabular}{ccrrrr}
\toprule
 \multirow[c]{2}{*}{\textbf{GPT-2 Size}}&   
 \multirow[c]{2}{*}{\textbf{Decoding}} &  \multicolumn{2}{c}{\textbf{\mauve} ($\uparrow$)}  & 
 \multicolumn{2}{c}{\textbf{Empty bins}}
 \\
 \cmidrule(lr){3-4} \cmidrule(lr){5-6}
 & & No Smoothing  & K-T Smoothing & Total Number & Percentage \\
 \midrule
\multirow[c]{3}{*}{small} & Sampling & $0.589_{0.018}$ & $0.655_{0.018}$ & $54.2_{6.6}$ & $5.4_{0.7}$ \\
 & Greedy & $0.008$ & $0.019_{0.000}$ & $373.0$ & $37.3$ \\
 & Nucleus, $0.9$ & $0.878_{0.006}$ & $0.906_{0.005}$ & $36.4_{4.9}$ & $3.6_{0.5}$ \\
 \midrule
\multirow[c]{3}{*}{medium} & Sampling & $0.373_{0.010}$ & $0.446_{0.010}$ & $77.0_{5.5}$ & $7.7_{0.5}$ \\
 & Greedy & $0.012$ & $0.024$ & $314.0$ & $31.4$ \\
 & Nucleus, $0.9$ & $0.915_{0.006}$ & $0.936_{0.004}$ & $29.0_{6.6}$ & $2.9_{0.7}$ \\
 \midrule
\multirow[c]{3}{*}{large} & Sampling & $0.845_{0.010}$ & $0.878_{0.008}$ & $30.2_{1.3}$ & $3.0_{0.1}$ \\
 & Greedy & $0.012_{0.000}$ & $0.026_{0.000}$ & $311.4_{0.8}$ & $31.1_{0.1}$ \\
 & Nucleus, $0.95$ & $0.936_{0.003}$ & $0.952_{0.002}$ & $26.6_{3.0}$ & $2.7_{0.3}$ \\
 \midrule
\multirow[c]{3}{*}{xl} & Sampling & $0.882_{0.006}$ & $0.908_{0.005}$ & $27.6_{6.8}$ & $2.8_{0.7}$ \\
 & Greedy & $0.016$ & $0.033$ & $288.0$ & $28.8$ \\
 & Nucleus, $0.95$ & $\tabemph{} \mathbf{0.940}_{0.006}$ & $\tabemph{} \mathbf{0.955}_{0.004}$ & $23.4_{2.9}$ & $2.3_{0.3}$ \\

 \bottomrule
\end{tabular}
 \end{adjustbox}
\caption{\small
Comparison of \mauve with vector quantization without any smoothing (the default of \cite{pillutla2021mauve}) 
and with Krichevsky–Trofimov (K-T) smoothing (the default $\mauve^\star$ in this work). Their Spearman correlation is $\mathbf{1.00}$.
The last two columns show the number and fraction of empty bins obtained after vector quantization (without smoothing)
across both $P$ and $Q$ for the computation of $\mauve(P, Q)$.
The subscript of each column denotes the standard deviation over 5 random seeds. 
}
\label{tab:smoothing}
\end{table*}

\subsubsection{Effect of Smoothing on Vector Quantization-Based Estimation}
We now analyze the effect of smoothing on vector quantization-based estimation. 
\Cref{tab:smoothing} compares vector quantization (\Cref{alg:mauve:quant}) 
with and without the Krichevsky–Trofimov smoothing. Their Spearman rank correlations are $1.0$, meaning that they induce the same 
ranking. We note that their numerical values can be different, depending on the number of empty bins. 

Recall the computation pipeline of \Cref{alg:mauve:quant}: we jointly quantize the embedded samples
$\{\varphi(\xv_1, \ldots, \varphi(\xv_n)\}$ and $\{\varphi(\xv_1'), \ldots, \varphi(\xv_m')\}$
from $P$ and $Q$ respectively with $k$-means clustering. 
If some bin $l$ contained samples only from $P$, then the mass in that particular bin of $\hat Q_{\Scal, m}(l)$ would be missing, i.e., 
$\hat Q_{\Scal, m}(l) = 0$. 
\Cref{tab:smoothing} shows the number and fraction of empty bins. 
We observe around $2\%$ to $5\%$ empty bins for nucleus and ancestral sampling.
The number of empty bins increases with an increasing gap between the two distributions: 
greedy decoding has around $30\%$ of the bins empty while the best setting (nucleus sampling with the xl model) only has $2.3\%$ of the bins empty.
This motivates the use of smoothed distribution estimators even with data-dependent vector quantization. 

\subsubsection{Direct Estimation from Model Probabilities} \label{sec:direct-est}

In contrast to these previous estimation methods based on model embeddings, we compute \mauve directly using the model probabilities $Q(\cdot)$. Since the human probabilities $P(\cdot)$ are not available to us, we use the probabilities from GPT-2 xl (without reshaping the model probabilities) as a surrogate $P'$.
Then, 
using samples $\xv_1, \ldots, \xv_n \sim P$
and $\xv_1', \ldots, \xv_n' \sim Q$, we approximate
the coordinates of the KL divergence curve by the Monte Carlo estimates
\begin{align*}
\kl(P \Vert \lambda P + (1-\lambda Q))
\approx \frac{1}{n} \sum_{i=1}^n \log \frac{P'(\xv_i)}{\lambda P'(\xv_i) + (1-\lambda) Q(\xv_i)} \,,
\end{align*}
and similarly for $\kl(Q \Vert \lambda P + (1-\lambda Q))$. 

\myparagraph{Results}
The results are shown in \Cref{fig:estimation-direct}. We observe that this direct estimation can identify the right trend for model size for nucleus sampling, but fails to identify the right trend for ancestral sampling for medium $\prec$ small $\prec$ large (see \Cref{tab:mauve:expt:webtext-full}). 
Similarly, it fails to identify the right trends for the decoding algorithm, rating ancestral sampling as better than nucleus sampling. 

\subsubsection{Summary and Discussion}
The results of this section show that all the estimation procedures considered in \Cref{sec:compute} can produce useful estimates of the divergence frontier summaries at the right hyperparameter values, while the direct estimation procedure fails. 
{
These experiments suggest that the particular vector quantization is not a key factor behind the empirical success of \mauve and refute the argument of \citet{pimentel2022cluster} that the embedding-based vector quantization is the key ingredient leading to \mauve's strong correlation with human judgments (see \S\ref{sec:related-frontiers} for a detailed discussion). 
We note, however, that vector quantization has orthogonal benefits such as its simplicity and fast open-source implementation.
As we explore in the upcoming \S\ref{sec:expt:embedding}, a reliable vector embedding turns out to be the key component behind \mauve's strong correlation with human judgment.
}

\begin{figure}[t]
    \centering
    \includegraphics[width=0.99\linewidth]{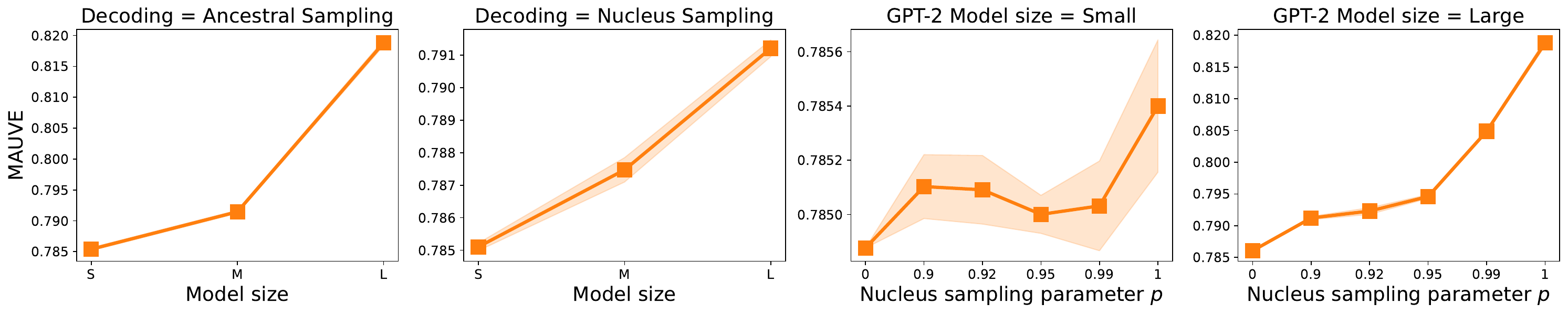}
    \caption{\small
    Direct estimation of \mauve from model probabilities $Q(\cdot)$, using the probabilities from GPT-2 xl as a surrogate for the human distribution $P(\cdot)$. The Spearman rank correlation of this direct estimation with $\mauve^\star$ (the default vector quantization with smoothing) is 
    $\mathbf{0.430}$ and its worst-case Spearman rank correlation (defined in \eqref{eq:worse-case-spearman})
    with human evaluation scores from \Cref{tab:mauve:expt:webtext-full} is $\mathbf{0.371}$.
    }
    \label{fig:estimation-direct}
\end{figure}

\subsection{Comparison to Other Divergences and Optimal Transport Costs}
\label{sec:expt:other-divergence}

Next, we compare our default choice of $\mauve_\kl^\star$ with different $f$-divergences and optimal transport-based distances. 

\subsubsection{Divergence Frontier Summaries and Other $f$-Divergences}
\label{sec:expt:other-summaries}

We compare $\mauve_\kl$ with other KL divergence frontier summaries, $\fint_\kl$, and $\midp_\kl$. We also evaluate the corresponding summaries of the $\chi^2$-divergence frontier and two other divergence metrics: the total variation distance $\tv(P, Q)$ and the squared Hellinger distance $H^2(P, Q)$. 
Since we approximate all the $f$-divergences in question using vector quantization and Krichevsky–Trofimov (add-$1/2$) smoothing, we refer to them 
using their starred names, e.g., $\mauve_\kl^\star$ and $\fint_\kl^\star$.

\begin{table}[t]
\centering
\begin{adjustbox}{max width=0.8\linewidth}
\begin{tabular}{lccccccc}
\toprule
\begin{tabular}{c}
\textbf{Correlation}
\end{tabular} 
&  $\mauve_\kl^\star$ &    $\fint_\kl^\star$ &   $\midp_\kl^\star$ &  $\mauve_{\chi^2}^\star$ &  $\midp_{\chi^2}^\star$ &        $\text{TV}^\star$ &  $H^2_\star$ \\
\midrule
$\mauve_\kl^\star$ & 
  $1.0$ & $0.99$ &    $1.0$ &          $1.0$ &  $1.0$ & $1.0$ &       $1.0$ \\
BT/Human-like & $0.857$ & \tabemph{} $\mathbf{0.929}$ & $0.857$ & $0.857$ & $0.857$ & $0.857$ & $0.857$ \\
BT/Interesting & $0.714$ & \tabemph{} $\mathbf{0.738}$ & $0.714$ & $0.714$ & $0.714$ & $0.714$ & $0.714$ \\
BT/Sensible & $0.762$ & \tabemph{} $\mathbf{0.833}$ & $0.762$ & $0.762$ & $0.762$ & $0.762$ & $0.762$ \\
\bottomrule
\end{tabular}
\end{adjustbox} \caption{\small
Comparison of various divergence frontier summaries and $f$-divergences with the default $\mauve_\kl^\star$ and human judgments on the web text dataset. 
We show their worst-case Spearman rank correlation within one standard deviation (defined in \Cref{eq:worse-case-spearman}).
}
\label{table:fdiv}
\end{table}

\myparagraph{Results}
The results are given in \Cref{table:fdiv}. 
We see that all divergence frontier summaries correlate perfectly with each other, with a near-perfect Spearman correlation coefficient of $0.99$ or higher. 
Notably, the correlation of $\fint_\kl$ with the Bradley-Terry human evaluation coefficients is larger than the other measures, which are all equal ($0.93$ versus $0.85$ for how human-like the text is). 
From a closer inspection of the actual values of the various divergences in \Cref{table:fdiv-full} of \Cref{sec:a:expt}, 
we see that $\fint_\kl$ ranks ancestral sampling for the xl model as better than nucleus sampling for the small model and agreeing with human evaluators for how human-like the text is. On the other hand, all other measures (including $\mauve_\kl$) are not able to distinguish between these two in the sense that they are within one standard deviation of each other.

\subsubsection{Variants based on Optimal Transport}

\begin{figure}[t]
    \centering
    \includegraphics[width=0.99\linewidth]{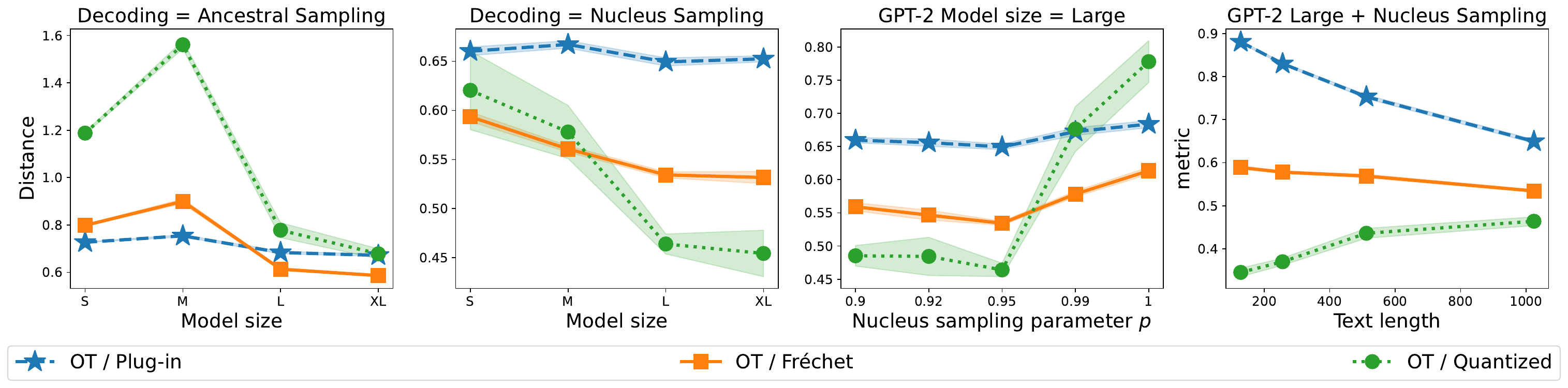}
    \caption{\small
    Optimal transport costs for GPT-2 generations in the web text domain. We rescale each measure by a constant so that all the numbers are $O(1)$. 
    Note that a lower transport cost denotes a smaller gap between the distributions. 
    Their correlations with $\mauve^\star$ (default) and human evaluations are given in \Cref{table:ot-corr}.
    }
    \label{fig:ot-v-mauve}
\end{figure}

\begin{figure}[t]
    \centering
    \includegraphics[width=0.99\linewidth]{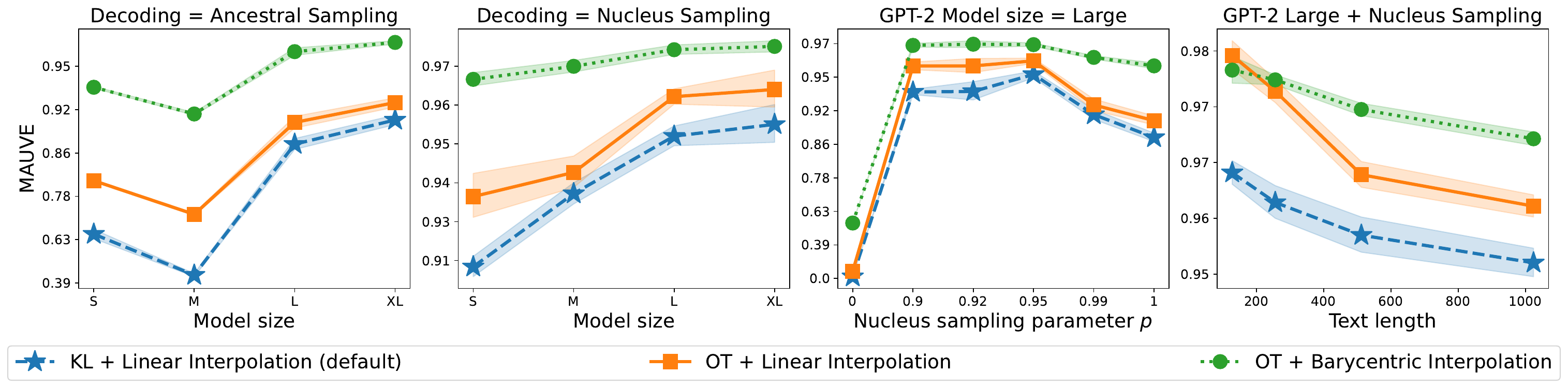}
    \caption{\small
    Comparison of variants of \mauve based on optimal transport costs for GPT-2 generations in the web text domain. Larger values denote a smaller gap for each variant. Their correlations with human evaluations are given in \Cref{table:ot-corr}.
    }
    \label{fig:bary-mauve}
\end{figure}

\begin{table}[t]

\centering
\begin{adjustbox}{max width=0.9 \linewidth}
\begin{tabular}{lcccccc}
\toprule

\multirow{3}{*}{\textbf{Correlation}}

& \multicolumn{3}{c}{\textbf{OT variants}}  & \multicolumn{3}{c}{\textbf{\mauve variants}} \\
\cmidrule(lr){2-4} \cmidrule(lr){5-7} 
&  Plug-in &    Fr\'echet &   Quantized &  
\begin{tabular}{c} OT + \\ Linear \\ interpolation \end{tabular} &  
\begin{tabular}{c} OT + \\ Barycenteric \\ interpolation \end{tabular} &    
\begin{tabular}{c} (Default) KL + \\ Linear \\ interpolation  \end{tabular} 
\\
\midrule
$\mauve_\kl^\star$ & 
  $0.954$ & $0.997$ & $0.980$  & $0.983$  &  $0.980$  & $1.000$   \\
BT/Human-like  &     
  $0.810$  &  $0.857$  &  $0.810$  & $0.810$  &  $0.857$ & $0.857$ \\
BT/Interesting &      
  $0.714$ &   $0.714$ &    $0.714$ &         $0.714$ &       $0.714$ &  $0.714$  \\
BT/Sensible    &      $0.738$ &   $0.762$ &    $0.738$ &         $0.738$ &       $0.762$ &  $0.762$ \\
\bottomrule
\end{tabular}
\end{adjustbox} \caption{ \small
Comparison of optimal transport baselines and variants of \mauve defined using optimal transport distances with the default $\mauve_\kl^\star$ and human evaluations on the web text dataset. 
We show their worst-case Spearman rank correlation within one standard deviation (defined in \eqref{eq:worse-case-spearman}) for the human evaluations. 
}
\label{table:ot-corr}
\end{table}

We investigate divergence frontier summaries based on optimal transport costs rather than $f$-divergences. Given two distributions $P, Q \in \Pcal(\Xcal)$ and a cost function 
$\rho : \Xcal \times \Xcal \to \reals_+$, the optimal transport cost between $P$ and $Q$ induced by $\rho$ is defined as 
\[
    \ot_\rho(P, Q) = \min\left\{
        \int_{\Xcal \times \Xcal} \rho(\xv, \xv') \, \D \pi(\xv, \xv') \,: \,
        \pi \in \Pcal(\Xcal \times \Xcal) \text{ has marginals } P, Q
    \right\} \,.
\]
In our context, following \Cref{sec:mauve:knn}, we use the cost function 
\[
    \rho(\xv, \xv') = \norm{\varphi(\xv) - \varphi(\xv')}_2^2 
\]
based on an embedding model $\varphi: \Xcal \to \reals^d$. This is also the 
squared Wasserstein-$2$ distance between the pushforward distributions 
$P' = \varphi_\sharp P$ and $Q' = \varphi_\sharp Q$. 
Similar to \Cref{sec:compute}, we simply use the plug-in estimate 
$\ot_\rho(\hat P_n, \hat Q_n)$ between the empirical distributions to estimate the
optimal transport cost -- we refer to it as the \textbf{plug-in optimal transport cost}. 

We consider quantized versions of this cost following 
the recipe of \Cref{sec:mauve:quant}. 
We quantize the empirical distributions $\hat P_n$ and $\hat Q_n$ into $k$-dimensional multinomial distributions $\hat P_{n, k}, \hat Q_{n, k} \in \Delta^{k-1}$.
We define a cost $\rho_k(i, j) = \norm{\cv_i - \cv_j}_2^2$, where $\cv_i$ 
is the cluster center obtained from $k$-means clustering of the embeddings. 
We refer to the resulting cost $\ot_{\rho_k}(\hat P_{n, k}, \hat Q_{n, k})$ as the \textbf{quantized optimal transport cost}.

The Fr\'echet distance~\cite{heusel2017gans} is a parametric approximation of $\ot_\rho$ which approximates the pushforwards $\varphi_\sharp P$ and $\varphi_\sharp Q$ by multi-variate Gaussians. Note that the approach of \Cref{sec:a:parametric} for \mauve follows this recipe. Unlike the methods of \Cref{sec:a:parametric}, the Fr\'echet distance has the advantage that it can be computed in closed form. 

We also explore variants of the divergence frontier (\Cref{def:fdiv-frontier}) based on the optimal transport cost. Define the \textbf{optimal transport frontier with linear interpolation} as 
\begin{align}
    \Fcal_{\ot, \rho}(P, Q) := 
    \left\{
        \big(\ot_\rho(P, R_\lambda), \ot_\rho(Q, R_\lambda)\big)\,:\,
        \lambda \in (0, 1)
    \right\} \,,
\end{align}
where $R_\lambda = \lambda + (1-\lambda) Q$. 
Inspired by the original characterization of the KL-divergence frontiers
as Pareto frontiers~\cite{djolonga2020precision}, we define a Pareto frontier of optimal transport 
costs.  Concretely, we define the \textbf{optimal transport frontier with barycentric interpolation} as 
\begin{align} \label{eq:bary-frontier-def}
\begin{aligned}
    \Fcal^{\text{bary}}_{\ot, \rho}(P, Q) :=
     \left\{
        \big(\ot_\rho(P, R_\lambda^\star), \ot_\rho(Q, R_\lambda^\star)\big)\,:\,
        \lambda \in (0, 1)
    \right\} \,, \\
    \text{where} \quad 
    R_\lambda^\star = \argmin_R \left\{ \lambda \, \ot_\rho(P, R) + (1-\lambda) \ot_\rho(Q, R)  \right\} 
\end{aligned}
\end{align}
is the barycenter of $P$ and $Q$ with weights $\lambda$ and $1-\lambda$. 
While the two formulations are equivalent for the KL divergence as we show in \Cref{prop:div-pareto-opt}, they are distinct in general for optimal transport costs. 
The definition in \eqref{eq:bary-frontier-def} is the analogue of \eqref{eq:div-pareto-opt} for the KL divergence frontier. 
We define the corresponding versions of \mauve, 
namely $\mauve_\ot$ and $\mauve_\ot^{\text{bary}}$
to be the area under the negative exponential of the frontiers, as in \eqref{eq:mauve:area-def}. 

\myparagraph{Computation and Hyperparameter Tuning}
Similar to \Cref{sec:mauve:quant}, we estimate the divergence frontiers 
$\Fcal_\ot(P, Q)$ and $\Fcal_\ot^{\text{bary}}(P, Q)$
on quantized versions $\Fcal_{\ot, \rho_k}(\hat P_{n, k}, \hat Q_{n, k})$
and $\Fcal_{\ot, \rho_k}^{\text{bary}}(\hat P_{n, k}, \hat Q_{n, k})$.
To compute them efficiently, a widely used approach is to add entropic regularization to the optimal transport problem \citep{cuturi13sinkhorn}.
Their behavior depends crucially on the regularization parameter being chosen.
A good default choice is the median of all the pairwise costs.

\myparagraph{Results} The results are shown in \Cref{fig:ot-v-mauve,fig:bary-mauve}, and 
\Cref{table:ot-corr}.

First, we note that 
the plug-in optimal transport cost fails to capture the correct dependence for the model size as it rates the medium-sized model as worse than GPT-2 small under nucleus sampling ($1499 \pm 5$ vs. $1473 \pm 4$, cf. \Cref{table:ot-full} in \Cref{sec:a:expt}).
The plug-in estimator also fails to capture the dependence on the text length. Similar to the Fr\'echet distance in \Cref{sec:expt:properties}, its numbers suggest that longer model generations drift closer to the human distribution rather than farther away. 

This issue of optimal transport costs can be fixed by vector quantization. 
Indeed, both the quantized optimal transport costs and their frontier summary variants capture the correct dependence in terms of text length, while simultaneously capturing the right trends for the model size and decoding algorithm. This suggests that vector quantization may have a regularizing effect on the estimation problem --- we leave a deeper exploration of this phenomenon for future work. 

\myparagraph{Correlation Analysis}
We see from \Cref{table:ot-corr} that the plug-in optimal transport cost
has a smaller worst-case Spearman correlation of $0.810$ with human evaluations. 
This is smaller than $\mauve_\kl$, Fr\'echet, and $\mauve_\ot^{\text{bary}}$ ($0.857$) 
and is on par with Gen. PPL. 
Comparing the full numbers in \Cref{table:ot-full} in \Cref{sec:a:expt} allows us to find the reason for this discrepancy.
The quantized OT cost rates GPT-2 small and medium models with nucleus sampling (resp. $0.083$ and $0.077$) as better than the large and xl models with ancestral sampling (resp. $0.090$ and $0.104$; these gaps are larger than the standard deviation of $0.005$ across runs). These trends disagree with human evaluations. $\mauve_\ot$ and $\mauve_\ot^{\text{bary}}$ make the same mistake while $\mauve_\kl^\star$ identifies the small model with nucleus 
sampling as being worse than the large and xl models with ancestral sampling.

\myparagraph{Summary and Discussion}
Naïve use of optimal transport costs such as the Fr\'echet distance (parametric Gaussian approximation) or the empirical estimator in the embedding space leads to a failure to capture the right trend with respect to the generation length. This issue is specific to the text setting due to the lack of a temporal dimension for images; indeed, the Fr\'echet distance is the de facto standard evaluation metric for image generation. 
Optimal transportation in the quantized embedding space (similar to \Cref{sec:mauve:quant}), 
as well as frontier summaries that build upon them can overcome this issue.

\subsection{Effect of the Embedding}
\label{sec:expt:embedding}

The results of \Cref{sec:expt:estimation,sec:expt:other-summaries} suggest that the embedding is a key factor in the empirical usefulness of \mauve and other divergence frontier summaries. In this section, we analyze the effect of the embeddings, experimenting with using the generative model itself, using masked language models, shallow embeddings, and finally string-based embeddings that are not learned from data.

\begin{figure}[t]
    \centering
    \includegraphics[width=0.99\linewidth]{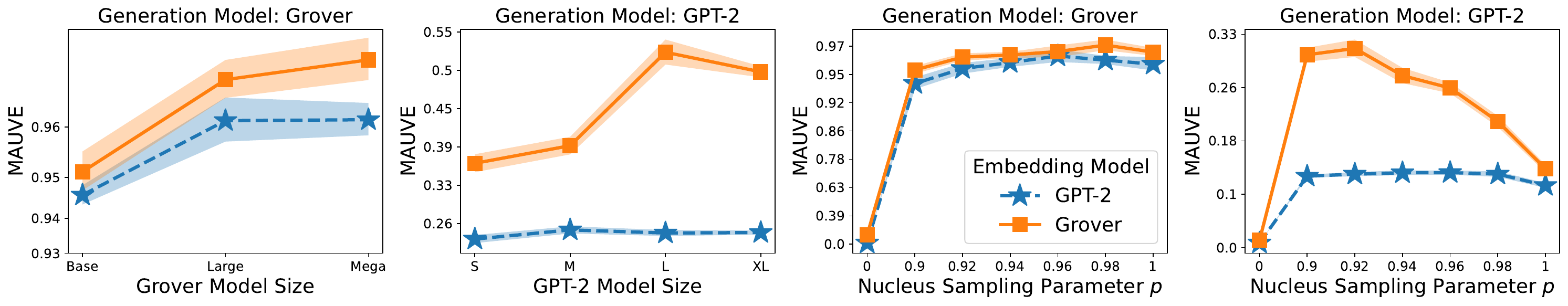}
    \caption{\small
    Effect of embeddings on the news generations. We compare generative models GPT-2 and Grover using embeddings from both GPT-2 and Grover. The Spearman rank correlation between $\mauve^\star_{\text{Grover}}(P, \cdot)$ and $\mauve^\star_{\text{GPT-2}}(P, \cdot)$ is $\textbf{0.971}$.
    }
    \label{fig:embedding:news}
\end{figure}

\begin{figure}[t]
    \centering
    \includegraphics[width=0.99\linewidth]{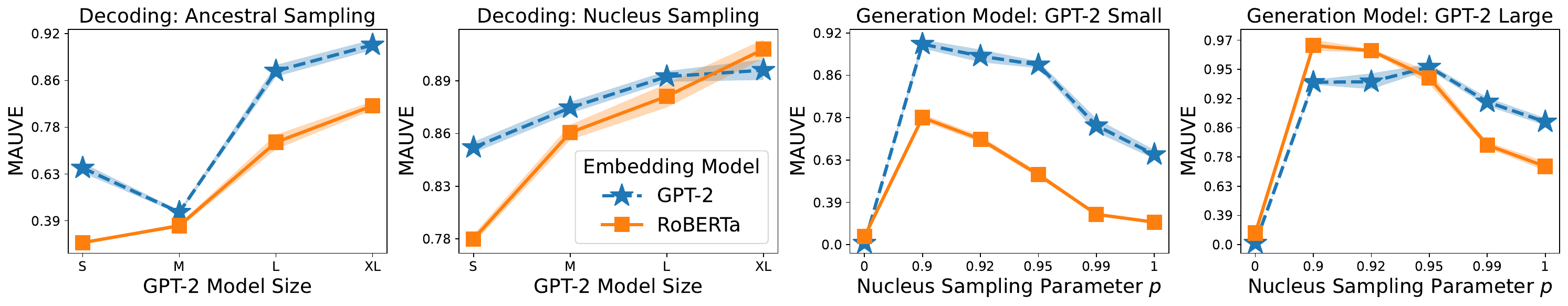}
    \caption{\small
    Effect of embeddings on web text generations with GPT-2. $\mauve^\star$ computed from GPT-2 embeddings and RoBERTa embeddings have a Spearman rank correlation of $\textbf{0.962}$.
    }
    \label{fig:embedding:roberta}
\end{figure}

\subsubsection{Reusing a Generative Model For Embeddings}
First, we study whether using the embeddings from the same generative model we are evaluating might bias the proposed measures toward generations from that model. In particular, consider two generative models $Q_1$ and $Q_2$, and let $\mauve_i(P, \cdot)$ denote the value of \mauve obtained from using embeddings from model $Q_i$ for $i \in \{1, 2\}$. We check whether $\mauve_1(P, Q_1) > \mauve_1(P, Q_2)$ but $\mauve_2(P, Q_1) < \mauve_2(P, Q_2)$. 

We perform a comparison in the news domain, where $P$ denotes the distribution of articles in the RealNews dataset. We take $Q_1$ to the Grover model and $Q_2$ to be GPT-2, both of various sizes and decoding algorithms.\footnote{
    Although the training data of GPT-2 is proprietary, its open version OpenWebText~\cite{Gokaslan2019OpenWeb} contains a significant number of news articles~\cite{sharoff2020know}.
    The most frequently occurring web domains in OpenWebText are news domains~\cite[Figure 5]{gehman2020real}. 
}
We use Grover large and GPT-2 large to compute the embeddings.

The results are given in \Cref{fig:embedding:news}. We observe that the embeddings from both GPT-2 and Grover agree that generations from Grover are closer to the RealNews distribution than GPT-2. This trend holds uniformly across model sizes and decoding algorithms. Indeed, the Spearman rank correlation between $\mauve_{\text{GPT-2}}$ and $\mauve_{\text{Grover}}$ is $\textbf{0.971}$. 
Still, there are some minor differences in the trends revealed by each of the features. For instance, Grover embeddings suggest that news generations from GPT-2 large are better than those from GPT-2 xl. Similarly, Grover embeddings suggest $p=0.92$ as the best nucleus sampling hyperparameter for GPT-2 generations, while features from GPT-2 think $0.9 \le p \le 1$ are roughly equivalent. 

Overall, we find that the \mauve scores obtained from both generative models are strongly correlated, and we do not find any evidence of bias from reusing a generative model for embeddings.

\subsubsection{Masked Language Model Embeddings}
So far, we only considered embeddings from left-to-right language models such as GPT-2 and Grover. In this next experiment, we consider using embeddings from a masked language model, RoBERTa large~\cite{liu2019roberta}. We repeat the experiments in the web text domain with GPT-2 as the generative model and RoBERTa as the embedding model.

The results are given in \Cref{fig:embedding:roberta}. 
First, we note that the correlation between \mauve computed from GPT-2 embeddings and RoBERTa embeddings has a Spearman rank correlation of $\mathbf{0.962}$. Second, we observe that RoBERTa embeddings also capture the trends concerning model size and decoding, with some minor differences. For instance, both models identify the greedy $\prec$ ancestral $\prec$ nucleus trend from \Cref{sec:expt:properties}. While both embedding models agree that $p=0.9$ is the best nucleus sampling hyperparameter for the small model, they disagree on generations from the large model. Other baselines such as Gen. PPL. that do not use embeddings suggest that $p=0.95$ is the best hyperparameter, agreeing with embeddings from GPT-2. We also note that RoBERTa features do not capture the medium $\prec$ small $\prec$ large $\prec$ xl trend for model sizes under ancestral sampling (cf.~\Cref{tab:mauve:expt:webtext-full}). 

In summary, the proposed measures computed with masked language models correlate strongly with those computed from left-to-right language models. They can quantify trends concerning model size and decoding.

\subsubsection{Learned Shallow GloVe Embeddings}

Next, we examine \mauve equipped with learned embeddings predating the advent of transformer language models. We repeat the web text experiments with GPT-2 generations where \mauve is computed based on the GloVe word embeddings~\cite{pennington2014glove}.

The GloVe embeddings differ from the deep embeddings of the preceding sections in two ways.
First, they are {non-contextual}, meaning that a word (e.g. ``bank'') has the same embedding regardless of the context (e.g. ``river \emph{bank}'' or the ``\emph{Bank} of America''). Second, they are embeddings of whitespace-separated words, as opposed to BPE tokens that are used in transformer language models. Overall, we represent a sequence $\xv = (w_1, \ldots, w_T)$ of words\footnote{
    We use $w_i$ instead of $x_i$ to emphasize that these are words rather than BPE tokens as in the rest of the paper.
} 
 using the average GloVe embedding of words in the vocabulary $V_\text{glove}$:
\[
\varphi_\text{glove}(\xv) = \frac{1}{T} \sum_{i=1}^T \text{GloVe}(w_i) \cdot \mathbb{I}(w_i \in V_\text{glove}) \,.
\]

\begin{figure}[t]
    \centering
    \includegraphics[width=0.99\linewidth]{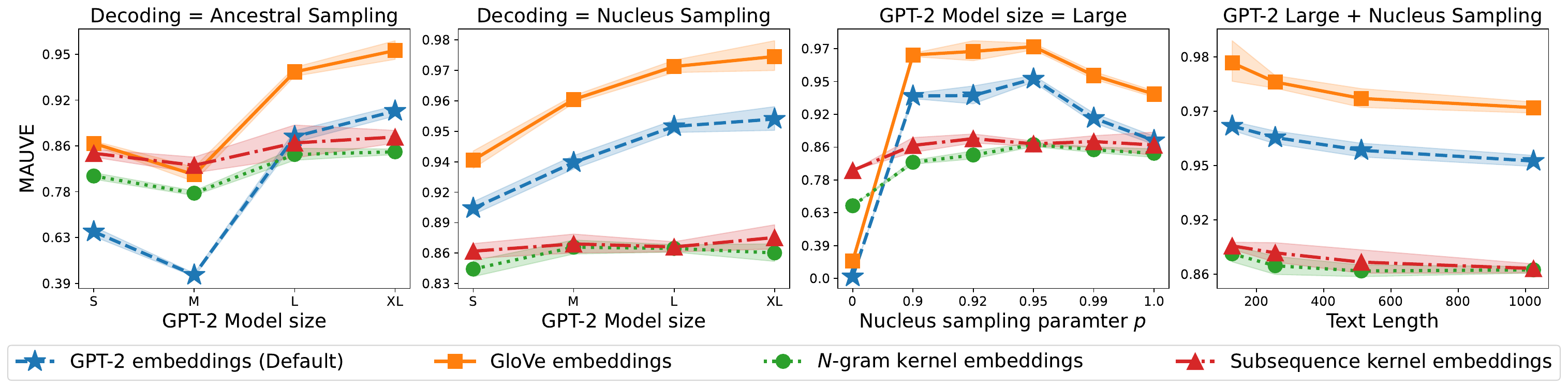}
    \caption{\small
    \mauve from shallow and string-based embeddings on web text generations with GPT-2. 
    }
    \label{fig:embedding:non-deep}
\end{figure}

\begin{table}[t]
\centering
\begin{adjustbox}{max width=0.9\linewidth}
\renewcommand{\arraystretch}{1.2}
\small
    \begin{tabular}{ccccc}
    \toprule
        \textbf{Correlation} & 
        \textbf{GPT-2 Embedding} & 
        \textbf{GloVe Embedding} &
        \textbf{$N$-gram Kernel} &
        \textbf{Subsequence Kernel} \\
        \midrule
    $\mauve_\kl^\star$ (default) &
    $1.00$ & $0.993$ & $0.727$ & $0.783$
    \\
    BT/Human-like & 
    $0.857$ &	$0.928$ & $0.500$ & $0.286$
    \\
    BT/Interesting &
    $0.714$ & $0.738$ & $0.262$ & $0.214$
    \\
    BT/Sensible &
    $0.762$ & $0.833$ & $0.429$ & $0.214$
    \\
    \bottomrule
    \end{tabular}
\end{adjustbox} \caption{\small
Correlation of $\mauve_\kl^\star$ computed from shallow and string-based embeddings with the default GPT-2 embeddings and with human evaluations.
For the latter, we show their worst-case Spearman rank correlation within one standard deviation (defined in \Cref{eq:worse-case-spearman}).
}
\label{table:non-deep}
\end{table}

\begin{figure}[t]
    \centering
    \includegraphics[width=0.75\linewidth]{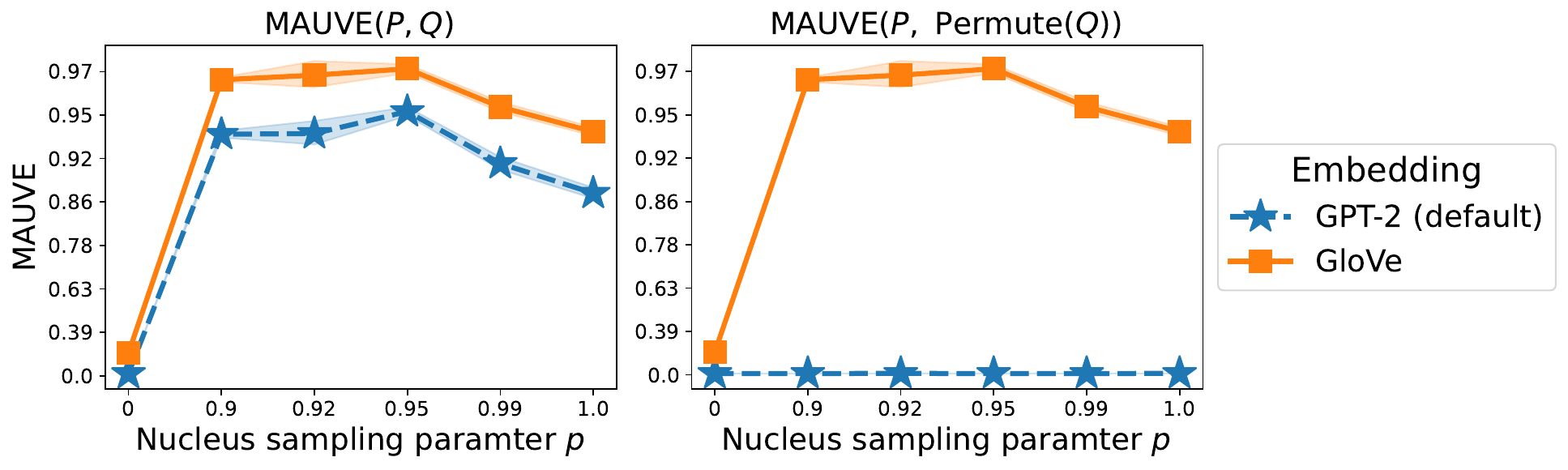}
    \caption{\small
    Robustness to permutations at the word level: \mauve with GPT-2 embeddings is sensitive to the order of words whereas GloVe embeddings are not.
    We define $\text{Permute}(Q)$ as the distribution over word sequences $(w_{\pi(1)}, \ldots, w_{\pi(n)})$ where $(w_1, \ldots, w_n) \sim Q$ and $\pi$ is a uniformly random permutation on $[n]$.
    }
    \label{fig:embedding:non-deep-permute}
\end{figure}

\myparagraph{Results}
We note that the GloVe embeddings identify the key trends concerning model size, decoding, and text length in \Cref{fig:embedding:non-deep}. Indeed, its worst-case Spearman correlation with the human evaluation in \Cref{table:non-deep} is even (marginally) better than that of the GPT-2 embeddings ($0.93$ vs. $0.86$). However, the GloVe embeddings have a significant drawback: they come from a bag-of-words model where word order is irrelevant. As shown in \Cref{fig:embedding:non-deep-permute},  GPT-2 embeddings do not suffer from this drawback.
Overall, these results show that \mauve can extract useful information from shallow GloVe embeddings, demonstrating the versatility of \mauve.

\subsubsection{String-based Kernel Embeddings}
\label{sec:string-kernel}

Next, we compute \mauve directly from strings, without any learned embeddings, shallow or deep. Concretely, we consider the embeddings implied by a positive definite kernel $\kappa(\xv, \xv')$ between text sequences $\xv, \xv'$.

Recall that a kernel $\kappa: \Xcal \times \Xcal \to \reals_+$ over a space $\Xcal$ is said to be positive definite if the Gram matrix $\bm{K} \in \reals^{r \times r}$ with entries $[\bm{K}]_{i,j} = \kappa(\xv_i, \xv_j)$ defined by any collection $\xv_1, \ldots, \xv_r \in \Xcal$ of $r$ inputs is a symmetric and positive definite matrix for all integers $r$; we refer to the textbook \cite{shawe2004kernel} for background. A key property of positive definite kernels is that they can be viewed as dot products in an abstract feature space.
Specifically, Mercer's theorem states that there is a unique feature map $\varphi_\kappa: \Xcal \to \Hcal$ onto a Hilbert space $\Hcal$ equipped with an inner product $\inp{\cdot}{\cdot}_\Hcal$ such that $\kappa(\xv, \xv') = \inp{\varphi_\kappa(\xv)}{\varphi_\kappa(\xv')}_\Hcal$ for all $\xv, \xv' \in \Xcal$~\cite{mercer1909xvi}.

We compute \mauve using these embeddings $\varphi_\kappa$ induced by two string kernels, where $\Xcal$ is the space of text sequences (i.e., strings):
\begin{enumerate}[nosep, label=(\alph*)]
    \item \textbf{$N$-gram kernel}: Given an integer $N$, the $N$-gram kernel is defined as the ratio of common $N$-grams of its inputs to the total number unique of $N$-grams~\citep[Sec.  11.2]{shawe2004kernel}. Specifically, letting $A_N(\xv)$ denote the set of all $N$-grams in the sequence $\xv$, the $N$-gram kernel $\kappa_N$ is defined as 
    \[
        \kappa_N(\xv, \xv') = \frac{|A_N(\xv) \cap A_N(\xv')|}{|A_N(\xv) \cup A_N(\xv')|} \,.
    \]
    This is also the Jaccard similarity between the set of $N$-grams of $\xv$ and those of $\xv'$.
    \item \textbf{Subsequence kernel}: The subsequence kernel~\cite{lodhi2002text} is based on the number of common (non-contiguous) subsequences of length $N$ and scaled by the gap using a decay factor $\lambda \in (0, 1)$, known also as the gap penalty. Concretely, the feature map $\varphi_{N, \lambda}$ used to define the subsequence kernel $\kappa_{N, \lambda}$ has one component for every possible length-$N$ sequence $\zv \in V^N$. The corresponding component is zero if $\zv$ is not a subsequence of $\xv$, else it is
    \[
        \varphi_{N, \lambda}(\xv)[\zv] = \sum_{\bm{i} \, :\, \zv = \xv[\bm{i}]} \lambda^{\text{len}(\bm{i})} \,,
    \]
    where $\bm{i}$ is an index sequence and $\xv[\bm{i}]$ is the subsequence of $\xv$ obtained by selecting the indices from $\bm{i}$, 
    and $\text{len}(\bm{i}) = i_{|\bm{i}|} - i_1 + 1$ is the length of the subsequence in $\xv$.

    A na\"ive implementation of $\kappa_{N, \lambda}(\xv, \xv')$ has a complexity of $O(|V|^N)$ but it can be implemented using sparse dynamic programming in $O(N M \log |\xv|)$ time, where $M = |\{(i, j) \,:\, x_i = x'_j\}|$ is the total number of matches between $\xv$ and $\xv'$~\cite{rousu2005efficient}.
\end{enumerate}
\noindent We compute \mauve from the respective embeddings of these two kernels at the level of word-piece tokens using the nearest neighbor method of \S\ref{sec:mauve:knn}. To keep the \mauve computation time to under two hours, we use $n=800$ samples for the $N$-gram kernel and $n=200$  samples for the subsequence kernel. We sweep over the hyperparameters $N \in \{3, 4, 5\}$ and $\lambda \in \{0.1, 0.2, \ldots, 0.9\}$ of the kernels and report the hyperparameters that have the highest correlation with the human evaluation: these are $N=3$ for the $N$-gram kernel and $N=5, \lambda=0.5$ for the subsequence kernel.

\myparagraph{Results}
\Cref{fig:embedding:non-deep} shows the dependence of \mauve on the trends concerning model size, decoding, and text length. We see that string kernel embeddings only identify these trends weakly and unreliably, i.e., the mean across $5$ runs trend is as expected but the gaps are often smaller than the standard deviation across runs. This is true of all three trends but take the text length as an example. \mauve from the subsequence kernel at a length of $512$ tokens is $0.879 \pm 0.013$, which is smaller than $0.889 \pm 0.010$ at length $256$ and larger than $0.871 \pm 0.005$ at $1024$ tokens, but all three numbers are within one standard deviation of each other.
Similarly, we see from \Cref{table:non-deep} that the worst-case Spearman correlations with the human evaluation results are small, always under $0.5$.
This shows that the raw strings are not informative enough for \mauve.

\subsubsection{Summary and Discussion}
The results of this subsection demonstrate the importance of the embedding to the usefulness of \mauve. The poor performance of $N$-gram and subsequence kernels, and direct model probabilities (\Cref{sec:direct-est}) show that some care must be taken to use informative embeddings. Yet, \mauve is versatile enough to leverage information from a wide variety of embeddings, including language model embeddings (left-to-right LMs, even if it has been used for generation, or masked LMs), and shallow non-contextual embeddings.

\subsection{Comparison to Generative Precision and Recall}
\label{sec:expt-kynka-pr}

Metrics based on divergence frontiers have been previously used extensively in the computer vision community~\cite{sajjadi2018assessing,kynknniemi2019improved,djolonga2020precision}. How do these metrics fare in the evaluation of text generative models?
We now examine the applicability of the most widely used such metrics, i.e., \citeauthor{kynknniemi2019improved}'s precision and recall for generative models, in the web text domain.

\myparagraph{Definitions}
These notions of precision and recall rely on whether a point $\xv$ lies within the manifold of a set of samples $Y$. Concretely, letting $\dist_k(\zv, Y)$ denoted the distance of $\zv$ to its $k$\textsuperscript{th} neighbor in $Y$, define
\[
    s_k(\xv, Y) := 
    \begin{cases}
        1, & \text{ if } \exists \zv \in Y \, :\, \rho(\xv, \zv) \le \dist_k(\zv, Y), \\
        0, & \text{ else.}
    \end{cases}
\]
Using this notion, the generative precision and recall (evaluated with $k$ nearest neighbors) of a generative distribution $Q$ relative to a target distribution $P$ based on $n$ samples $X_Q \sim Q^n$ and $X_P \sim P^n$ are defined as
\begin{align*}
    \mathrm{Precision}_k(X_P, X_Q) = \frac{1}{n} \sum_{\xv' \in X_Q} s_k(\xv', X_P), \quad \text{and} \quad
    \mathrm{Recall}_k(X_P, X_Q) = \frac{1}{n} \sum_{\xv \in X_P} s_k(\xv, X_Q) \,.
\end{align*}
Intuitively, the precision is high if the generated data looks more human-like (i.e., plausibly drawn from $P$) and the recall is high if the generative model captures the diversity of the target distribution $P$. Higher values of both precision and recall are desirable. We find that all $1 \le k \le 25$ produced the same qualitative trends, so we show the results for $k=5$.

\begin{figure}[t]
    \centering
    \includegraphics[width=0.99\linewidth]{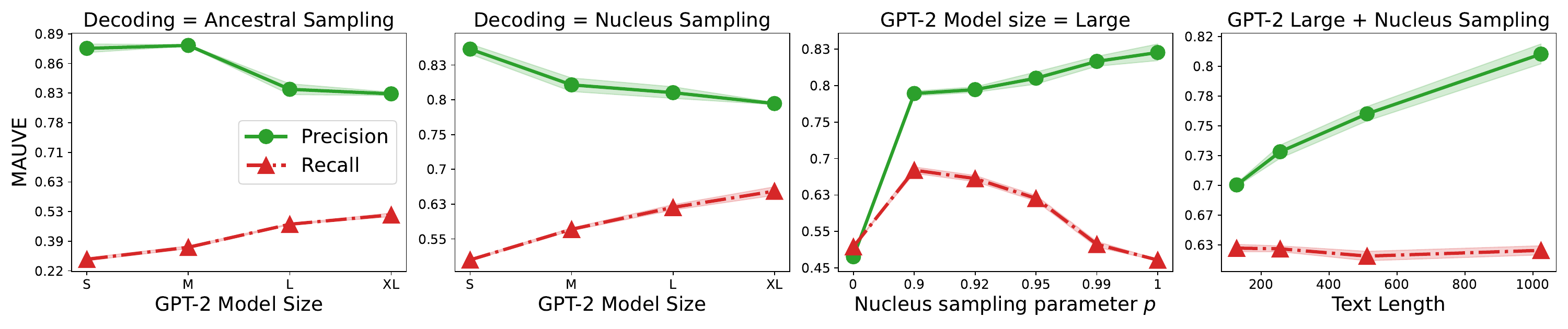}
    \caption{\small Empirical behavior of generative precision and recall (originally proposed by \citet{kynknniemi2019improved} for evaluating GANs in computer vision) for natural language generation in the web text domain. Higher values denote better performance.
    }
    \label{fig:kynka-pr-nlp}
\end{figure}

\myparagraph{Results}
\Cref{fig:kynka-pr-nlp} shows the trends with respect to model scale, decoding algorithm, and text length. Given that larger language models are generally better text generators, we expect the precision and recall to both increase with the model scale. We see for both ancestral and nucleus sampling that the recall increases as expected. However, the precision decreases with increasing model scale; this suggests that smaller models produce more human-like text, which is qualitatively untrue.

Next, we consider the effect of decoding in terms of the nucleus sampling parameter $p$. Prior work suggests that $p \in [0.9, 0.95]$ should give the most human-like text while $p=1$ gives the most diverse text~\cite{holtzman2019curious}. Thus, we would expect the precision to peak in $p \in [0.9, 0.95]$, while we expect the recall to increase with $p$ monotonically. We see that the actual trends are the exact opposite of what we would expect, i.e. $p=1$ produces the most human-like text whereas $p=0.9$ best matches the diversity of human text, both of which are qualitatively untrue.

Finally, since model text generations degrade as they get longer, we expect both precision and recall to get worse with text length. Again, the precision metric says that the generated text gets more human-like as its length increases, which is untrue.

\myparagraph{Summary and Discussion}
In summary, these results demonstrate that the notion of generative precision and recall proposed by \citeauthor{kynknniemi2019improved} do not behave as expected for natural language generation. In contrast, \mauve identifies the expected behavior with respect to model size, decoding algorithm, and text length.

\subsection{Evaluating Image Generative Models with \mauve}
\label{sec:expt:vision}

In this section, we explore the applicability of our approach to measure the gap between a distribution $Q$ of images generated by a neural net and its target distribution $P$ of real-world images.

\begin{figure}[t]
    \centering
    \includegraphics[width=0.99\linewidth]{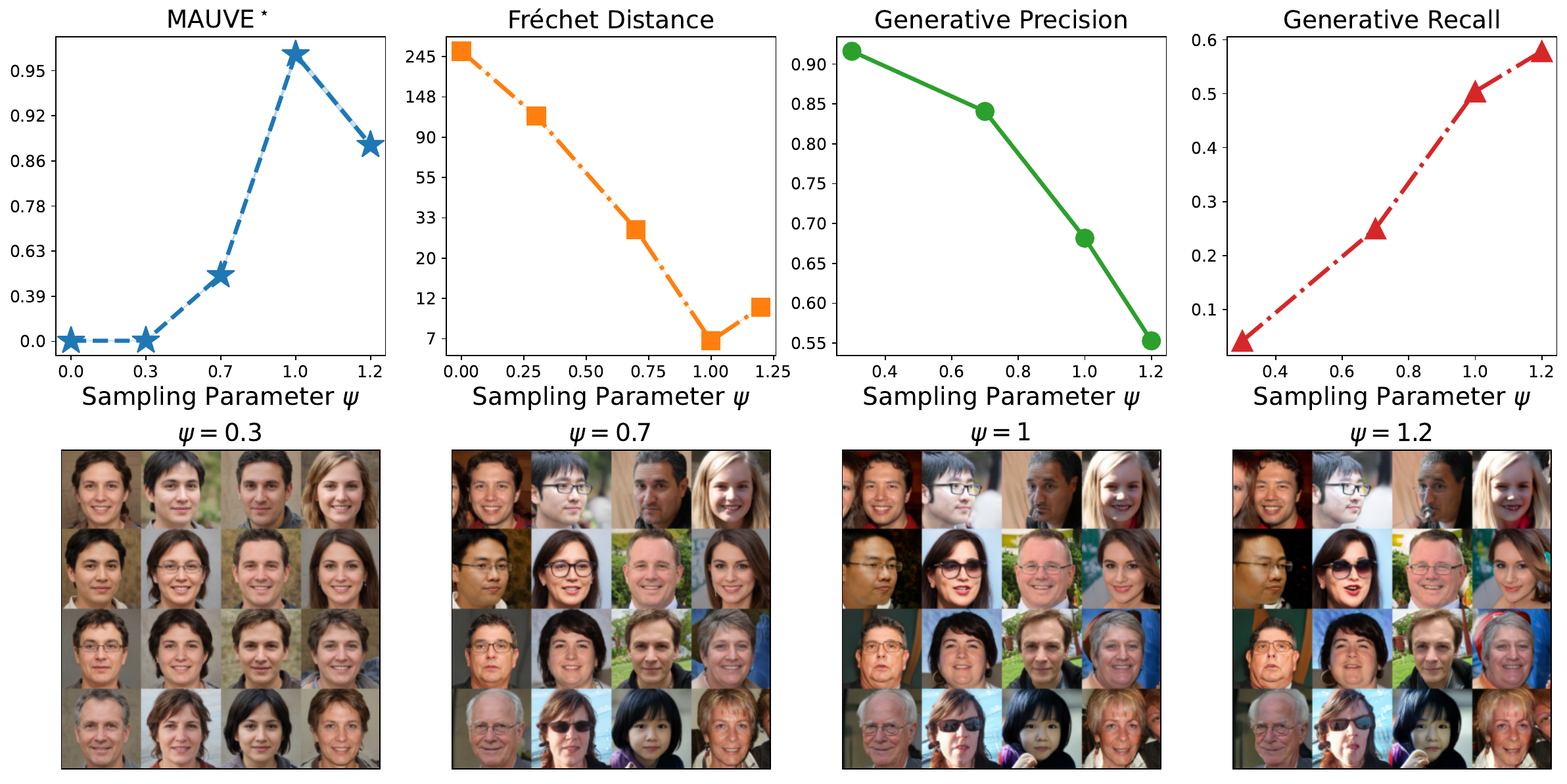}
    \caption{\small
     Evaluating image generative models across sampling algorithms with \mauve, Fre\'chet distance~\cite{heusel2017gans} and generative precision-recall~\cite{kynknniemi2019improved} (top row). Some sample images generated at various sampling parameters are shown in the bottom row. We use the StyleGAN2-ADA model~\cite{karras2020training} with various values of the $\psi$-sampling parameter $\psi$ as the model distribution $Q$ and compare it with the reference distribution $P$ over the FFHQ dataset. 
    The generations shown at each threshold $\psi$ are generated from the same initial randomness for a given position in the grid. 
    We recommend zooming in for a closer inspection of the generated images. 
    }
    \label{fig:vision}
\end{figure}

\myparagraph{Setup} 
We study the distribution of images generated by models trained on the Flickr-Faces-HQ Dataset (FFHQ) \citep{karras2019style}. The models we consider are based on the StyleGAN2-ADA generative adversarial networks described by \citet{karras2020training}.

As a representative divergence frontier summary, we consider $\mauve_\kl^\star$ 
computed using quantization with $k=1000$ clusters. 
We use $50,000$ samples from the model in comparison to $50,000$ samples from the FFHQ training data, unless specified otherwise. The resolution of each image is $1024\times1024$. 
Note that in the language modality, we compute \mauve using samples from the \emph{test set} whereas here--following standard practice in vision when comparing distributions using Inception Score or Fr\'echet  distance--we use samples from the \emph{train set}. Similar to these baselines, we use as an embedding model the standard features of an Inception network pre-trained on Imagenet. 
This setting for Fr\'echet  distance corresponds exactly to the FID-50k metric commonly used in the vision literature.
We also compare to the generative precision-recall~\cite{kynknniemi2019improved}; cf. \S\ref{sec:expt-kynka-pr} for definitions.

\subsubsection{Effect of the Sampling}

We consider samples drawn from the GAN model using $\psi$-\emph{sampling}, a technique that biases sampling towards modes of the model distribution.\footnote{$\psi$-sampling is referred to as \emph{truncation} by \citet{karras2020training}.} 

We briefly describe $\psi$-sampling.
The generator function of these models maps a simple random latent variable $\zv \sim \mathcal{N}(0, {I}_\mathcal{Z})$ to an image $\xv = g(\zv) \in \mathcal{X}$ drawn from the pushforward distribution defined by a learned generator function $g: \mathcal{Z} \to \mathcal{X}$. The generator itself is decomposed into $g = s \circ h$ consisting of an embedding mapping function $h: \mathcal{Z} \to \mathcal{W}$ and synthesis network $s: \mathcal{W} \to \mathcal{X}$. Let $\wv^* = \mathop{\mathbb{E}}_{\zv \sim \mathcal{N}(0,I_\mathcal{Z})}[h(\zv)]$ be the average embedding of noise. Given $\zv \sim \mathcal{N}(0,\text{I}_\mathcal{Z})$, we define $\psi$-\emph{sampling} using a modified generator function defined by
\[
    g_\psi(\zv) = s(\wv^* + \psi(h(\zv) - \wv^*)) \,.
\]
If $\psi < 1$, this transformation linearly contracts the mapped value $h(\zv) \in \mathcal{W}$ towards the mean mapping $\wv^*$. Intuitively, this will result in higher probability, but less diverse, output images.
In contrast, $\psi > 1$ will emphasize the lower probability regions of the image space, resulting in more diverse images of lower quality.

\myparagraph{Results}
The results are given in \Cref{fig:vision}. 
Both \mauve and Fr\'echet distance identify the same ordering of $\psi$: $1 \succ 1.2 \succ 0.7 \succ 0.3 \succ 0$. 
Qualitatively, we observe the expected quality-diversity tradeoff as we vary $\psi$. The extreme $\psi=0.3$ produces high-quality images of faces that look very similar to each other. At $\psi=0.7$, we observe more diversity in the generated faces over attributes such as hair color and style, eyewear, and other factors. We get a greater diversity at $\psi = 1$ with more diversity in hair and eyewear but also in the direction the generated face points towards and facial expressions. 
At $\psi = 1.2$, we that the generated faces start to appear distorted. Some images also feature parts of a second face.
The notions of generative precision and recall capture both quality and diversity trends, as was demonstrated in previous work. Thus, similar to Fr\'echet distance, \mauve accounts for both quality and diversity to produce a single measure of the gap between the model distribution and the target distribution. Unlike both precision-recall and the Fr\'echet distance, however, \mauve also perfectly identifies various trends in the natural language modality.

\subsubsection{Effect of the Model Scale and Architecture}

We compare image generative models across different model architectures, analogous to the effect of the model scale in \S\ref{sec:expt:properties}. For these experiments, we use $5000$ samples to compute each metric (compared to the $50000$ samples used for the experiments in the previous experiment).

\begin{figure}[t]
    \centering
    \includegraphics[width=0.99\linewidth]{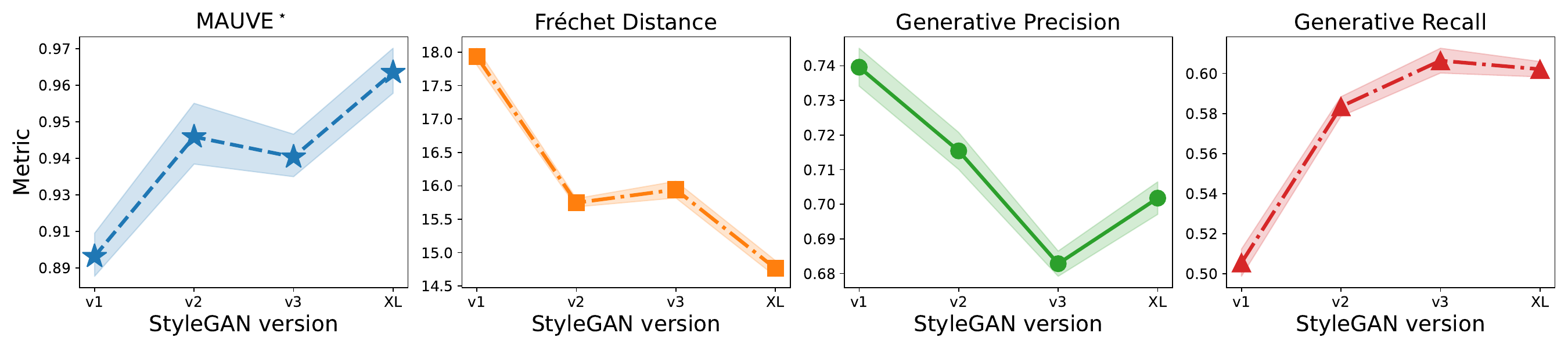}
    \caption{\small
    Comparing StyleGAN model generations with \mauve, Fre\'chet distance~\cite{heusel2017gans} and generative precision-recall. We compare: 
    (v1) the original StyleGAN~\cite{karras2019style},
    (v2) StyleGAN2-ADA~\cite{karras2020training},
    (v3) StyleGAN3~\cite{karras2021alias}, and
    (XL) StyleGAN-XL~\cite{sauer2022styleganxl}.
    These plots use $5000$ samples for each metric, and the shaded region denotes the standard deviation across 10 runs with different subsamples of the target distribution.
    }
    \label{fig:vision-scaling}
\end{figure}

\myparagraph{Effect of Model Scale and Architectural Improvements}
We compare various generations of StyleGAN models. Each model in this family builds upon the previous one with innovations in the architecture and training pipeline to address certain artifacts in the generated images. We consider the following models:
\begin{enumerate}[label=(\alph*), nosep]
    \item \textbf{StyleGAN}~\cite{karras2019style}: the first model in this family with $26$ million parameters.
    \item \textbf{StyleGAN2-ADA}~\cite{karras2020analyzing, karras2020training}: the second model in this family, with $30$ million parameters.
    \item \textbf{StyleGAN3}~\cite{karras2021alias}: this model makes substantial changes to the architecture of StyleGAN2. StyleGAN3 has only $15$ million parameters but produces image generations of similar quality as StyleGAN2-ADA as per standard metrics like the Fr\'echet distance. The authors claim that it is better suited to video generation.
    \item \textbf{StyleGAN-XL}~\cite{sauer2022styleganxl}: the largest model in this family that we consider with $71$ million parameters. 
\end{enumerate}
All images were produced using $\psi$-sampling with $\psi=1$.

The results are shown in \Cref{fig:vision-scaling}. We find that both $\mauve^\star$ and Fre\'chet distance find the same trends: more recent models are better with StyleGAN2-ADA and StyleGAN3 being rated almost the same (i.e., with one standard deviation of each other). Notably, the most recent and the largest model --- StyleGAN-XL --- produces the best images as per these metrics.

On the other hand, generative precision~\cite{kynknniemi2019improved} rates the oldest StyleGAN model as producing the most photorealistic images (highest precision). This fails to pass the visual inspection test, as the subsequent works in the StyleGAN family discuss the flaws of this model's generations and are designed to improve them. This is similar to the text domain where generative precision finds that the smallest GPT-2 model produces the most human-like text. Thus, designing fine-grained fidelity and diversity metrics for generative models that can be used reliably across model scales and families remains an important open problem.

\myparagraph{Comparing GANs with Diffusion Models}
We compare StyleGAN2-ADA with a diffusion model NCSN++~\cite{song2021score} on the FFHQ domain. NSCN++ is the first diffusion model to directly generate high-resolution images of $1024\times 1024$ pixels (without up-sampling lower-resolution images in a multi-step pipeline). \citet{stein2023exposing} show that diffusion models perform significantly worse than GANs on metrics computed in the Inception-V3 embedding space despite being comparable or better generators in terms of both fidelity (as measured by human evaluations) and diversity. We follow their recommendation and use embeddings from a DINOv2 model~\cite{oquab2023dinov2} (specifically, its ViT-L/14 configuration), which was shown to not have such a bias.

\begin{figure}[t]
  \begin{minipage}[b]{.5\linewidth}
    \centering
    \small
    \begin{tabular}{ *{3}{c} }
    \toprule
      \textbf{Model}  & Fr\'echet Distance & $\mauve_\kl^\star$ \\
      \midrule
      StyleGAN2 & $298.59_{2.02}$ & $0.979_{0.034}$  \\
      Diffusion NCSN++ & $646.49_{5.80}$ & $0.648_{0.002}$
      \\
    \bottomrule
    \end{tabular}
    \captionof{table}{\small \textbf{GANs vs. diffusion models}: Comparing StyleGAN2-ADA with the diffusion model NCSN++~\cite{song2021score}. We use features from the DINOv2 (ViT-L/14) model and each metric is computed using $5000$ samples. The subscript denotes the standard deviation over 10 runs with different subsamples of the target distribution. $\mauve^\star$ is computed using vector quantization of size $k=100$. 
    \label{tab:expt:gan-v-diffusion}
    }
  \end{minipage}
  \hfill
  \begin{minipage}[b]{.4\linewidth}
    \centering
    \includegraphics[width=0.7\linewidth]{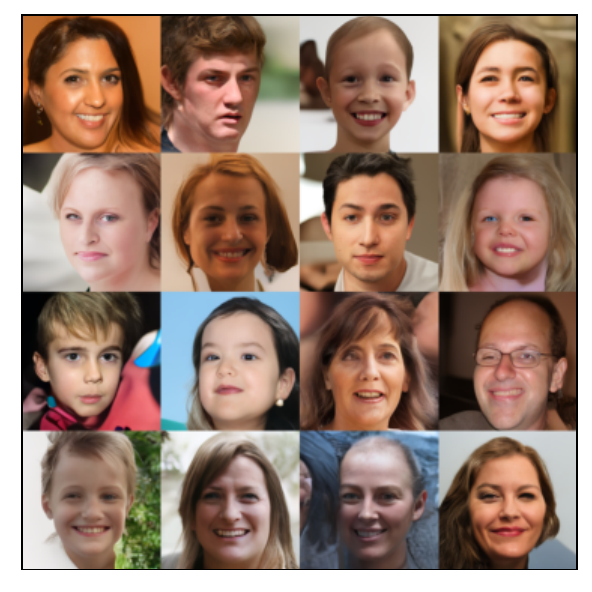}
    \captionof{figure}{\small Samples from the diffusion model NCSN++~\cite{song2021score}.
    \label{fig:diffusion-examples}}%
  \end{minipage}
\end{figure}

The results are given in \Cref{tab:expt:gan-v-diffusion}: StyleGAN2-ADA ($\psi=1$) outperforms the diffusion model by a large margin as per both \mauve and Fr\'echet distance. \Cref{fig:diffusion-examples}, which shows some samples from the diffusion model, explains the source of this large disparity. These generations contain more artifacts than the GANs generations (shown in \Cref{fig:vision}), including glaring asymmetries in facial features such as hairs or eyes.

Many successful diffusion-based generative models such as DALL-E 2~\cite{ramesh2022hierarchical} and Imagen~\cite{saharia2022photorealistic}
adopt a two-step pipeline: (a) generate low-resolution images with a diffusion model (e.g. $64 \times 64$ pixels), and (b) upsample the generation using one or more super-resolution models (e.g. $64^2 \rightarrow 256^2 \rightarrow 1024^2$ pixels). Our results above show that end-to-end diffusion modeling to directly generate high-resolution images remains an important open problem.

\subsubsection{Summary}
These results, together with those from the preceding sections, indicate that the general recipe of approximating gaps between distributions of complex high-dimensional objects using embeddings from a pre-trained deep net using $f$-divergence frontiers and \mauve is a powerful one.

\subsection{Tightness of the Statistical Error Bounds}

We conduct a numerical study to empirically investigate the tightness of the statistical error bounds presented in \Cref{thm:fdiv:consistency}. Using the frontier integral $\fint_\kl$ as a representative summary of the $f$-divergence frontier, we investigate the estimation error in divergence frontier summaries as a function of the sample size $n$ and the quantization size $k$ from samples.

We consider two domains: text generation in the web text domain using a pretrained GPT-2 large and nucleus sampling with $p=0.95$ (\S\ref{sec:expt-setup}) and face image generation using a StyleGAN2-ADA model pretrained on FFHQ sampled using $\psi=1$ (\S\ref{sec:expt:vision}).

We study the statistical error incurred by the plug-in estimator using $n$ samples to estimate
the population divergence, where each population distribution contains $N$ texts/images ($N=5000$ for the text domain and $N=50000$ for the image domain).
Following the recipe of \S\ref{sec:mauve:quant},
we first represent each text/image by its features.
Next, we quantize these $2N$ features into $k$ bins using $k$-means clustering. For each support size $k$, this gives us quantized distributions $P_{\Scal_k}$ and $Q_{\Scal_k}$.
Then, we sample $n$ i.i.d.~examples from each of the two distributions and use their empirical versions $\hat P_{\Scal_k, n}$ and $\hat Q_{\Scal_k, n}$ to compute $\fint(\hat P_{\Scal_k, n}, \hat Q_{\Scal_k, n})$.
We estimate the statistical error $\expect\abs{\fint_\kl(\hat P_{\Scal_k, n}, \hat Q_{\Scal_k, n}) - \fint_\kl(P_{\Scal_k}, Q_{\Scal_k})}$ from a \textbf{Monte Carlo} estimate using 100 random trials and compare it with two bounds from \Cref{thm:fdiv:consistency}: 
\begin{enumerate}[nosep, label=(\alph*)]
    \item \textbf{Bound}: the distribution independent bound $(\sqrt{k/n} + k/n) \log{n}$, and
    \item \textbf{Oracle Bound}: the distribution dependent bound $(\alpha_n(P) + \alpha_n(Q)) \log{n} + \beta_n(P) + \beta_n(Q)$ assuming the quantities $\alpha_n$ and $\beta_n$ (defined in \Cref{thm:fdiv:consistency}) are known.
\end{enumerate}
We fix the support size (i.e., the quantization size) $k$ and plot each of these quantities in a log-log plot with varying $n$ and compare their \emph{slope}.\footnote{A log-log plot of the function $f(x) = cx^\lambda$ is a straight line with slope $\lambda$, which thus captures the \emph{degree}.}
We then repeat the experiment with $n$ fixed and $k$ varying.
We scale the bounds by a factor of 30 for easier visual comparison of their slopes; this only changes the intercept and leaves the slope unchanged.

\myparagraph{Results}
\Cref{fig:main:bound:new} contains the Monte Carlo estimate and the bounds of the statistical error for real text and image data.
In \Cref{fig:main:bound:new}(b), we see that the oracle bound captures the right rate for small sample sizes where $k/n > 1$. Whereas, for large $n$, the distribution-independent bound is better at matching the slope of the Monte Carlo estimate.
The same is true for \Cref{fig:main:bound:new}(c), where the oracle bound is better for large $k$.
For parts (a) and (d), however, both bounds do not capture the right slope of the Monte Carlo estimate; \Cref{thm:fdiv:consistency} is not a tight upper bound in this case.
Yet, we notice that \Cref{thm:addb:fdiv:consistency} is still a valid upper bound. Indeed, for part (a), we observe that the rate of decrease of the Monte Carlo estimate is only faster than the bound but not slower.
Overall, these results demonstrate the favorable statistical error properties of \mauve.

\begin{figure}[t]
    \centering
    \includegraphics[width=\textwidth]{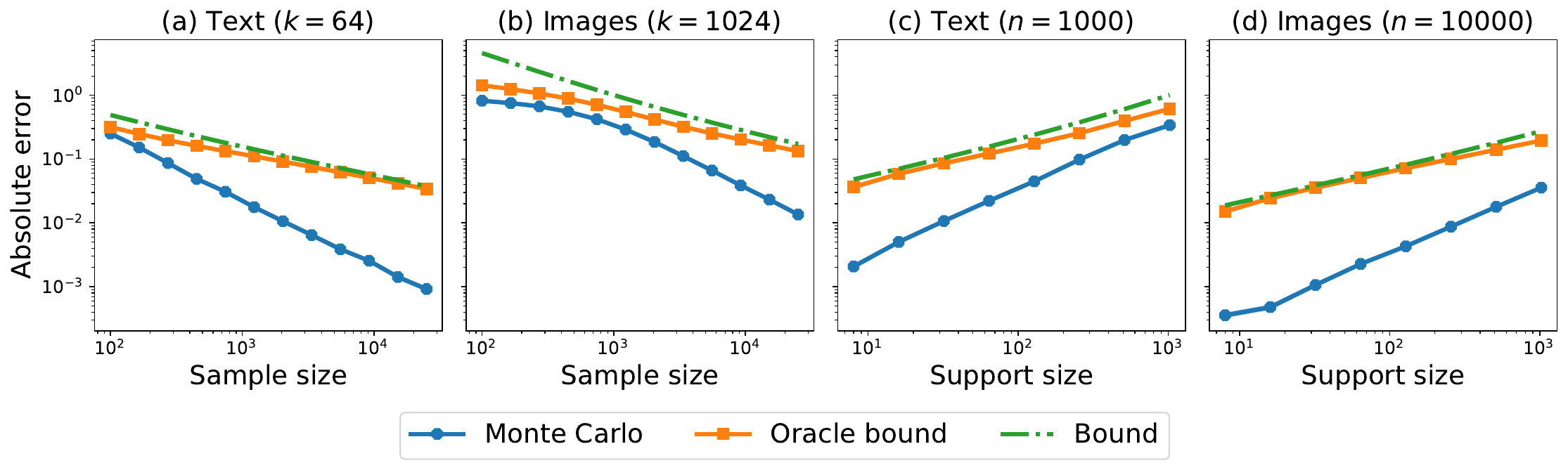}
    \caption{Statistical error of the estimated frontier integral $\fint_\kl$ on real text/image data, as a function of the sample size $n$ and support size $k$. 
    \textbf{(a)}: Text data with $k = 64$; \textbf{(b)}: Image data with $k = 1024$; \textbf{(c)}: Text data with $n = 10^3$; \textbf{(d)}: Image data with $n = 10^4$. These bounds are scaled by $30$.}
    \label{fig:main:bound:new}
\end{figure} 
\section{Empirical Recommendations} \label{sec:best-practices}

Following the introduction of \mauve in the conference paper \citep{pillutla2021mauve}, it has been 
adopted by the language modeling community for measuring performance and hyper-parameter tuning in diverse language generation settings, including contrastive decoding \cite{su2022contrastive,li2022contrastive}, truncation decoding \cite{meister2022locally,hewitt2022truncation}, and momentum decoding \cite{lan2022momentum}; controllable text generation \cite{yang2022unified}; architectural innovations \cite{hu2022fuse}; and differentially private language generation \cite{mattern2022differentially,yue2022synthetic,kurakin2023harnessing}.

We review some subtleties of using the proposed measures in practice and offer some practical guidelines.

\myparagraph{Aligning Automatic Evaluation to the Goal of Generative Modeling}
A common objective of generative modeling is to exactly match the model distribution $Q$ to a real data distribution $P$.
As discussed in \Cref{sec:mauve}, this can fail due to a Type I error, where the model produces unrealistic or low-quality data, or a Type II error, where the model is unable to produce some plausible real samples and fails to capture the diversity of real data. 
On the other hand, there are scenarios where ensuring a low Type I error is the only objective of generation (and matching the target $P$ is not important).
For instance, correctness is the key objective in machine translation, while using a diverse vocabulary is not the main concern.

The proposed divergence frontier summaries, as measures of the gap between a model distribution $Q$ and real data distribution $P$, are 
\emph{well-suited for the first objective and are ill-suited for the second one}.
For instance, in the context of open-ended text generation, \citep{su2022empirical} empirically show that contrastive search has a lower \mauve score than nucleus sampling while producing higher quality text (i.e., lower type I error) as inferred from human evaluations. This can be explained by a large type II error in contrastive search, leading to a large gap or smaller \mauve score. Indeed, each token in contrastive search is chosen deterministically from its top-$K$ vocabulary for small $K<10$, so it can fail to generate the occasional surprising or low-probability words found in human text~\cite{holtzman2019curious}.

In summary, we recommend the use of the proposed divergence frontier summaries when the goal of the generative model is to match \emph{both} the quality and diversity of the target real data distribution.

\myparagraph{Relative Comparisons Instead of Absolute Scores}
We find that the proposed methods are best suited for relative comparisons while the absolute scores are less meaningful.
For instance, if we wish to find which model distribution among $Q_1$ and $Q_2$ has a smaller gap to the target distribution $P$, 
we can compare $\mauve(P, Q_1)$ to $\mauve(P, Q_2)$.
The individual value of $\mauve(P, Q_i)$ can vary based on the computational approximation, its hyperparameters, and the number of samples. 
Indeed, we only consider the rankings induced by \mauve in \Cref{sec:experiments} by comparing the Spearman rank correlation with other rankings. 

\myparagraph{Randomness and Standard Deviations}
There are multiple sources of randomness in the computation of \mauve: the randomness from sampling for stochastic decoding algorithms, 
as well as the random initialization for $k$-means quantization. 
Since the absolute values of the proposed measures are not meaningful, the standard deviations are equally important in making relative comparisons.
We strongly recommend taking into account the standard deviation across multiple runs rather than just the mean even for relative comparisons; the worst-case Spearman rank correlation defined in \eqref{eq:worse-case-spearman} is one such measure. 
We also observed that, while the proposed measures can capture the basic properties as in \Cref{sec:expt:properties}, it is much harder to quantify subtle differences (e.g., when trying to improve over nucleus sampling). In this case, we recommend increasing the sample size or the number of random seeds to reduce the uncertainty in the statistical estimation.

\myparagraph{Sample Size and Text Length}
The greater the number of samples, the smaller the statistical estimation error (cf. \Cref{sec:compute}).
We recommend empirically that each distribution contains at least 1000 samples.
The proposed measure computed with a smaller number of samples is biased towards optimism (that is, the score typically goes down as the number of samples increases) and exhibits a larger standard deviation.
Likewise, we find that the proposed measures can capture the gap between long texts (at least 256 tokens, preferably 512 tokens) but they might not always capture the difference between shorter texts (see the overlapping shaded areas denoting the standard deviation in \Cref{fig:expt:length:main}). 
In \Cref{sec:experiments}, we use 5000 samples of up to 1024 tokens (with a prefix length of 35) to compute \mauve and we report the mean and standard deviation over 5 repetitions.

\acks{
Part of this work was done while Zaid Harchaoui was visiting the Simons Institute for the Theory of Computing, and while Krishna Pillutla, Lang Liu, John Thickstun, and Rowan Zellers were at the University of Washington. This work was supported by NSF DMS-2134012, NSF CCF-2019844, NSF DMS-2023166, the DARPA MCS program through NIWC Pacific (N66001-19-2-4031), the CIFAR ``Learning in Machines \& Brains'' program, a Qualcomm Innovation Fellowship, and faculty research awards.
}

\appendix

\section*{Appendix}
The outline of the appendix is as follows:
\begin{itemize}[nosep]
    \item \Cref{sec:a:properties}: Complete proofs of divergence frontier properties from \Cref{sec:mauve}. 
    \item \Cref{sec:a:quant}: Full proofs of estimation via quantization from \Cref{sec:mauve:quant}.
    \item \Cref{sec:a:parametric}: 
    {Details of the parametric approximation approach mentioned in \Cref{sec:compute}}.
    \item \Cref{sec:a:expt}: Additional experimental results to augment those in \Cref{sec:experiments}.
    \item \Cref{sec:a:human-eval}: Additional details of the human evaluations described in \Cref{sec:expt-setup:human-eval}.
\end{itemize}

\section{Properties of the Divergence Frontiers} \label{sec:a:properties}
We give a closed-form expression for \mauveray, for the special case of the KL divergence.  
\begin{property} \label{prop:fi-properties-kl}
    The integral summary \mauveray of the KL divergence frontier 
    is an $f$-divergence generated 
        by the convex function 
        \[
            \tilde f_{\kl}(t) = \frac{t+1}{2} - \frac{t}{t-1} \log t \,,
        \]
        with the understanding that 
        $\tilde f_{\kl}(1) = \lim_{t \to 1} \tilde f_{\kl}(t) = 0$, 
\end{property}
\begin{proof}
    Let $P$ and $Q$ be dominated by some probability measure $\mu$ with density $p$ and $q$, respectively.
    We will establish the expression
    \begin{align} \label{eq:fint:closed-form}
        \mray(P, Q) = \int_{\Xcal}
        \indone{\{p(x) \neq q(x)\}} \left(\frac{p(x) + q(x)}{2} - \frac{p(x)q(x)}{p(x) - q(x)} \log\frac{p(x)}{q(x)} 
        \right) \D\mu(x) \,,
    \end{align}
    with the convention $0 \log 0 = 0$.
    This gives the expression for $\tilde f_\kl$ from the definition of an $f$-divergence.
     
    We now establish \eqref{eq:fint:closed-form}. 
    Denote $\bar \lambda = 1-\lambda$.
    By Tonelli's theorem, it holds that $\mray_\kl(P, Q) = 2\int_{\Xcal} h(p(x), q(x)) \D \mu(x)$, where
    \[
        h(p, q) = 
        \int_0^1 \left(
        \lambda p \log p + \bar\lambda q\log q
        - (\lambda p  + \bar\lambda q) \log (\lambda p  + \bar\lambda q)\right) \D\lambda.
    \]
    When $p=q$, the integrand is $0$. 
    If $q = 0$, then the second term inside the integral is $0$, while the first term is 
    $\int_0^1 \lambda p \log({1}/{\lambda}) \D\lambda = {p}/{4} $.
    Finally, when $p \neq q$ are both non-zero, we evaluate the integral to get, 
    \begin{align*}
        h(p, q) = 
        \frac{p}{2}\log p + \frac{q}{2}\log q
        - \frac{2p^2 \log p - p^2 - 2 q^2 \log q + q^2}{4(p - q)}\,,
    \end{align*}
    and rearranging the expression gives \eqref{eq:fint:closed-form}. 
\end{proof}

Next, we give a technical lemma used to establish properties of $\mray_f$ and $\midp_f$.

\begin{proposition} \label{prop:lerror-max}
    Let $P, Q \in \Pcal(\Xcal)$ be probability measures with finite support.
    Then, the linearized cost $\lerror{f, \lambda}$ defined in \Cref{eq:linear_cost} satisfies the bound
    \[
        \lerror{f, \lambda}(P \Vert Q) \le \lambda \, f^*(\lambda) + (1-\lambda) f^*(1-\lambda) 
            + 2 \lambda(1-\lambda) f(0) \,.
    \]
\end{proposition}
\begin{proof}
    Denote $\bar \lambda = 1- \lambda$.
    Let $P, Q \in \Delta^{k-1}$ be discrete distributions over $k < \infty$ items.
    The function $P, Q \mapsto \lerror{f, \lambda}(P \Vert Q)$, by virtue of being an $f$-divergence, is jointly convex in $P, Q$. 
    So, $\lerror{f, \lambda}(P \Vert Q)$ is maximized for $P^\star, Q^\star$ that lie at some vertices of 
    the probability simplex $\Delta^{k-1}$.
    We can rule out $P^\star = Q^\star$ as $\lerror{f, \lambda}(P \Vert Q) = 0$ in this case.
    Therefore, without loss of generality, we can assume that $P^\star = (1, 0, \ldots, 0) \in \Delta^{k-1}$
    and $Q^\star = (0, 1, 0, \ldots, 0) \in \Delta^{k-1}$. 
    Plugging this in gives the upper bound 
    \[
        \lerror{f, \lambda}(P^\star \Vert Q^\star)
        = \lambda^2 f(1/\lambda) + 2 \lambda \bar \lambda f(0) + \bar\lambda^2 f(1/\bar\lambda)
        = \lambda f^*(\lambda) + \bar \lambda f^*(\bar \lambda) + 2 \lambda \bar \lambda f(0) \,.
    \]
\end{proof} 
\section{Proofs of Theoretical Bounds: Quantization} \label{sec:a:quant}
In this section, we give the complete proofs of quantization in \Cref{sec:mauve:quant}. The outline is as follows:
\begin{itemize}[nosep]
\item \Cref{apppend:sub:stat_error}: Proof of the statistical error bound for the empirical estimator (\Cref{thm:fdiv:consistency}).
\item \Cref{apppend:sub:stat_error_smoothing}: Proof of the statistical error bound for the add-constant estimator (\Cref{thm:addb:fdiv:consistency}).
\item \Cref{append:sub:quant_error}: Proof of the quantization error bound (\Cref{thm:quant_error_fdiv}).
\end{itemize}

\subsection{Statistical Error Bound}
\label{apppend:sub:stat_error}
In this section, we prove \Cref{thm:fdiv:consistency}.

The proof relies on two key lemmas---the approximate Lipschitz lemma (\Cref{lem:fdiv:taylor-ex})
and the missing mass lemma (\Cref{lem:fdiv:expected-missing-mass}). 
The argument breaks into two cases in $P$ (and analogously for $Q$) for each atom $a \in \Xcal$:
\begin{enumerate}[noitemsep,topsep=0pt, label=(\alph*)]
    \item $\hat P_{n, a} > 0$: Since $\Phatn$ is an empirical measure, we have that $\hat P_{n, a} \ge 1/n$. In this case the approximate Lipschitz lemma gives us the Lipschitzness in $\norm{P - \Phatn}_\tv$ up to a factor of $\log n$.
    \item $\hat P_{n, a} = 0$: In this case, the mass corresponding to $P_{a}$ is missing in the empirical measure and we directly bound its expectation following similar arguments as in the missing mass literature; see, e.g., \cite{berend2012missing,mcallester2005concentration}.
\end{enumerate}

\subsubsection{Approximate Lipschitz Property}
First, we express the derivatives of $\psi(p, q) = q f(p/q)$ in terms of the derivatives of $f$:
\begin{subequations}
\begin{align}
    \label{eq:psi-partial-p}
    \frac{\partial\psi}{\partial p}(p, q)
    &= \fdiv'\left( \frac{p}{q} \right)
    = \ftil\left( \frac{q}{p}  \right) - \frac{q}{p} \ftilg\left( \frac{q}{p}  \right) \\
    \label{eq:psi-partial-q}
    \frac{\partial\psi}{\partial q} (p, q) 
    &= \fdiv\left( \frac{p}{q}  \right) - \frac{p}{q} \fdiv'\left( \frac{p}{q}  \right)  
    =  \ftilg\left( \frac{q}{p} \right) \\
    \label{eq:psi-partial-pp}
    \frac{\partial^2\psi}{\partial p^2} (p, q) 
    &= \frac{1}{q} \fdiv''\left( \frac{p}{q} \right) 
    = \frac{q^2}{p^3} \ftilh\left( \frac{q}{p} \right) \ge 0 \\
    \label{eq:psi-partial-qq}
    \frac{\partial^2\psi}{\partial q^2} (p, q) 
    &= \frac{p^2}{q^3} \fdiv''\left( \frac{p}{q} \right)
    = \frac{1}{p} \ftilh\left( \frac{q}{p} \right)  \ge 0 \\
    \label{eq:psi-partial-pq}
    \frac{\partial^2\psi}{\partial p \partial q} (p, q) 
    &= -\frac{p}{q^2} \fdiv''\left( \frac{p}{q} \right)
    = -\frac{q}{p^2} \ftilh\left( \frac{q}{p} \right)  \le 0 \,,
\end{align}
\end{subequations}
where the inequalities $\fdiv'', \ftilh \ge 0$ followed from convexity of 
$\fdiv$ and $\ftil$ respectively.

We now present the main lemma that shows that the function $\psi$ is nearly Lipschitz, up to a log factor. 
This lemma can be leveraged to directly obtain a bound on the statistical error of the $\fdiv$-divergence in terms of the expected total variation distance, provided the probabilities are not too small. 

\begin{lemma}\label{lem:fdiv:taylor-ex}
    Suppose that $\fdiv$ satisfies Assumption~\ref{asmp:fdiv}.
    Consider
    $\psi:[0, 1]\times [0, 1] \to [0, \infty)$ given by
    $\psi(p, q) = q \fdiv(p/q)$.
    We have, for all $p, p', q, q' \in [0, 1]$ 
    with $p\vee p' > 0$, $q\vee q' > 0$, that
    \begin{align*}
        |\psi(p', q) - \psi(p, q)| &\le 
        \left(\ConstI \max\left\{1, \log\frac{1}{p \vee p'}\right\} + \ConstZTil\vee\ConstII  \right) | p-p'| \\
        |\psi(p, q') - \psi(p, q)| &\le 
        \left(\ConstITil \max\left\{1, \log\frac{1}{q \vee q'}\right\} + \ConstZ \vee \ConstIITil \right) | q-q'| \,.
    \end{align*}
\end{lemma}
\begin{proof}
    We only prove the first inequality. The second one is identical with the use of $\ftil$ rather than $\fdiv$. 
    Suppose $p' \ge p$. 
    From the fact that $\psi$ is convex in $p$
    together with a Taylor expansion of $\psi(\cdot, q)$ around $p'$, we get, 
    \begin{align*}
    0 \le \psi(p, q) - \psi(p', q) 
        &- (p- p')\frac{\partial \psi}{\partial p}(p', q)
        = \frac{1}{2} \int_{p'}^p \frac{\partial^2\psi}{\partial p^2}(s, q)(p-s) \D s \\
        &= -\frac{p}{2} \int_{p}^{p'} \frac{\partial^2\psi}{\partial p^2}(s, q) \D s
        + \frac{1}{2} \int_p^{p'} s \frac{\partial^2\psi}{\partial p^2}(s, q) \D s \\
        &\le 0 + \ConstII(p' - p)\,,
    \end{align*}
    where we used $\partial^2\psi / \partial p^2$ is non-negative due to convexity and, 
    by \eqref{eq:psi-partial-pp} and Assumption~\ref{asmp:fdiv:2nd-deriv},
    \[
        s \frac{\partial^2\psi}{\partial p^2}(s, q)
        = \frac{s}{q} \fdiv''\left({s}/{q}\right) \le 2\ConstII \,.
    \]
    This yields
    \[
        - (p' - p) \frac{\partial \psi}{\partial p}(p', q) \le
        \psi(p, q) - \psi(p', q) \le 
        - (p' - p) \frac{\partial \psi}{\partial p}(p', q)
        + \ConstII(p' - p) \,.
    \]
    We consider two cases based on the sign of 
    $\tfrac{\partial \psi}{\partial p}(p', q) = f'(p/q)$ (cf. Eq. \eqref{eq:psi-partial-p}).
    
    \myparagraph{Case 1} 
    $\tfrac{\partial \psi}{\partial p}(p', q) \ge 0$.
    Since $q \mapsto f'(p/q)$ is 
    decreasing in $q$, 
    we have
    \begin{align*}
    0 \le (p' - p) \frac{\partial \psi}{\partial p}(p', q)
    = (p' - p) \fdiv'(p/q) 
    \le \lim_{q \to 0} (p' - p) \fdiv'(p/q) 
    = (p' - p) \ftil(0) \,,
    \end{align*}
    where we used $\fdiv'(\infty) = \ftil(0)$ 
    from \Cref{lem:fdiv:ftil-infty}.
    From Assumption~\ref{asmp:fdiv:bounded}, we get the bound
    \[
        |\psi(p, q) - \psi(p', q)| \le  (\ConstZTil \vee \ConstII) (p' - p) \,.
    \]
    
    \myparagraph{Case 2}
    $\tfrac{\partial \psi}{\partial p}(p', q) < 0$.
    By Assumption~\ref{asmp:fdiv:1st-deriv},
    it holds that
    \[
        \left|\frac{\partial \psi}{\partial p}(p', q)\right|
        \le \ConstI\, \max\{1, \log (q/p')\}
        \le \ConstI\, \max\{1, \log (1/p')\}\,,
    \]
    and thus
    \[
        |\psi(p, q) - \psi(p', q)| \le  \left(\ConstI \max\left\{1, \log \frac{1}{p'}\right\} + \ConstII \right) (p' - p) \,.
    \]
\end{proof}

With the above lemma, the estimation error of the empirical $f$-divergence can be upper bounded by the total variation distance between the empirical measure and its population counterpart up to a logarithmic factor, where:
\begin{align}
    \norm{\Phatn - P}_\tv = \sum_{a \in \Xcal} |\hat P_{n, a} - P_{a}| \,.
\end{align}

For the first part, we further upper bound the expected total variation distance of the plug-in estimator, which is 
\[
    \norm{\Phatn - P}_\tv = \sum_{a \in \Xcal} |\hat P_{n, a} - P_{a}| \,.
\]
\begin{lemma}\label{lem:fdiv:l1-bound}
    Assume that $P$ is discrete.
    For any $n \ge 1$, it holds that
    \begin{align*}
        \expect \norm{\Phatn - P}_\tv
        \le \alpha_n(P).
    \end{align*}
    Furthermore, if $k = |\Supp{P}| < \infty$, then
    \[
        \expect \norm{\Phatn - P}_\tv
        \le \alpha_n(P)
        \le \sqrt{\frac{k}{n}} \,.
    \]
\end{lemma}
\begin{proof}
    Using Jensen's inequality, we have,
    \begin{align*}
        \expect \sum_{a \in \Supp{P}} |\hat P_{n, a} - P_{a}|
        &\le \sum_{a \in \Supp{P}} \sqrt{\expect(\hat P_{n, a} - P_{a})^2} \\
        &= \sum_{a \in \Supp{P}} \sqrt{\frac{P_{a}(1 - P_{a})}{n}}
        \le \alpha_n(P)\,,
    \end{align*}
    If $k < \infty$, then it follows from Jensen's inequality applied to the concave function $t \mapsto \sqrt{t}$ that
    \[
        \frac{1}{k} \sum_{i=1}^k \sqrt{a_k} \le \sqrt{\frac{1}{k}{\sum_{i=1}^k a_k}} \,.
    \]
    Hence, $\alpha_n(P) \le \sqrt{k/n}$ and it completes the proof.
\end{proof}

\subsubsection{Missing Mass Computation}
For the second part, we treat the missing mass directly.
\begin{lemma}[Missing Mass] \label{lem:fdiv:expected-missing-mass}
    Assume that $k = |\Supp{P}| < \infty$.
    Then, for any $n \ge 3$,
    \begin{align}
        \expect\left[ \sum_{a \in \Xcal} \indone\big\{\hat P_{n, a}=0\big\} P_{a} \right] &\le \frac{k}{n} \label{eq:exp_miss_mass} \\
        \beta_n(P) := \expect\left[ \sum_{a\in\Xcal} \indone\big\{\hat P_{n, a}=0\big\} P_{a} \left(  1 \vee \log\frac{1}{P_{a}}\right) \right] &\le \frac{k \log n}{n} \label{eq:exp_var_miss_mass} \,,
    \end{align}
    where $a \vee b := \max\{a, b\}$.
\end{lemma}
\begin{proof}
    We prove the second inequality. The first one is identical.
    Note that $\expect[\indone\{\hat P_{n, a}=0\}] = 
    \prob(\hat P_{n, a} = 0) = (1 - P_{a})^n$.
    Therefore, the left-hand side (LHS) of the second inequality is
    \begin{align*}
        \text{LHS} &= \sum_{a \in \Xcal}  (1-P_{a})^n P_{a} \max\{1, -\log P_{a}\} \\
        &\le \sum_{a \in \Xcal} \frac{1}{n} \vee \frac{\log n}{n} = \frac{k \log n}{n} \,,
    \end{align*}
    where we used \Cref{lem:techn:missing-mass-2} and \Cref{lem:techn:missing-mass-1}.
\end{proof}
\begin{remark}
    According to \cite[Prop. 3]{berend2012missing}, the bound $k/n$ in \eqref{eq:exp_miss_mass} is tight up to a constant factor.
\end{remark}

\subsubsection{Full Proof of the Statistical bound}
Now, we are ready to prove \Cref{thm:fdiv:consistency}.
\begin{proof}[Proof of \Cref{thm:fdiv:consistency}]
    Define $\Delta_{n,m}(a) := \left|\psi\big(P_{a}, Q_{a}\big) - \psi\big(\hat P_{n, a}, \hat Q_{m, a}\big)\right|$. We have from the triangle inequality that 
    \[
        \Delta_{n,m}(a) \le 
       \underbrace{\left|\psi\big(P_{a}, Q_{a}\big) - \psi\big(\hat P_{n, a}, Q_{a}\big)\right|}_{=:\Tcal_1(a)}
        + 
        \underbrace{\left|\psi\big(\hat P_{n, a}, Q_{a}\big) - \psi\big(\hat P_{n, a}, \hat Q_{m, a}\big)\right|}_{=:\Tcal_2(a)} \,.
    \]
    Since $\hat P_{n, a} = 0$ or $\hat P_{n, a} \ge  1/n$, 
    the approximate Lipschitz lemma (\Cref{lem:fdiv:taylor-ex}) gives
    \[
        \Tcal_1(a) \le 
        \begin{cases}
            P_{a} \left(\ConstI \max\{1, \log (1/P_{a})\} + \ConstZTil \vee \ConstII\right) \,,
            & \text{ if } \hat P_{n, a} = 0, \\
            |P_{a} - \hat P_{n, a}|\,\big(\ConstI \log n + \ConstZTil \vee \ConstII\big) \,, &\text{ else}.
        \end{cases}
    \]
    Consequently, \Cref{lem:fdiv:l1-bound} yields
    \begin{align*}
        \sum_{a \in \Xcal} \expect[\Tcal_1]
        &\le \sum_{a \in \Xcal} \expect\left[ \indone\{\hat P_{n, a} = 0\} P_{a} \left(\ConstI \max\{1, \log (1/P_{a})\} + \ConstZTil \vee \ConstII\right) \right] \\
        &\quad + \sum_{a \in \Xcal} \expect\left[ \abs{\hat P_{n, a} - P_{a}} \right]\big(\ConstI \log n + \ConstZTil \vee \ConstII\big) \\
        &\le \left(\ConstI + \ConstZTil \vee \ConstII\right) \beta_n(P) + \big(\ConstI \log n + \ConstZTil \vee \ConstII\big) \alpha_n(P)\,.
    \end{align*}
    Since $\psi(p, q) = q\fdiv(p/q) = p\ftil(q/p)$,
    an analogous bound holds for $\Tcal_2$ with the appropriate adjustment of constants.
    Hence, the inequality \eqref{eq:fdiv:stat_error_oracle} holds.
    Moreover, when $k < \infty$, the inequality \eqref{eq:fdiv:stat_error} follows by invoking again \Cref{lem:fdiv:expected-missing-mass} and \Cref{lem:fdiv:l1-bound}.
\end{proof}
We now prove \Cref{prop:consis_df}. 
\begin{proof}[Proof of \Cref{prop:consis_df}]
    The inequality is a direct consequence of \Cref{thm:fdiv:consistency}.
    Recall from \Cref{property:frontier-as-f-div} that $\Df{P}{R_\lambda} = D_{f_\lambda}(P \Vert Q)$
    where $f_\lambda(t) := f(t/(\lambda t + 1 - \lambda)) (\lambda t + 1 - \lambda)$.
    From the proof of \Cref{thm:fdiv:consistency} we have
    \begin{align*}
        &\quad \abs{D_{f_\lambda}(\hat P_n \Vert \hat Q_m) - D_{f_\lambda}(P \Vert Q)} \\
        &\le \sum_{a \in \Xcal} \indone\{\hat P_{n, a} = 0\}\, P_{a} \left(\ConstI \max\{1, \log (1/P_{a})\} + \ConstZTil \vee \ConstII\right) \\
        &\quad + \sum_{a \in \Xcal} \indone\{\hat Q_{m, a} = 0\} \, Q_{a} \left(\ConstITil \max\{1, \log (1/Q_{a})\} + \ConstZ \vee \ConstIITil \right) \\
        &\quad + \sum_{a \in \Xcal} \big|P_{a} - \hat P_{n, a}\big| \big(\ConstI \log n + \ConstZTil \vee \ConstII\big)
        + \sum_{a \in \Xcal} \big|Q_{a} - \hat Q_{m, a}\big| \big(\ConstITil \log m + \ConstZ \vee \ConstIITil\big)\,.
    \end{align*}
    Note that, for the interpolated KL divergence, we have
    \begin{align*}
        \ConstZ = 1 - \lambda \le 1&, \quad \ConstZTil = \log{\frac{1}{\lambda}} - 1 + \lambda \le \log{\frac{1}{\lambda_{n,m}}} \\
        \ConstI = 1&, \quad \ConstITil = \frac{(1-\lambda)^2}{\lambda} \le \frac{1}{\lambda_{n,m}} \\
        \ConstII = 1/2&, \quad \ConstIITil = \frac{1-\lambda}{8\lambda} \le \frac{1}{8\lambda_{n,m}}
    \end{align*}
    for all $\lambda \in [\lambda_{n,m}, 1 - \lambda_{n,m}]$.
    The claim then follows from the same steps of \Cref{thm:fdiv:consistency}.
\end{proof}

\subsection{Statistical Error Bound with Smoothing}
\label{apppend:sub:stat_error_smoothing}

In this section, we apply add-constant smoothing to estimate the $\fdiv$-divergences and study its statistical error.

Consider $P \in \Pcal(\Xcal)$ and an i.i.d.~sample $\{X_i\}_{i=1}^n \sim P$.
The add-constant estimator of $P$ is defined by
\begin{align*}
    \hat P_{n,a}^b = \frac{N_a + b}{n + kb}, \quad \mbox{for all } a \in \Xcal\,,
\end{align*}
where $b > 0$ is a constant and $N_a = \abs{\{i \in [n]: X_i = a\}}$ is the number of times the symbol $a$ appears in the sample.
In practice, $b = b_a$ could be different depending on the value of $N_a$, but we use the same constant $b$ for simplicity.
Similarly, We define $\hat Q_{m}^b$ with $M_a = \abs{\{i \in [m]: Y_i = a\}}$.
The goal is to upper bound the statistical error
\begin{align}
    \expect\abs{\Df{P}{Q} - \Df{\hat P_n^b}{\hat Q_m^b}}
\end{align}
under \Cref{asmp:fdiv}.

Compared to the statistical error of the plug-in estimator, a key difference is that each entry in the add-constant estimator is at least $(n + kb)^{-1} \wedge (m + kb)^{-1}$.
Hence, we can directly apply the approximate Lipschitz lemma without the need to control the missing mass part.
Another difference is that the total variation distance is now between the add-constant estimator and its population counterpart, which can be bounded as follows.
\begin{lemma}\label{lem:tv_smooth}
    Assume that $k = \Supp{P} < \infty$.
    Then, for any $b > 0$,
    \begin{align*}
        \sum_{a \in \Xcal} \expect\abs{\hat P_{n,a}^b - P_a} \le \sum_{a \in \Xcal} \frac{\sqrt{nP_a(1 - P_a)} + bk \abs{P_a - 1/k}}{n+kb} \le \frac{\sqrt{kn} + 2b(k-1)}{n + kb} \,.
    \end{align*}
\end{lemma}
\begin{proof}
    Note that
    \begin{align*}
        \abs{\hat P_{n,a}^b - P_a}
        = \abs{\frac{N_a - nP_a}{n + kb} + \frac{b(1 - kP_a)}{n + kb}}
        \le \abs{\frac{N_a - nP_a}{n + kb}} + \abs{\frac{b(1 - kP_a)}{n + kb}}.
    \end{align*}
    Using Jensen's inequality, we have
    \begin{align*}
        \sum_{a \in \Xcal} \expect\abs{\hat P_{n,a}^b - P_a}
        &\le \sum_{a \in \Xcal} \left[ \sqrt{\expect\abs{\frac{N_a - nP_a}{n + kb}}^2} + \frac{c\abs{1 - kP_a}}{n + kb} \right] \\
        &= \sum_{a \in \Xcal} \left[ \frac{\sqrt{n P_a(1 - P_a)}}{n + kb} + \frac{b k\abs{1/k - P_a}}{n + kb} \right].
    \end{align*}
    We claim that
    \begin{align*}
        \sum_{a \in \Xcal} \abs{P_a - \frac{1}{k}} \le \frac{2(k-1)}{k}.
    \end{align*}
    If this is true, we have
    \begin{align*}
        \sum_{a \in \Xcal} \expect\abs{\hat P_{n,a}^b - P_a} \le \frac{\sqrt{kn} + 2b(k-1)}{n + kb}\,,
    \end{align*}
    since $\sum_{a \in \Xcal} \sqrt{P_a(1 - P_a)} \le \sqrt{k}$.
    It then remains to prove the claim.
    Take $a_1, a_2 \in \Xcal$ such that $P_{a_1} \ge k^{-1} \ge P_{a_2}$.
    It is clear that
    \begin{align*}
        \abs{P_{a_1} - \frac1k} + \abs{P_{a_2} - \frac1k}
        &\le \abs{P_{a_1} + P_{a_2} - \frac1k} + \abs{P_{a_2} - P_{a_2} - \frac1k} \\
        &= P_{a_1} + P_{a_2}.
    \end{align*}
    Repeating this argument gives
    \begin{align*}
        \sum_{a \in \Xcal} \abs{P_a - \frac1k} \le 1 - \frac1k + \frac{k-1}{k} = \frac{2(k-1)}{k}.
    \end{align*}
\end{proof}

Now we are ready to prove \Cref{thm:addb:fdiv:consistency}.

\begin{proof}[Proof of \Cref{thm:addb:fdiv:consistency}]
    Following the proof of \Cref{thm:fdiv:consistency}, we define
    \[
        \Delta_{n,m}(a) := \abs{\psi(P_a, Q_a) - \psi(\hat P_{n,a}^b, \hat Q_{m,a}^b)}\,.
    \]
    We have from the triangle inequality that 
    \[
        \Delta_{n,m}(a) \le 
       \underbrace{\left|\psi\big(P_a, Q_a\big) - \psi\big(\hat P_{n,a}^b, Q_a\big)\right|}_{=:\Tcal_1(a)}
        + 
        \underbrace{\left|\psi\big(\hat P_{n,a}^b, Q_a\big) - \psi\big(\hat P_{n,a}^b, \hat Q_{m,a}^b\big)\right|}_{=:\Tcal_2(a)} \,.
    \]
    Since $\hat P_{n,a}^b \ge  b/(n + kb)$, 
    the approximate Lipschitz lemma (\Cref{lem:fdiv:taylor-ex}) gives
    \[
        \Tcal_1(a) \le |P_a - \hat P_{n,a}^b|\,\big(\ConstI \log (n/b + k) + \ConstZTil \vee \ConstII\big) \,,
    \]
    By \cref{lem:tv_smooth}, it holds that
    \begin{align*}
        \frac{\sum_{a \in \Xcal} \expect[\Tcal_1(a)]}{\ConstI \log (n/b + k) + \ConstZTil \vee \ConstII}
        &\le \sum_{a \in \Xcal} \left[ \frac{\sqrt{n P_a}}{n + kb} + \frac{bk \abs{1/k - P_a}}{n + kb} \right]
        = \frac{n \alpha_n(P)}{n + kb} + \gamma_{n,k}(P) \\
        &\le \frac{\sqrt{kn} + 2b(k-1)}{n + kb}\,.
    \end{align*}
    Since $\psi(p, q) = q\fdiv(p/q) = p\ftil(q/p)$,
    an analogous bound holds for $\Tcal_2(a)$ with the appropriate adjustment of constants and the sample size.
    Putting these together, we get,
    \begin{align*}
        &\quad \expect\big|\Df{P}{Q} - \Df{\hat P_n^b}{\hat Q_m^b}\big|
        \le \expect\left[\sum_{a \in \Xcal} |\Delta_n(a)|\right] \\
        &\le \left[ \frac{n \alpha_n(P)}{n + kb} + \gamma_{n,k}(P) \right] \, \big(\ConstI \log (n/b + k) + \ConstZTil \vee \ConstII\big)  \\
        &\quad + \left[ \frac{m \alpha_m(Q)}{m + kb} + \gamma_{m,k}(Q) \right] \, \big(\ConstITil \log (m/b + k) + \ConstZ \vee \ConstIITil\big) \\
        &\le \big(\ConstI \log (n/b + k) + \ConstZTil \vee \ConstII\big) \frac{\sqrt{kn} + 2b(k-1)}{n + kb} \\
        &\quad + \big(\ConstITil \log (m/b + k) + \ConstZ \vee \ConstIITil\big) \frac{\sqrt{km} + 2b(k-1)}{m + kb}\,.
    \end{align*}
\end{proof}

\subsection{Quantization Error}
\label{append:sub:quant_error}
We establish a bound on the quantization error of $\fdiv$-divergences, i.e.,
\begin{align}
    \inf_{\abs{\Scal} \le k} \abs{\Df{P}{Q} - \Df{P_{\Scal}}{ Q_{\Scal}}},
\end{align}
where the infimum is over all partitions of $\Xcal$ of size no larger than $k$, and $P_{\Scal}$ and $Q_{\Scal}$ are the quantized versions of $P$ and $Q$ according to $\Scal$, respectively.
Note that we do not assume $\Xcal$ to be discrete in this section.
All the results hold for the linearized cost $\lerror{\lambda}(\hat P_n, \hat Q_n)$ and the frontier integral $\mray(\hat P_n, \hat Q_n)$ from \Cref{tab:fdiv:asmp-examples}.

Our analysis is inspired by the following result, which shows that the $\fdiv$-divergence can be approximated by its quantized counterpart; see, e.g., \cite[Theorem 6]{gyorfi1978fdiv}.
\begin{theorem}
    For any $P, Q \in \Pcal(\Xcal)$, it holds that
    \begin{align}
        \Df{P}{Q} = \sup_{\Scal} \Df{P_\Scal}{Q_\Scal},
    \end{align}
    where the supremum is over all finite partitions of $\Xcal$.
\end{theorem}

We now prove \Cref{thm:quant_error_fdiv}, the finite-partition analogue of this. 
\begin{proof}[Proof of \Cref{thm:quant_error_fdiv}]
    Assume $\fdiv(0) + \ftil(0) < \infty$. Otherwise, there is nothing to prove. 
    Fix two distributions $P, Q$ over $\Xcal$.
    Partition the measurable space $\Xcal$ into 
    \[
        \Xcal_1 = \left\{x \in \Xcal \, :\, \frac{\D P}{\D Q}(x) \le 1 \right\}\,,
        \quad\text{and,}\quad
        \Xcal_2 = \left\{x \in \Xcal \, :\, \frac{\D P}{\D Q}(x) > 1 \right\} \,,
    \]
    so that 
    \[
        \Df{P}{Q} = 
        \int_{\Xcal_1} \fdiv\left( \frac{\D P}{\D Q}(x)\right) \D Q(x)
        + 
        \int_{\Xcal_2} \ftil\left( \frac{\D Q}{\D P}(x)\right) \D P(x) 
        =: D_\fdiv^+(P \Vert Q) + D_{\ftil}^+(Q \Vert P) \,.
    \]
    We quantize $\Xcal_1$ and $\Xcal_2$ separately, 
    starting with $\Xcal_1$. Define sets $S_1, \cdots, S_k$ as 
    \[
        S_m = \left\{
        x \in \Xcal_1 \, :\, 
        \frac{\fdiv(0)(m-1)}{k} \le \fdiv\left( \frac{\D P}{\D Q}(x) \right) < \frac{\fdiv(0)m}{k} 
        \right\} \,,
    \]
    where the last set $S_k$ is also extended 
    to include $\{x \in \Xcal_1\,:\,\fdiv( (\D P / \D Q)(x)) = \fdiv(0)\}$.
    Since $\fdiv$ is nonincreasing on $(0, 1]$, it follows that $\sup_{x \in \Xcal_1} \fdiv( ({\D P}/{\D Q})(x)) \le \fdiv(0)$.
    As a result, the collection $\Scal = \{S_1, \cdots, S_k\}$ is a partition of $\Xcal_1$.
    This gives
    \begin{align} \label{eq:fdiv:quant:pf:1}
        \frac{\fdiv(0)}{k} \sum_{m=1}^k (m-1) \, Q[S_m] \le
        D_\fdiv^+(P \Vert Q) \le \frac{\fdiv(0)}{k} \sum_{m=1}^k m \, Q[S_m] \,.
    \end{align}
    Further, since $\fdiv$ is nonincreasing on $(0, 1]$, we also have 
    \[
        \frac{\fdiv(0) (m-1)}{k} \le
        \fdiv\left(\sup_{x \in F_m} \frac{\D P}{\D Q}(x)\right)  \le 
        \fdiv\left( \frac{P[F_m]}{Q[F_m]}\right)
        \le 
        \fdiv\left(\inf_{x \in F_m} \frac{\D P}{\D Q}(x)\right) \le \frac{\fdiv(0) m}{k} \,.
    \]
    Hence, it follows that
    \begin{align} \label{eq:fdiv:quant:pf:2}
        \frac{\fdiv(0)}{k} \sum_{m=1}^k (m-1) \, Q[S_m] \le
        D_\fdiv^+(P_{\Scal_1} \Vert Q_{\Scal_1}) \le \frac{\fdiv(0)}{k} \sum_{m=1}^k m \, Q[S_m] \,.
    \end{align}
    Putting \eqref{eq:fdiv:quant:pf:1} and \eqref{eq:fdiv:quant:pf:2} together gives 
    \begin{align}  \label{eq:fdiv:quant:pf:3a}
        \inf_{|\Scal_1| \le k} 
        \left|
        D_\fdiv^+(P \Vert Q) -
        D_\fdiv^+(P_{\Scal_1} \Vert Q_{\Scal_1})
        \right| \le 
        \frac{\fdiv(0)}{k} \sum_{m=1}^k Q[S_m] \le \frac{\fdiv(0)}{k} \,,
    \end{align}
    since $\sum_{m=1}^k Q[S_m] = Q[\Xcal_1] \le 1$.
    Repeating the same argument with $P$ and $Q$
    interchanged and replacing $\fdiv$ by $\ftil$ gives
    \begin{align}  \label{eq:fdiv:quant:pf:3b}
        \inf_{|\Scal_2| \le k} 
        \left|
        D_{\ftil}^+(Q \Vert P) -
        D_{\ftil}^+(Q_{\Scal_2} \Vert P_{\Scal_2})
        \right| \le  \frac{\ftil(0)}{k} \,.
    \end{align}
    To complete the proof, we upper bound the inf of $\Scal$ over all partitions of $\Xcal$ with $|\Scal|=k$ by 
    the inf over $\Scal = \Scal_1 \cup \Scal_2$ with 
    partitions 
    $\Scal_1$ of $\Xcal_1$
    and $\Scal_2$ of $\Xcal_2$, and
    $|\Scal_1| = |\Scal_2| = k$.
    Now, under this partitioning, 
    we have, 
    $D_\fdiv^+(P_{\Scal} \Vert Q_{\Scal}) 
    = D_\fdiv^+(P_{\Scal_1} \Vert Q_{\Scal_1})$
    and $D_{\ftil}^+(Q_{\Scal} \Vert P_{\Scal}) 
    = D_{\ftil}^+(Q_{\Scal_2} \Vert P_{\Scal_2})$. 
    Putting this together with the triangle inequality, 
    we get, 
    \begin{align*}
        &\quad \inf_{|\Scal|\le 2k}
        \Big| 
            \Df{P}{Q} - \Df{P_\Scal}{Q_\Scal}
        \Big| \\
        &\le 
        \inf_{\Scal = \Scal_1 \cup \Scal_2} \left\{
        \left| 
            D_\fdiv^+(P \Vert Q) - D_\fdiv^+(P_{\Scal} \Vert Q_{\Scal})
        \right|
        +
        \left| 
            D_{\ftil}^+(Q \Vert P) - D_{\ftil}^+(Q_{\Scal} \Vert P_{\Scal})
        \right| \right\} \\
        &\le
        \inf_{|\Scal_1|\le k} 
        \left| 
            D_\fdiv^+(P \Vert Q) - D_\fdiv^+(P_{\Scal_1} \Vert Q_{\Scal_1})
        \right|
        + \inf_{|\Scal_2|\le k}
        \left| 
            D_{\ftil}^+(Q \Vert P) - D_{\ftil}^+(Q_{\Scal_2} \Vert P_{\Scal_2})
        \right| \\
        &\le \frac{\fdiv(0) + \ftil(0)}{k} \,.
    \end{align*}
\end{proof}

\subsection{Properties and Technical Lemmas}

\begin{lemma} \label{lem:fdiv:ftil-infty}
    Suppose the generator $\fdiv$ satisfies Assumptions~\ref{asmp:fdiv:bounded} and~\ref{asmp:fdiv:1st-deriv}. Then, 
    \[
        \lim_{t \to \infty} \fdiv'(t) = \ftil(0)\,,
        \quad 
        \text{and}
        \lim_{t \to \infty} \ftilg(t) = \fdiv(0)\,.
    \]
\end{lemma}
\begin{proof}
    We start by observing that 
    \[
        \lim_{t \to 0} t |\fdiv'(t)| 
        \le \ConstI \lim_{t \to 0} t \vee t\log\frac{1}{t}
        = 0 \,.
    \]
    Next, a direct calculation gives 
    \[
        \ftilg(1/t) = \fdiv(t) - t\fdiv'(t) \,,
    \]
    so that taking the limit $t \to 0$ gives
    \[
        \lim_{t \to \infty} \ftilg(t) = \fdiv(0) 
        - \lim_{t \to 0} t \fdiv'(t) = \fdiv(0) \,.
    \]
    The proof of the other part is identical. 
\end{proof}

\begin{property} \label{property:fdiv:fprime:monotonicity}
    Suppose $f: (0, \infty) \to [0, \infty)$ is convex 
    and continuously differentiable with $f(1) = 0 = f'(1)$. 
    Then, $f'(x) \le 0$ for all $x \in (0, 1)$
    and $f'(x) \ge 0$ for all $x \in (1, \infty)$.
\end{property}
\begin{proof}
    Monotonicity of $f'$ means that we have for any $x \in (0, 1)$ and $y \in (1, \infty)$
    that $f'(x) \le f'(1) = 0 \le f'(y)$.
\end{proof}

\begin{lemma} \label{lem:techn:missing-mass-2}
    For all $x \in (0, 1)$ and $n \ge 3$, we have
    \[
        0 \le (1-x)^n x \log \frac{1}{x} \le \frac{\log n}{n} \,.
    \]
\end{lemma}
\begin{proof}
    Let $h(x) = (1-x)^n x \log(1/x)$ be defined on $(0, 1)$.
    Since $\lim_{x \to 0} h(x) = 0 < h(1/n)$, the global supremum does not occur as $x \to 0$.
    We first argue that $h$ obtains its global maximum in 
    $(0, 1/n]$. 
    We calculate
    \[
        h'(x) = (1-x)^{n-1}\left( -nx \log\frac{1}{x} + (1-x)\left(\log\frac{1}{x} - 1\right) \right) 
        \le (1-x)^{n-1} (1-nx) \log\frac{1}{x} \,.
    \]
     Note that $h'(x) < 0$ for $x > 1/n$, 
     so $h$ is strictly decreasing on $(1/n, 1)$. Therefore, 
     it must obtain its global maximum on $(0, 1/n]$.
     On this interval, we have, 
    \[
        (1-x)^n x \log\frac{1}{x} \le x \log\frac{1}{x}
        \le \frac{\log n}{n} \,,
    \]
    since $x \log(1/x)$ is increasing on $(0, \exp(-1))$.
\end{proof}

The next lemma comes from \cite[Theorem 1]{berend2012missing}.
\begin{lemma}
\label{lem:techn:missing-mass-1}
    For all $x \in (0, 1)$ and $n \ge 1$, we have
    \[
        0 \le (1-x)^n x \le \exp(-1)/(n+1) < 1/n \,.
    \]
\end{lemma} 
\section{Estimation of Divergences via Parametric Approximations} \label{sec:a:parametric}
We discuss here the parametric approximation approach for divergence estimation mentioned in \Cref{sec:compute}.
Given an embedding model $\varphi: \Xcal \to \reals^d$, we first approximate the $f$-divergence $D_f(P \Vert Q)$ by $D_f(\varphi_{\sharp}P \Vert \varphi_{\sharp}Q)$.
Since $\{\varphi(\xv_i)\}_{i=1}^n$ is an i.i.d.~sample from $\varphi_{\sharp}P$,
we then approximate $\varphi_{\sharp}P$ by a multivariate Gaussian distribution with mean and covariance matrix given by
\begin{align*}
    \hat \mu_P := \frac1n \sum_{i=1}^n \varphi(\xv_i) \quad \mbox{and} \quad \hat \Sigma_P := \frac1{n-1} \sum_{i=1}^n (\varphi(\xv_i) - \hat \mu_P) (\varphi(\xv_i) - \hat \mu_P)^\top,
\end{align*}
respectively.
The distribution $\varphi_{\sharp}Q$ can be approximated by $\mathcal{N}_d(\hat \mu_Q, \hat \Sigma_Q)$ similarly.
Finally, we approximate $D_f(\varphi_{\sharp}P \Vert \varphi_{\sharp}Q)$ by
\begin{align}\label{eq:parametric_est_fdiv}
    D_f\left( \mathcal{N}_d(\hat \mu_P, \hat \Sigma_P) \;\big\Vert\; \mathcal{N}_d(\hat \mu_Q, \hat \Sigma_Q) \right) = \int f\left( \frac{\phi(\zv; \hat \mu_P, \hat \Sigma_P)}{\phi(\zv; \hat \mu_Q, \hat \Sigma_Q)} \right) \phi(\zv; \hat \mu_Q, \hat \Sigma_Q) \D \zv \,,
\end{align}
where $\phi(\cdot \, ; \mu, \Sigma)$ is the probability density function of the multivariate normal distribution $\Ncal(\mu, \Sigma)$. 
To evaluate the integration in \eqref{eq:parametric_est_fdiv}, we can use the Monte Carlo approach---(i) generate i.i.d.~samples $\{\zv_b\}_{b=1}^B$ from $\mathcal{N}_d(\hat \mu_Q, \hat \Sigma_Q)$, and (ii) approximate \eqref{eq:parametric_est_fdiv} by the empirical average
\begin{align} \label{eq:parametric_est_fdiv:monte-carlo}
    \hat D_{f}(\mu_P, \Sigma_P \Vert \mu_Q, \Sigma_Q) = 
    \frac1B \sum_{b=1}^B f\left( \frac{\phi(\zv_b; \hat \mu_P, \hat \Sigma_P)}{\phi(\zv_b; \hat \mu_Q, \hat \Sigma_Q)} \right).
\end{align}

Although this approach is widely used in practice, it has no theoretical guarantee.
Its performance can get arbitrarily bad depending on the two distributions $P$ and $Q$.
We give below a simple example to illustrate this.
\begin{example}
    Consider two distributions $\varphi_{\sharp}P \sim \frac12 \Ncal(-\mu, 1) + \frac12 \Ncal(\mu, 1)$ and $\varphi_{\sharp}Q \sim \Ncal(0, 1)$.
    It is straightforward to get that $\varphi_{\sharp} P$ has mean zero and variance
    \begin{align*}
        \int x^2 \D \varphi_{\sharp} P(x)
        = \frac{1}{2} \int x^2 \phi(x; -\mu, 1) \D x + \frac12 \int x^2 \phi(x; \mu, 1) \D x
        = 1 + \mu^2.
    \end{align*}
    As a result, the KL divergence $\kl(\varphi_{\sharp} P \Vert \varphi_{\sharp} Q)$ can be approximated by
    \begin{align*}
        \kl\left( \Ncal(0, 1+\mu^2) \Vert \Ncal(0, 1) \right) = \frac{\mu^2 - \log{(1 + \mu^2)}}{2}.
    \end{align*}
    On the other hand, we also know that
    \begin{align*}
        \kl(\varphi_{\sharp} P \Vert \varphi_{\sharp} Q)
        &= \int \left[ \frac12 \phi(x; -\mu, 1) + \frac12 \phi(x; \mu, 1) \right] \log{\left[ \frac12 \frac{\phi(x; -\mu, 1)}{\phi(x; 0, 1)} + \frac12 \frac{\phi(x; \mu, 1)}{\phi(x; 0, 1)} \right]} \D x \\
        &= \int \phi(x; \mu, 1) \log{\left[ \frac12 \frac{\phi(x; -\mu, 1)}{\phi(x; 0, 1)} + \frac12 \frac{\phi(x; \mu, 1)}{\phi(x; 0, 1)} \right]} \D x.
    \end{align*}
    Notice that
    \begin{align*}
        \frac12 \frac{\phi(x; -\mu, 1)}{\phi(x; 0, 1)} + \frac12 \frac{\phi(x; \mu, 1)}{\phi(x; 0, 1)}
        = \frac12 e^{-\mu^2/2} \left( e^{-x\mu} + e^{x \mu} \right) = \frac12 e^{-\mu^2/2} e^{-x\mu} (1 + e^{2x \mu}).
    \end{align*}
    As a result, we get
    \begin{align*}
        \kl(\varphi_{\sharp} P \Vert \varphi_{\sharp} Q)
        &= -\frac{\mu^2}{2} - \log{2} - \mu \int x \, \phi(x; \mu, 1) \, \D x + \int \log{(1 + e^{2x\mu})} \, \phi(x; \mu, 1) \, \D x \\
        &= -\frac{\mu^2}{2} - \log{2} - \mu^2 + \int \log{(1 + e^{2x\mu})} \, \phi(x; \mu, 1) \, \D x \\
        &\ge -\frac{3\mu^2}2 - \log{2} + \int 2 x\mu \, \phi(x; \mu, 1) \, \D x \\
        &= \frac{\mu^2}{2} - \log{2}.
    \end{align*}
    This implies that
    \begin{align*}
        \kl(\varphi_{\sharp} P \Vert \varphi_{\sharp} Q) - \kl\left( \Ncal(0, 1+\mu^2) \Vert \Ncal(0, 1) \right) \ge \frac12 \log{(1 + \mu^2)} - \log{2}
    \end{align*}
    can be arbitrarily large as $\mu$ increases.
    Hence, the parametric approximation approach can be extremely inaccurate even in this simple example.
\end{example}

\myparagraph{Computational Complexity}
Estimating the mean vectors and covariance matrices takes $O(nd^2)$ time.
Since evaluating the density $\phi(z; \mu, \Sigma)$ involves computing the quadratic form $(z - \mu)^\top \Sigma^{-1} (z - \mu)$, we can compute $\Sigma^{-1}$ once with time complexity $O(d^3)$ and evaluate $\Sigma^{-1} (z - \mu)$ for different $z$'s where each evaluation cost $O(d^2)$ time.
Assuming that sampling an observation from $\mathcal{N}_d(\hat \mu_Q, \hat \Sigma_Q)$ takes $O(d)$ time, the time complexity of the Monte Carlo approximation is $O(Bd^2 + d^3)$.
Hence, the parametric approximation approach has overall time complexity $O(nd^2 + Bd^2 + d^3)$. 
\section{Experiments: Additional Results} \label{sec:a:expt}
We elaborate on the results in \Cref{sec:experiments} in this section as follows.
\begin{itemize}[nosep]
    \item Appendix~\ref{sec:supp:expt-size_decoding}: full results across model size and decoding for all domains.
    \item Appendix~\ref{sec:supp:expt-length}: full results across text length.
    \item Appendix~\ref{sec:a:expt:fdiv-full}: full comparison to other $f$-divergence frontier summaries.
    \item Appendix~\ref{sec:supp:expt:misc}: use of \mauve for hyperparameter tuning.
\end{itemize}

\begin{table*}[t!]
\centering
\begin{adjustbox}{max width=0.95\textwidth}
\begin{tabular}{llrrrrrrr}
\toprule
\bf Grover Size      &    \bf  Decoding              &       \bf     Gen. PPL &     \bf    Zipf Coef. &     \bf           Rep. &    \bf     Distinct-4 &      \bf    Self-BLEU &  \bf \% Disc. Acc.($\downarrow$) & \bf $\mauve^\star$($\uparrow$) \\
\midrule
\multirow{3}{*}{base} & Sampling &           $37.505$ &           $0.942$ &           $0.002$ &           $0.882$ &           $0.419$ &                     $99.925$ &                    $0.754$ \\
      & Greedy &            $1.413$ &           $1.038$ &           $0.518$ &           $0.081$ &           $0.548$ &                    $100.000$ &                    $0.012$ \\
      & Nucleus, $0.96$ &           $23.064$ &           $0.974$ &           $0.006$ &  \tabemph{$\mathbf{0.847}$} &           $0.462$ &                     $99.950$ &                    $0.764$ \\
\midrule[0.03em]
\multirow{3}{*}{large} & Sampling &           $27.796$ &  ${0.946}$ &  \tabemph{$\mathbf{0.002}$} &           $0.878$ &           $0.429$ &                     $99.450$ &                    $0.836$ \\
      & Greedy &            $1.575$ &           $1.012$ &           $0.366$ &           $0.124$ &           $0.504$ &                    $100.000$ &                    $0.013$ \\
      & Nucleus, $0.98$ &           $20.792$ &           \tabemph{$\mathbf{0.962}$} &           $0.002$ &           $0.859$ &           $0.450$ &                     $98.475$ &                    $0.800$ \\
\midrule[0.03em]
\multirow{3}{*}{mega} & Sampling &           $22.656$ &           $0.950$ &           $0.001$ &           $0.879$ &           $0.427$ &                     $97.300$ &                    $0.847$ \\
      & Greedy &            $1.796$ &           $1.003$ &           $0.316$ &           $0.176$ &           $0.500$ &                    $100.000$ &                    $0.013$ \\
      & Nucleus, $0.96$ &  \tabemph{$\mathbf{14.834}$} &           $0.972$ &           $0.003$ &           $0.848$ &  \tabemph{$\mathbf{0.469}$} &            \tabemph{$\mathbf{88.675}$} &           \tabemph{$\mathbf{0.852}$} \\
\midrule
Human &      n/a             &           $15.356$ &           $0.956$ &           $0.002$ &           $0.842$ &           $0.473$ &         --                     &        --                    \\
\bottomrule
\end{tabular}
 \end{adjustbox}
\caption{\small
News generation evaluation across different Grover model sizes, and decoding approaches.
For nucleus sampling, we show the best hyperparameter value from $\{0.9, 0.92, 0.94, 0.96, 0.98\}$
as per \name.
Disc. Acc. denotes the discrimination accuracy (\%) of a Grover large model trained to distinguish human text
from machine text generated with the model and decoding algorithm of each row. 
Boldfaced numbers indicate the performance closest to the human reference when applicable, or the best performance according to the measure.
}
\label{tab:mauve:expt:grover-appendix}
\end{table*}

\begin{table*}[t!]
\centering
\begin{adjustbox}{max width=0.95\textwidth}
\begin{tabular}{llllllll}
\toprule
\textbf{Decoding} &                   \textbf{Gen. PPL} &                \textbf{Zipf Coef.} &                       \textbf{REP} &                \textbf{Distinct-4} &                 \textbf{Self-BLEU} & \textbf{\% Disc. Acc. ($\downarrow$)} & $\textbf{\mauve}^\star(\uparrow)$ \\
\midrule
Sampling & $38.983_{0.143}$ & $\tabemph{}  \mathbf{1.066}_{0.002}$ & $\tabemph{}  \mathbf{0.001}_{0.000}$ & $0.833_{0.001}$ & $0.518_{0.003}$ & $78.098_{0.365}$ & $0.929_{0.007}$ \\
Nucleus, $0.9$ & $15.433_{0.042}$ & $1.201_{0.002}$ & $0.006_{0.001}$ & $0.719_{0.001}$ & $0.637_{0.002}$ & $75.150_{0.373}$ & $0.914_{0.005}$ \\
Nucleus, $0.92$ & $17.422_{0.060}$ & $1.179_{0.002}$ & $0.004_{0.001}$ & $0.742_{0.001}$ & $0.620_{0.003}$ & $71.979_{0.594}$ & $0.926_{0.003}$ \\
Nucleus, $0.95$ & $\tabemph{}  \mathbf{21.599}_{0.127}$ & $1.147_{0.002}$ & $0.003_{0.000}$ & $\tabemph{}  \mathbf{0.775}_{0.002}$ & $0.589_{0.005}$ & $\tabemph{}  \mathbf{68.586}_{0.583}$ & $\tabemph{}  \mathbf{0.940}_{0.003}$ \\
Top-$50$ & $13.735_{0.027}$ & $1.293_{0.004}$ & $0.002_{0.000}$ & $0.706_{0.001}$ & $0.664_{0.003}$ & $83.549_{0.381}$ & $0.886_{0.010}$ \\
Top-$100$ & $16.527_{0.041}$ & $1.252_{0.001}$ & $0.002_{0.000}$ & $0.743_{0.001}$ & $0.631_{0.001}$ & $78.150_{0.207}$ & $0.913_{0.005}$ \\
Top-$500$ & $23.833_{0.076}$ & $1.153_{0.001}$ & $0.001_{0.000}$ & $0.794_{0.001}$ & $\tabemph{}  \mathbf{0.576}_{0.002}$ & $69.680_{0.450}$ & $\tabemph{}  \mathbf{0.942}_{0.004}$ \\
Greedy & $1.739$ & $1.362$ & $0.988$ & $0.101$ & $0.742$ & $99.712$ & $0.013$ \\
\midrule
Human             &                   $19.704$ &                   $1.101$ &                   $0.001$ &                   $0.783$ &                   $0.571$ &                              &                            \\
\bottomrule
\end{tabular}

 \end{adjustbox}
\caption{\small
Story continuation evaluation across different decoding approaches with GPT-2 medium.
Disc. Acc. denotes the discrimination accuracy (\%) of a classifier (a frozen GPT-2 large model with a classification head) trained to distinguish human text from machine text generated with the decoding algorithm of each row. 
Boldfaced numbers indicate the performance closest to the human reference when applicable, or the best performance according to the measure.
}
\label{tab:mauve:expt:wp-appendix}
\end{table*}

\begin{table*}[t!]
\centering
\begin{adjustbox}{max width=0.7\textwidth}
\begin{tabular}{llrrrrl}
\toprule
\bf GPT-2  Size    &   \bf    Decoding             &  \bf   SP($\uparrow$) &  \bf JS($\downarrow$) & \bf $\epsilon$-PPL($\downarrow$) & \bf Human/BT($\uparrow$) & \bf $\mauve^\star$($\uparrow$) \\
\midrule
\multirow{4}{*}{small} & Greedy &                         $0.431$ &           $0.394$ &                   $1049.589$ &          --            &                    $0.019$ \\
& Sampling &             $0.653$ &           $0.425$ &                     $19.401$ &             $-27.52$ &            $0.655_{0.018}$ \\
      & Nucleus, $0.9$ &         $0.652$ &           $0.414$ &                     $25.938$ &             $-15.78$ &            $0.906_{0.005}$ \\
\midrule
\multirow{4}{*}{medium} & Greedy &                          $0.465$ &           $0.371$ &                    $708.057$ &              --        &                    $0.024$ \\
& Sampling &      $0.670$ &           $0.402$ &                     $14.631$ &             $-30.77$ &            $0.446_{0.010}$ \\
      & Nucleus, $0.9$ &              $0.670$ &           $0.391$ &                     $18.821$ &              $-3.43$ &            $0.936_{0.004}$ \\
\midrule
\multirow{4}{*}{large} & Greedy &                     $0.483$ &           $0.359$ &                    $580.020$ &            --          &                    $0.026$ \\
& Sampling &     $0.679$ &           $0.381$ &                     $12.658$ &              $-6.93$ &            $0.878_{0.008}$ \\
      & Nucleus, $0.95$ &    $0.679$ &           $0.374$ &                     $14.938$ &              $12.55$ &            $0.952_{0.002}$ \\
\midrule[0.03em]
\multirow{4}{*}{xl} & Greedy &                      $0.496$ &  \tabemph{$\mathbf{0.349}$} &                    $497.696$ &              --        &                    $0.033$ \\
    & Sampling &      \tabemph{$\mathbf{0.686}$} &           $0.369$ &            \tabemph{$\mathbf{11.412}$} &               $8.97$ &            $0.908_{0.005}$ \\
      & Nucleus, $0.95$ &            $0.686$ &           $0.363$ &                     $13.677$ &    \tabemph{$\mathbf{15.66}$} &   \tabemph{$\mathbf{0.955}_{0.004}$} \\
      & Adversarial &           n/a &           n/a &                     n/a &             -- &            $0.057$ \\
\bottomrule
\end{tabular}
 \end{adjustbox}
\caption{\small
\name versus comparison measures based on language modeling (SP, JS, and $\eps$-PPL) across different model sizes, and decoding approaches for web text generations. 
SP, JS, and $\eps$-PPL are deterministic because they do not require generations from a decoding algorithm. Moreover, they cannot measure the quality of the adversarial decoding.
The column ``Human/BT'' gives the Bradley-Terry score obtained from a pairwise human evaluation (\Cref{sec:expt:human-eval}).
Boldfaced numbers indicate the best performance according to the measure.
}
\label{tab:mauve:expt:webtext-app-LM}
\end{table*}

\begin{figure*}[t]
\includegraphics[width=0.97\linewidth]{fig_smooth/length/nucleus_main.pdf}
\includegraphics[width=0.97\linewidth]{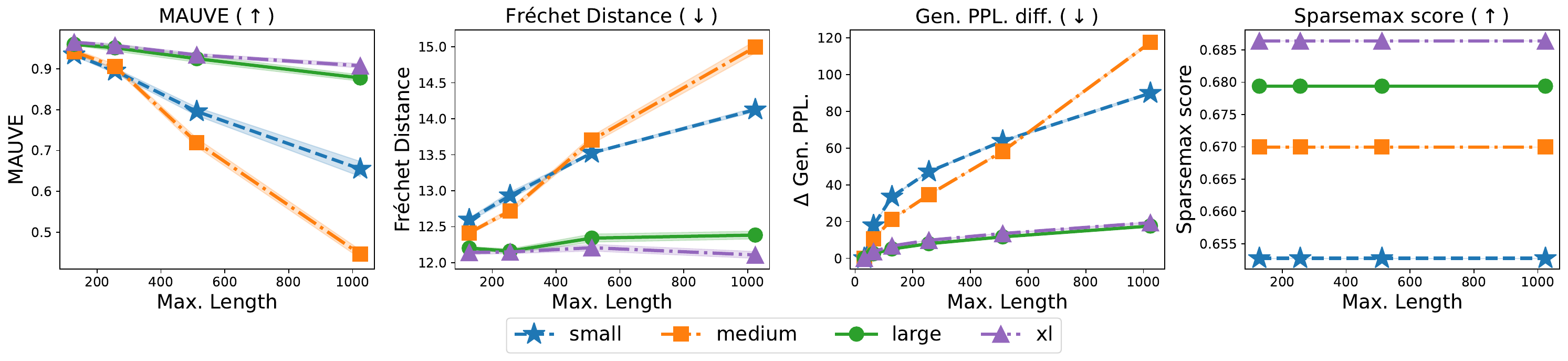}
\includegraphics[width=0.97\linewidth]{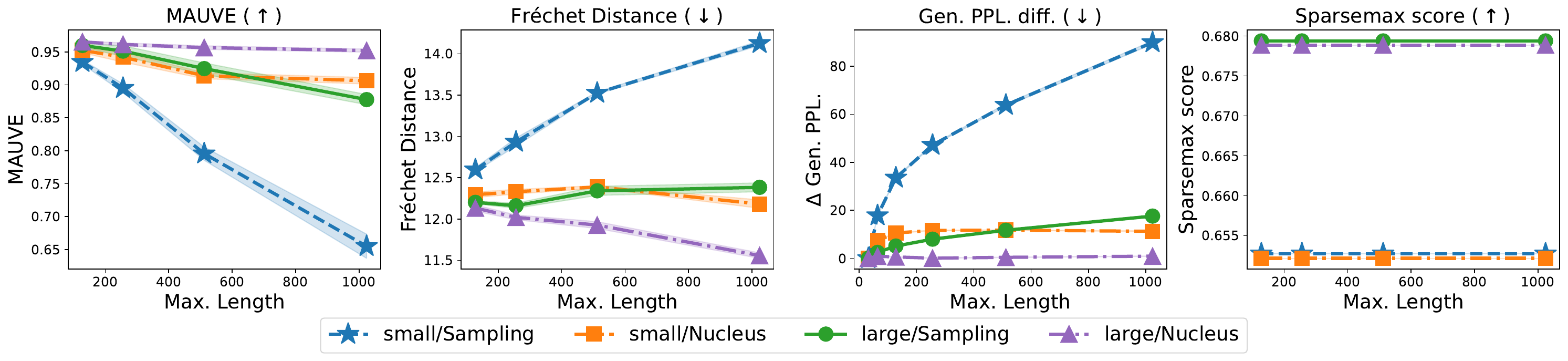}
\caption{ \small
Generation quality versus maximum generation length as per various comparison measures for 
web text generation with GPT-2. 
We expect the quality of the generation to degrade as the maximum length of the text (both machine and human-written) increases. 
\name is the only comparison measure that correctly shows this behavior across all models and decoding algorithms.
The shaded area denotes one standard deviation over generations from 5 random seeds.
}
\label{fig:a:expt:length:all}
\end{figure*}

\subsection{Results for Other Domains: News and Stories} \label{sec:supp:expt-size_decoding}

The analogue of \Cref{tab:mauve:expt:webtext-full}
for the news and story domains are
Tables~\ref{tab:mauve:expt:grover-appendix}
and~\ref{tab:mauve:expt:wp-appendix} respectively. 
These are qualitatively similar to the web text domain. 
\name, like discrimination accuracy, rates larger models as better and nucleus sampling as better than ancestral sampling and greedy decoding. An exception to this rule is Grover large, where \name thinks ancestral sampling is better than nucleus sampling. The statistics-based measures, namely Zipf coefficient, Repetition, and the fraction of distinct $ 4$ grams all prefer smaller Grover sizes.

Next, we turn to the language modeling comparison measures
in \Cref{tab:mauve:expt:webtext-app-LM}. 
JS consistently favors greedy decoding,
which produces far worse text than other decoding algorithms. Likewise, $\eps$-PPL favors ancestral sampling, which also produces somewhat degenerate
text~\cite{holtzman2019curious}, while SP appears to
be unable to distinguish between ancestral sampling and nucleus sampling. This makes SP, JS, and $\eps$-PPL unsuitable to compare generated text to human text.

\subsection{Effect of Text Length} \label{sec:supp:expt-length}

We now turn to the plot of comparison measures versus text length in \Cref{fig:a:expt:length:all}. 
This shows the results of \Cref{fig:expt:length:main} for different hyperparameters.
Recall that we expect the quality of the generation to degrade as the maximum length of the text (both machine and human-written) increases.

\subsection{Comparison with Other Divergences and Optimal Transport}
\label{sec:a:expt:fdiv-full}

The full version of \Cref{table:fdiv,table:ot-corr} from \Cref{sec:expt:other-divergence} 
are given as \Cref{table:fdiv-full,table:ot-full} respectively.

\begin{table*}[t!]
\centering
\begin{adjustbox}{max width=\textwidth}
\begin{tabular}{ccrrrrrrrr}
 \toprule
\begin{tabular}{c} \textbf{GPT-2} \\  \textbf{Size}\end{tabular}
 &  \textbf{Decoding} & 
 ${\mauve}_\kl^\star$ ($\uparrow$) &    $\fint_\kl^\star$ ($\downarrow$) &   $\midp_\kl^\star$ ($\downarrow$) &  $\mauve_{\chi^2}^\star$ ($\uparrow$) &  $\midp_{\chi^2}^\star$ ($\downarrow$) &        ${\tv}^\star$ ($\downarrow$) &  $H^2_\star$ ($\downarrow$) & 
 \begin{tabular}{c}\textbf{BT} ($\uparrow$) \\ \textbf{Human-like}  \end{tabular}
 \\
 \midrule
\multirow[c]{2}{*}{small} & Sampling & $0.655_{0.018}$ & $0.033_{0.002}$ & $0.105_{0.004}$ & $0.335_{0.020}$ & $0.191_{0.007}$ & $0.363_{0.006}$ & $0.225_{0.010}$ & $-27.518$ \\
 & Nucleus, $0.9$ & $0.906_{0.005}$ & $0.016_{0.001}$ & $0.044_{0.001}$ & $0.734_{0.011}$ & $0.084_{0.003}$ & $0.230_{0.005}$ & $0.091_{0.003}$ & $-15.783$ \\
 \hline
\multirow[c]{2}{*}{medium} & Sampling & $0.446_{0.010}$ & $0.042_{0.001}$ & $0.160_{0.003}$ & $0.164_{0.004}$ & $0.277_{0.003}$ & $0.443_{0.004}$ & $0.356_{0.009}$ & $-30.769$ \\
 & Nucleus, $0.9$ & $0.936_{0.004}$ & $0.012_{0.001}$ & $0.035_{0.001}$ & $0.805_{0.009}$ & $0.068_{0.002}$ & $0.205_{0.004}$ & $0.073_{0.002}$ & $-3.429$ \\
 \hline
\multirow[c]{2}{*}{large} & Sampling & $0.878_{0.008}$ & $0.017_{0.000}$ & $0.052_{0.002}$ & $0.672_{0.016}$ & $0.098_{0.003}$ & $0.251_{0.004}$ & $0.107_{0.004}$ & $-6.935$ \\
 & Nucleus, $0.95$ & $0.952_{0.002}$ & $0.010_{0.000}$ & $0.030_{0.001}$ & $0.849_{0.007}$ & $0.058_{0.002}$ & $0.187_{0.005}$ & $0.061_{0.002}$ & $12.553$ \\
 \hline
\multirow[c]{2}{*}{xl} & Sampling & $0.908_{0.005}$ & $0.014_{0.001}$ & $0.044_{0.001}$ & $0.737_{0.012}$ & $0.083_{0.003}$ & $0.232_{0.005}$ & $0.090_{0.003}$ & $8.966$ \\
 & Nucleus, $0.95$ & $\tabemph{}  \mathbf{0.955}_{0.004}$ & $\tabemph{}  \mathbf{0.010}_{0.001}$ & $\tabemph{}  \mathbf{0.029}_{0.002}$ & $\tabemph{}  \mathbf{0.857}_{0.012}$ & $\tabemph{}  \mathbf{0.056}_{0.003}$ & $\tabemph{}  \mathbf{0.185}_{0.006}$ & $\tabemph{}  \mathbf{0.059}_{0.003}$ & $\tabemph{}  \mathbf{15.664}$ \\
\bottomrule
\end{tabular} \end{adjustbox}
\caption{\small Comparison $f$-divergences frontier summaries for the web text domain. 
    The correlations from this table are reported in \Cref{table:fdiv} of \Cref{sec:expt:other-divergence}.
    The subscripts denote standard deviations over 5 random seeds.
    Boldfaced numbers indicate the smallest gap between the two distributions.
}
\label{table:fdiv-full}
\end{table*}

\begin{table*}[t!]
\centering
\begin{adjustbox}{max width=0.99\textwidth}
\begin{tabular}{ccrrrrrrr}
\toprule
\multirow{3}{*}{
\begin{tabular}{c}
\textbf{GPT-2} \\
\textbf{Size}
\end{tabular}
} & 
\multirow{3}{*}{\textbf{Decoding}} & \multicolumn{3}{c}{\textbf{OT variants} ($\downarrow)$}  & \multicolumn{3}{c}{\textbf{\mauve variants} ($\uparrow$)} & \\
\cmidrule(lr){3-5} \cmidrule(lr){6-8} 
& &  Plug-in &    Fr\'echet &   Quantized &  
\begin{tabular}{c} OT + \\ Linear \\ interpolation \end{tabular} &  
\begin{tabular}{c} OT + \\ Barycenteric \\ interpolation \end{tabular} &    
\begin{tabular}{c} (Default) KL + \\ Linear \\ interpolation  \end{tabular}  &
 \begin{tabular}{c} \textbf{BT} ($\uparrow$) \\ \textbf{Human-like} \end{tabular}
\\
\midrule
\multirow[c]{2}{*}{small} & Sampling & $763.281_{1.264}$ & $199.591_{0.788}$ & $0.158_{0.001}$ & $0.814_{0.001}$ & $0.937_{0.001}$ & $0.655_{0.018}$ & $-27.518$ \\
 & Nucleus, $0.9$ & $693.263_{4.610}$ & $148.388_{1.236}$ & $0.083_{0.005}$ & $0.935_{0.007}$ & $0.964_{0.001}$ & $0.906_{0.005}$ & $-15.783$ \\
 \midrule
\multirow[c]{2}{*}{medium} & Sampling & $791.758_{4.780}$ & $224.970_{2.526}$ & $0.208_{0.002}$ & ${0.725}_{0.004}$ & ${0.914}_{0.001}$ & ${0.446}_{0.010}$ & ${-30.769}$ \\
 & Nucleus, $0.9$ & $700.496_{3.961}$ & $140.174_{0.813}$ & $0.077_{0.004}$ & $0.942_{0.005}$ & $0.967_{0.001}$ & $0.936_{0.004}$ & $-3.429$ \\
 \midrule
\multirow[c]{2}{*}{large} & Sampling & $717.909_{5.618}$ & $153.358_{1.325}$ & $0.104_{0.004}$ & $0.905_{0.006}$ & $0.958_{0.002}$ & $0.878_{0.008}$ & $-6.935$ \\
 & Nucleus, $0.95$ & $\tabemph{} \mathbf{681.883}_{4.367}$ & $133.583_{0.762}$ & $0.062_{0.001}$ & $0.961_{0.002}$ & $0.969_{0.001}$ & $0.952_{0.002}$ & $12.553$ \\
 \midrule
\multirow[c]{2}{*}{xl} & Sampling & $705.482_{4.617}$ & $146.593_{1.136}$ & $0.090_{0.003}$ & $0.924_{0.004}$ & $0.962_{0.001}$ & $0.908_{0.005}$ & $8.966$ \\
 & Nucleus, $0.95$ & $685.131_{3.258}$ & $\tabemph{} \mathbf{132.927}_{1.555}$ & $\tabemph{} \mathbf{0.061}_{0.003}$ & $\tabemph{} \mathbf{0.962}_{0.004}$ & $\tabemph{} \mathbf{0.970}_{0.001}$ & $\tabemph{} \mathbf{0.955}_{0.004}$ & $\tabemph{} \mathbf{15.664}$ \\
\bottomrule
\end{tabular} \end{adjustbox}
\caption{\small 
Comparison of measures based on optimal transport for the web text domain. 
    The correlations from this table are reported in \Cref{table:ot-corr} of \Cref{sec:expt:other-divergence}.
    The subscripts denote standard deviations over 5 random seeds.
    Boldfaced numbers indicate the smallest gap between the two distributions.
}
\label{table:ot-full}
\end{table*}

\subsection{Use of \mauve for Hyperparameter Tuning} \label{sec:supp:expt:misc}
\Cref{fig:a:expt:model_selection:appendix} plots \name for nucleus and top-$K$ sampling for various values of the hyperparameters $p$ and $K$. This illustrates the utility of \mauve for hyperparameter tuning.

\begin{figure*}[t]
\centering
\includegraphics[width=0.63\linewidth]{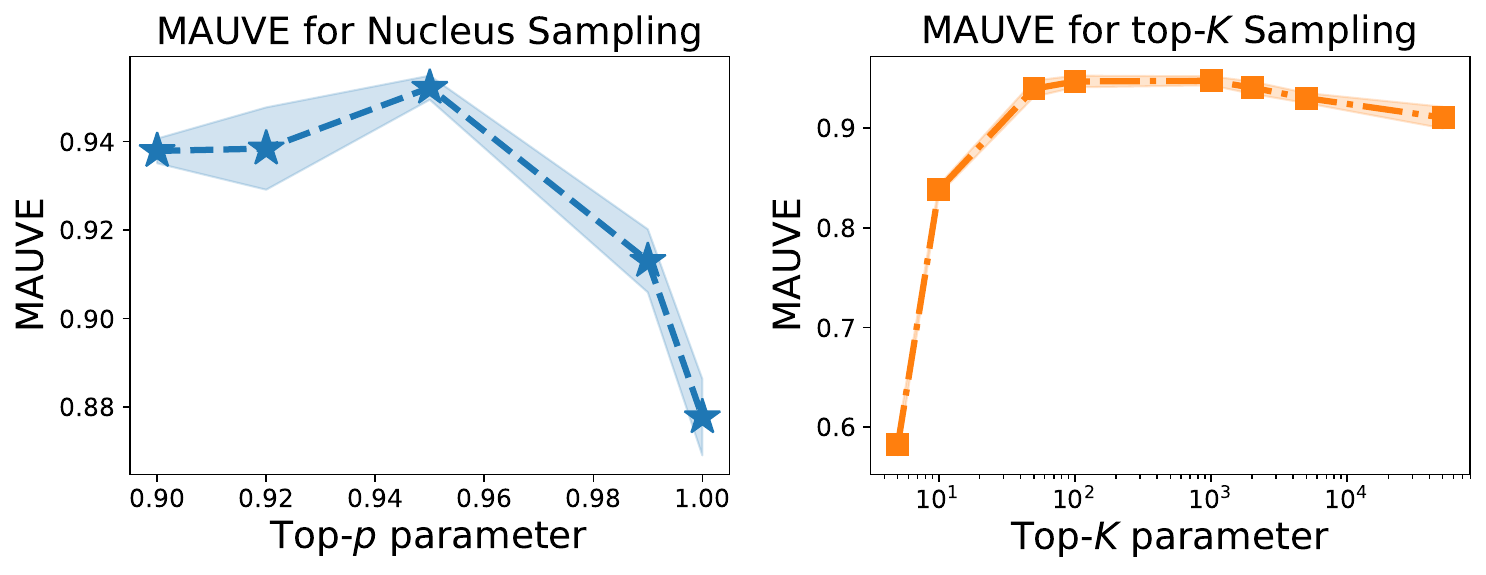}
\caption{\small
\name for nucleus and top-$K$ sampling 
for different values of $p$ and $K$
for GPT-2 large. 
\name rates nucleus sampling with $p=0.95$
and top-$K$ sampling with $100 \le K \le 1000$ as the best choices.
The shaded area denotes one standard deviation over generations from 5 random seeds.
}
\label{fig:a:expt:model_selection:appendix}
\end{figure*} 
\section{Human Evaluation: Protocol and Full Results} \label{sec:a:human-eval}
We describe the human evaluation protocol and results of \Cref{sec:expt-setup:human-eval} in detail. The outline for this section is:
\begin{itemize}[nosep]
    \item \Cref{sec:a:human-eval:bt}: Details of the Bradley-Terry statistical model used to obtain ranking from pairwise preferences. 
    \item \Cref{sec:a:human-eval:results}: Full results of the human evaluation. 
    \item \Cref{sec:a:human-eval:datasheet}: Additional details of the human evaluation protocol. \nocite{shimorina2021human}
\end{itemize}

\begin{figure*}[t!]
    \centering
    \fbox{\includegraphics[width=0.9\linewidth]{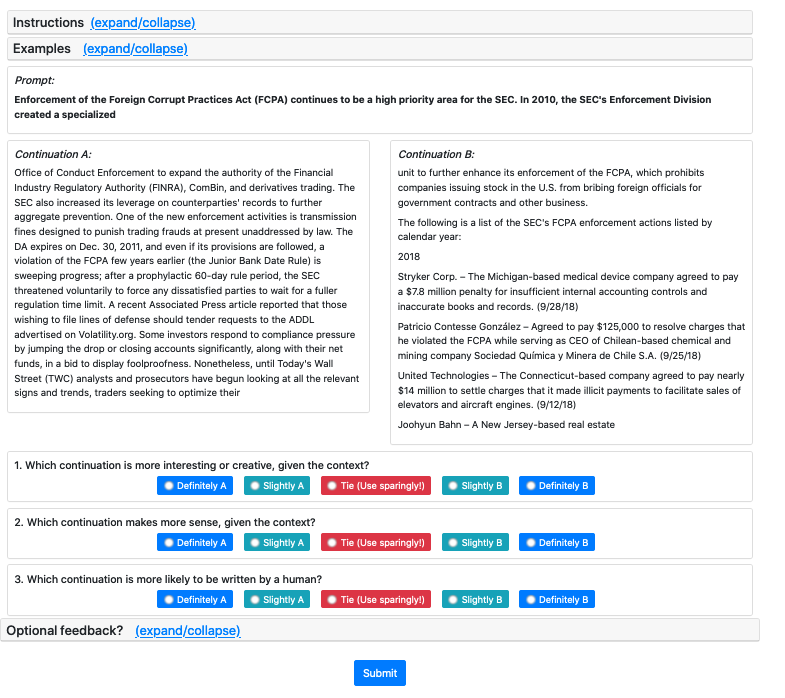}}
    \caption{\small Mechanical Turk interface for human evaluation.}
    \label{fig:mauve:human-eval-interface}
\end{figure*}

\subsection{From Pairwise Preferences to Ranking: the Bradley-Terry Model}
\label{sec:a:human-eval:bt}
We compute the Bradley-Terry (BT) scores from the pairwise preferences 
obtained from the human evaluation along each of the three axes
interesting, sensible, and more likely to be written by a human.

\myparagraph{Bradley-Terry Model Review}
Given $n$ players with scores $w_1, \cdots, w_n$, the 
the Bradley-Terry model~\cite{bt:book:1995} 
models the outcome of a head-to-head comparison
of any two players using a sigmoid%
\footnote{The scaling factor $100$ is arbitrary and does not change the model or the obtained rankings.}
\[
    \text{Prob}(i \text{ beats } j) = \frac{1}{1 + \e^{-(w_i - w_j) / 100}} \,.
\]
The model also assumes the outcome of each head-to-head comparison of any pair of players is independent of all
other comparisons.
Note that the model is invariant to 
additive shifts of the scores, i.e., 
the model probabilities induced by scores 
$w_1 + C, \cdots, w_n + C$ is same as the that 
induced by $w_1, \cdots, w_n$ for any constant $C$.
For uniqueness, we normalize the scores so that their mean is $0$.

\myparagraph{Fitting the Model}
The Bradley-Terry model can be fit to data 
using Zermelo's algorithm~\cite{hunter2004mm}.
Suppose that we are given a dataset
of head-to-head comparisons summarized by 
numbers $N_{ij}$ denoting the number of times player $i$
has defeated player $j$. 
Then, the negative log-likelihood $\ell(w_1, \cdots w_n)$ 
of the data under the 
The Bradley-Terry model can be written as
\[
    \ell(w_1, \cdots, w_n) = 
        -\sum_{i=1}^n\sum_{j=1}^n N_{ij} \log(1 + \e^{-(w_i - w_j) / 100}) \,.
\]
This is convex in the parameters $w_1, \cdots, w_n$ since 
the log-sum-exp function is convex. 
Zermelo's algorithm~\cite{hunter2004mm} can be used to compute the maximum likelihood estimate. 
Denote $\widetilde w_i = w_i / 100$. 
Starting from an initial estimate $\widetilde w_1^{(0)}, \cdots, \widetilde w_n^{(0)}$,
each iteration of Zermelo's algorithm performs the update
\[
     u_i^{(t)} = \log\left( \sum_{j\neq i} N_{ij} \right)
     - \log\left( \sum_{j\neq i} \frac{N_{ij} + N_{ji}}{
     \exp(\widetilde w_i^{(t)}) + \exp(\widetilde w_j^{(t)})} \  \right)
\]
followed by the mean normalization 
\[
    \widetilde w_i^{(t+1)} = u_i^{(t)} - \frac{1}{n} \sum_{j=1}^n u_j^{(t)} \,.
\]
\myparagraph{Processing Raw Data}
We collect the result of a head-to-head comparison using 5 options: Definitely A/B, Slightly A/B, or a Tie. We combine ``Definitely A'' and ''Slightly A`` into a single category denoting that A wins, while ties were assigned to either A or B uniformly at random.

\begin{table*}[t!]
\centering
\small
\begin{adjustbox}{max width=0.8\textwidth}
\begin{tabular}{ccrrr}
\toprule
  \textbf{GPT-2 Size}    &  \textbf{Decoding}  & \textbf{BT/Human-like} & \textbf{BT/Interesting} & \textbf{BT/Sensible} \\
\midrule
Human & {} &      $47.251$ &       $25.503$ &    $43.229$ \\
\hline
xl & Nucleus, $p=0.95$ &      \tabemph{$\mathbf{15.664}$} &       \tabemph{$\mathbf{23.046}$} &    \tabemph{$\mathbf{31.888}$} \\
      & Sampling &       $8.966$ &        $9.529$ &     $7.753$ \\
\hline
large & Nucleus, $p=0.95$ &      $12.553$ &        $6.785$ &     $8.781$ \\
      & Sampling &      $-6.935$ &       $-1.532$ &    $-7.106$ \\
\hline
medium & Nucleus, $p=0.9$ &      $-3.429$ &      $-12.824$ &    $-7.293$ \\
      & Sampling &     $-30.769$ &      $-34.323$ &   $-32.004$ \\
\hline
small & Nucleus, $p=0.9$ &     $-15.783$ &       $-0.697$ &    $-7.442$ \\
      & Sampling &     $-27.518$ &      $-15.487$ &   $-37.805$ \\
\bottomrule
\end{tabular}
 \end{adjustbox}
\caption{ \small
Fitted Bradley-Terry (BT) scores 
for each of the three axes 
rated by human annotators:
``Human-like'' denotes measures how likely 
the text is to be written by a human,
while ``Interesting'' and ``Sensible'' quantify how 
interesting or sensible the text is. 
The Spearman rank correlations between 
each of these scores are:
    Human-like and Interesting: $\textbf{0.917}$,
    Human-like and Sensible: $\textbf{0.917}$,
    Interesting and Sensible: $\textbf{0.967}$.
}
\label{tab:mauve:expt:webtext-human-eval-all}
\end{table*}

\subsection{Full Results of the Human Evaluation}
\label{sec:a:human-eval:results}

\myparagraph{BT Model for Human Eval}
In our setting, each ``player'' is
a source of text, i.e., 
one human, plus, 
eight model and decoding algorithm pairs
(four model sizes GPT-2 small/medium/large/xl
coupled with pure sampling or nucleus sampling). 
We compute the BT score of each player 
as the maximum likelihood estimate of 
corresponding the parameters $w_1,\cdots, w_n$
based on head-to-head human evaluation data.

A higher BT score indicates a stronger preference 
from human annotators.
The BT scores are reported in Table~\ref{tab:mauve:expt:webtext-human-eval-all}.

\myparagraph{Interpreting BT scores}
The BT scores reported in Table~\ref{tab:mauve:expt:webtext-human-eval-all}
give us predictions from the sigmoid model above.
For example, consider the column ``BT/Human-like''.
The best model-generated text, 
GPT-2 xl with nucleus sampling will 
lose to human text with probability $0.578$.
At the other end, GPT-2 small with nucleus sampling 
will lose to human text with probability $0.679$.
This shows that there is still much room for improvement in 
model-generated text. 

\myparagraph{Discussion}
In general, the BT scores from human evaluations 
and \name both indicate that 
(a) nucleus sampling is better than pure sampling for the same model size, and,
(b) larger model sizes are better for the same decoding algorithm.
There is one exception to this rule, as per both
the human evaluations and \name: 
GPT-2 small is better than GPT-2 medium for pure sampling.

\subsection{Additional Details}
\label{sec:a:human-eval:datasheet}

We describe more details for the human evaluation. The terminology below is taken from \cite{shimorina2021human}.

\myparagraph{Number of Outputs Evaluated}
We compare $9$ players: 
one player is ``human'', representing human-written text, 
whereas the other $8$ are text generated by the model using 
the first $35$ tokens of the corresponding human generation as a prompt. Each of the $8$ non-human players come from 
a GPT-2 model of different sizes (small, medium, large, xl)
and two decoding algorithms (pure sampling and nucleus sampling). 
We perform $90$ comparisons between each pair of players, 
so each player is evaluated $90 \times 8 = 720$ times. 

\myparagraph{Prompt Filtering}
We manually selected $1831$ out of $5000$ prompts which are well-formed English sentences from the web text test set.\footnote{
The web text dataset is scraped from the internet 
and is {\em not} curated. This dataset contains poor prompts
such as headers of webpages or error messages, such as:
``Having trouble viewing the video? Try disabling any ad-blocking extensions currently running on your browser''
or ``Front Page Torrents Favorites My Home My Galleries Toplists Bounties News Forums Wiki''.
We exclude such prompts as they are unsuitable for human evaluation.
}
For every head-to-head comparison, we sample $90$
prompt without replacement and then sample the corresponding 
completions (for human-generated text, we use the test set of web text). We only consider a pair of players for human evaluation if the generation from each player is 
at least $200$ BPE tokens long (and we truncate each generation
at a maximum length of $256$ BPE tokens).

\myparagraph{Number of Evaluators}
 $214$ unique evaluators 
participated in the evaluation. 
Of these, 
$11$ evaluators supplied at least $50$ annotations
$95$ evaluators supplied at least $10$
annotations.

\myparagraph{Evaluator Selection and Pay}
We conduct our human evaluation on Amazon Mechanical Turk. Since the task only requires elementary reading and understanding skills in English, we open 
the evaluations to non-experts. Each crowd worker was paid 
$0.40$ per annotation. The pay was estimated based on a 
$\$16$/hour wage for the $85$\textsuperscript{th} percentile
of response times from a pilot study (which was approx. $98$ seconds per annotation). These evaluators are not previously known to the authors.

\myparagraph{Training and Instructions}
The evaluators were given instructions about the task 
and two detailed examples. No other training was provided due to the elementary nature of the task. The screenshots of these 
examples are given in Figure~\ref{fig:mauve:human-eval-examples} 
while the instructions read:
\begin{displayquote}
{\small

\textbf{Task Info}: We are studying how good AI models are at generating text on the internet. You are given a snippet of text from a random document on the internet, called the "prompt" or the "context", as well as two continuations, A and B. One or both of these is written by an AI. You must choose (a) which of two continuations is more interesting, (b) which makes more sense given the prompt, and, (c) which is more likely to have been written by a human, as per your assessment.

\textbf{Guidelines}:
\begin{itemize}[itemsep=0cm,leftmargin=0.5cm]
    \item There are five choices for each question: Definitely A/B, Slightly A/B, or Tie. Please use the "Tie" option extremely sparingly! (No more than one in every ten pairs should be chosen as a tie along any of the three questions).
    \item The questions can have different answers! Some text is very creative or interesting, but it doesn't quite fit the prompt or make sense.
    \item Try to focus on quality over quantity. The text can be long but contain rambly gibberish.
    \item Don't worry if the text ends abruptly or has other artifacts of the website downloading process (text like 'Advertisement' for instance).
    \item Please do your best, some of these are pretty challenging!
    \item Answering each question should take around 1.5 minutes on average, as per our estimation. We have calibrated the pay to be $\$16$ per hour with this speed.
\end{itemize}
}
\end{displayquote}

\myparagraph{Quality Control}
All annotations made in under $25$ seconds were excluded
for quality control
(the mean response time per annotation was $47$ seconds).

\myparagraph{Quality Criteria}
We use three quality criteria. The questions asked to the evaluators are (verbatim):
\begin{enumerate}[itemsep=0cm,leftmargin=0.5cm]
\item Interestingness: ``Which continuation is more interesting or creative, given the context?"
\item Sensible: ``Which continuation makes more sense, given the context?''
\item Human-like: ``Which continuation is more likely to be written by a human?''
\end{enumerate}
Note that we do explicitly name the criteria in the evaluation form, although those names could be inferred from the definitions. We use these names only in the paper. 

\begin{figure*}[p!]
    \centering
    \includegraphics[width=0.65\textwidth]{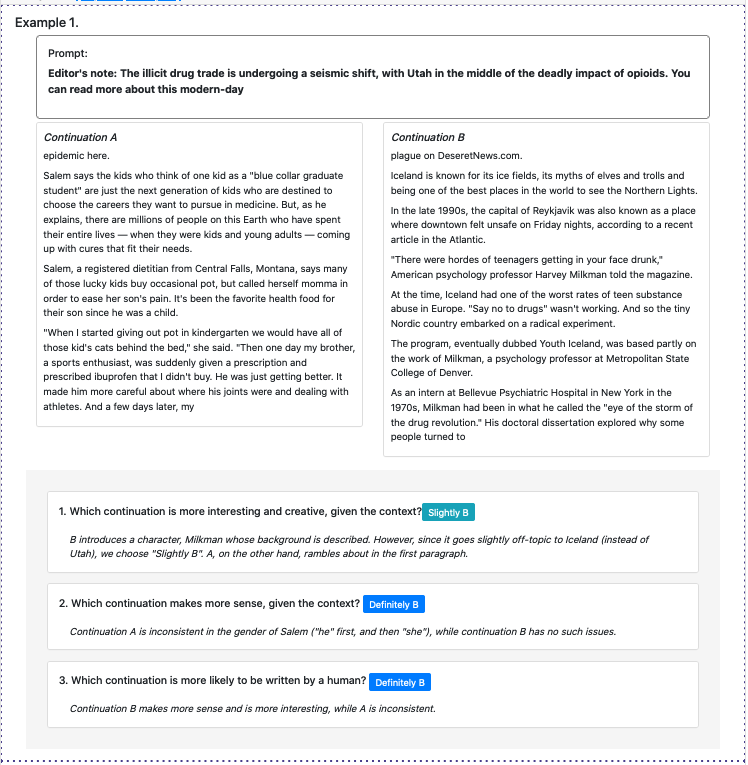}
    
    \includegraphics[width=0.65\textwidth]{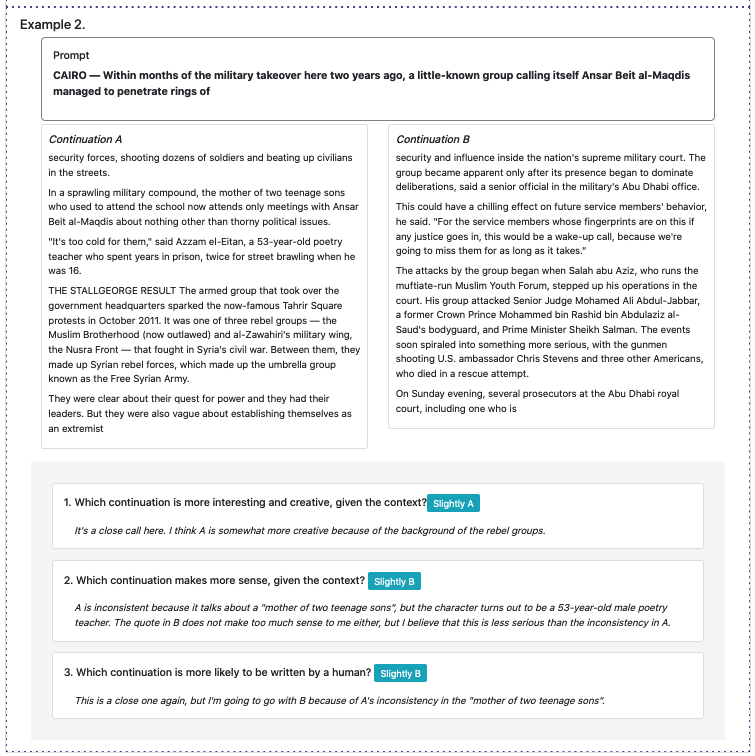}
    \caption{\small Annotated examples shown to the evaluators.}
    \label{fig:mauve:human-eval-examples}
\end{figure*}

\vskip 0.2in
\bibliography{bib/mauve}

\end{document}